\documentclass{elsart}
\usepackage{graphicx}
\usepackage{amssymb}
\usepackage{latexsym}
\usepackage{subfigure}

 \usepackage{epic}
 \usepackage{eepic}
 \usepackage{epsfig}
 \usepackage{url}
 \usepackage{xspace}
 \usepackage{latexsym}

\newif\ifreport

\newcommand{\nop}[1]{}

\newcommand{\NP}{\textrm{NP}\xspace}

\newcommand{\Or}{\ensuremath{\vee}\xspace}
\newcommand{\derives}{\ensuremath{\mathtt{\ :\!\!-}\ }}

\newcommand{\p}{\ensuremath{{\mathcal{P}}}}
\newcommand{\GP}{\ensuremath{Ground(\p)}}

\newcommand{\BP}{\ensuremath{B_{\p}}}
\newcommand{\UP}{\ensuremath{U_{\p}}}

\newcommand{\EDB}[1]{\ensuremath{EDB(#1)}}
\newcommand{\IDB}[1]{\ensuremath{IDB(#1)}}

\newcommand{\R}{\ensuremath{r}}

\newcommand{\HR}{\ensuremath{H(\R)}}
\newcommand{\BR}{\ensuremath{B(\R)}}
\newcommand{\BpR}{\ensuremath{B^+(\R)}}
\newcommand{\BnR}{\ensuremath{B^-(\R)}}

\newcommand{\naf}{\ensuremath{not}\xspace}

\newcommand{\dlv}{{\sc DLV}\xspace}

\newcommand{\tuple}[1]{\langle#1\rangle}

\newcommand{\ground}[1]{\ensuremath{Ground(#1)}}

\newcommand{\head}[1]{\ensuremath{H(#1)}}
\newcommand{\body}[1]{\ensuremath{B(#1)}}
\newcommand{\posbody}[1]{\ensuremath{B^+(#1)}}
\newcommand{\negbody}[1]{\ensuremath{B^-(#1)}}

\newenvironment{simpleprogram}[1][]
   {\vspace{-0.5ex}\begin{itemize}\item[]
      \tt
      \begin{tabbing}
      \code{#1}\ \= \kill
   }
   {\end{tabbing}\end{itemize}\vspace{-3ex}}

\newenvironment{simplealignedprogramstub}[1][]
   {\vspace{-0ex}
      \begin{tabbing}
      #1\kill
   }
   {\end{tabbing}
\vspace{-6ex}}

\newenvironment{sublabeledprogram}[1][]
   {\begin{array}{ll}\setlength{\arraycolsep}{0pt}}
   {\end{array}}

\newcommand{\code}[1]{\ensuremath{#1}}
\newenvironment{dlvcode}
  {\begin{displaymath}\begin{array}{l}}
  {\end{array}\end{displaymath}}

\newenvironment{proof}{\noindent \textbf{Proof.}}{\hfill $\Box$}
\newtheorem{theorem}{Theorem}[section]
\newtheorem{example}[theorem]{Example}
\newtheorem{definition}[theorem]{Definition}
\newtheorem{proposition}[theorem]{Proposition}

\newtheorem{lemma}[theorem]{Lemma}

\newcommand{\dat}{\mbox{Datalog}}
\newcommand{\datd}{\mbox{Datalog$^{\vee}$}\xspace}

\newcommand{\datdn}{\mbox{Datalog$^{\vee,\neg_s}$}}
\newcommand{\datdnr}{\mbox{Datalog$^{\vee,\neg}$}}
\newcommand{\datn}{\mbox{Datalog$^{\neg}$}\xspace}
\newcommand{\dats}{\mbox{Datalog$^{\neg_s}$}}
\newcommand{\datds}{\mbox{Datalog$^{\vee,\neg_s}$}}

\newcommand{\Ans}{{\it Ans}}
\newcommand{\F}{\mathcal{F}}
\renewcommand{\P}{\mathcal{P}}
\newcommand{\Q}{\mathcal{Q}}
\newcommand{\SM}{\mathcal{SM}}

\newcommand{\M}{\mathcal{M}}
\newcommand{\bcons}{\ensuremath{\models_b}}
\newcommand{\ccons}{\ensuremath{\models_c}}

\newcommand{\qrelation}[3]{\ensuremath{{#1}_{#2}^{#3}}}

\newcommand{\qequiv}[2]{\ensuremath{\qrelation{\equiv}{#1}{#2}}}
\newcommand{\bqequiv}[1]{\ensuremath{\qequiv{#1}{b}}}
\newcommand{\cqequiv}[1]{\ensuremath{\qequiv{#1}{c}}}

\newcommand{\magicRules}{\ensuremath{\mathit{magicRules}}}
\newcommand{\modifiedRules}{$\mathit{modifiedRules}$}

\newcommand{\magic}[1]{\ensuremath{magic(#1)}}
\newcommand{\dmsqp}{\ensuremath{\DMS(\Q,\p)}}

\newcommand{\killed}[4]{\ensuremath{killed^{#1}_{#3,#4}(#2)}}
\newcommand{\killedmpmp}{\ensuremath{\killed{M'}{M'}{\Q}{\p}}}
\newcommand{\killedmpnp}{\ensuremath{\killed{M'}{N'}{\Q}{\p}}}
\newcommand{\variant}[4]{\ensuremath{{variant}_{#1,#2}^{#3}(#4)}}
\newcommand{\variantqpm}[1]{\ensuremath{\variant{\Q}{\p}{#1}{M}}}
\newcommand{\variantqpi}[1]{\ensuremath{\variant{\Q}{\p}{#1}{I}}}
\newcommand{\Mpp}{\ensuremath{M'|_{\BP}}}
\newcommand{\Npp}{\ensuremath{N'|_{\BP}}}

\newcommand{\SIPSprec}[2]{\ensuremath{\prec_{#1}^{#2}}}
\newcommand{\SIPSprecR}[1]{\ensuremath{\SIPSprec{\R}{#1}}}
\newcommand{\magica}{*}

\newcommand{\DMS}{\ensuremath{\mathtt{DMS}}}

\newcommand{\SMS}{\ensuremath{\mathtt{SMS}}}

\newcommand{\G}{\mathcal{G}}
\newcommand{\I}{\mathcal{I}}
\renewcommand{\S}{\mathcal{S}}
\newcommand{\D}{\mathcal{D}}
\newcommand{\schemag}{\Psi}

\newcommand{\tup}[1]{\langle #1\rangle}

\newcommand{\Pim}{\Pi_{\M}}
\newcommand{\Ped}{\Pi_{ED}}
\newcommand{\Pkd}{\Pi_{KD}}

\newcommand{\la}{\derives}

\newcommand{\Ss}{{\mathcal{S}}}
\newcommand{\intsys}{\tup{\G,\Ss,\M}}

\renewcommand{\t}{\bar t}
\newcommand{\s}{\bar s}
\newcommand{\z}{\bar z}

\def\<{\mbox{$\langle$}}
\def\>{\mbox{$\rangle$}}

\newcounter{myenumctr}

\newcommand{\hs}{\hspace{3mm}}

\journal{Artificial Intelligence}

\begin{document}

\begin{frontmatter}
\title{Magic Sets for Disjunctive $\dat$ Programs}

\thanks[titlethanks]{%
Preliminary portions of this paper appeared in the proceedings of the 20th International Conference on Logic
Programming (ICLP'04).}
\author{Mario Alviano}
\ead{alviano@mat.unical.it}
\author{Wolfgang Faber}
\ead{faber@mat.unical.it}
\author{Gianluigi Greco}
\ead{ggreco@mat.unical.it}
\author{Nicola Leone}
\ead{leone@mat.unical.it}
\address{Department of Mathematics, University of Calabria, 87036 Rende, Italy}

\begin{abstract}

In this paper, a new technique for the optimization of (partially) bound queries over disjunctive $\dat$ programs
with stratified negation is presented.
The technique exploits the propagation of query bindings and extends the Magic Set optimization technique
(originally defined for non-disjunctive programs).

An important feature of disjunctive $\dat$ programs is
nonmonotonicity, which calls for nondeterministic implementations,
such as backtracking search. A distinguishing characteristic of the new method is
that the optimization can be exploited also during the nondeterministic
phase.  In particular, after some assumptions have been made during the computation,
parts of the program may become irrelevant to a query under these
assumptions. This allows for dynamic pruning of the search space. In
contrast, the effect of the previously defined Magic Set methods
for disjunctive $\dat$
is limited to the deterministic portion of the process.
In this way, the potential performance gain by using the proposed method can be exponential, as could be observed empirically.

The correctness of the method is established and proved in a formal way
thanks to a strong relationship between Magic Sets and unfounded
sets that has not been studied in the literature before. This
knowledge allows for extending the method and the correctness proof
also to programs with stratified negation in a natural way.

The proposed method has been implemented in the \dlv system and
various experiments on synthetic as well as on real-world
data have been conducted. The experimental results on synthetic data confirm the utility of
Magic Sets for disjunctive $\dat$, and they highlight the
computational gain that may be obtained by the new method with respect to
the previously proposed Magic Set method
for disjunctive $\dat$ programs.
Further experiments on data taken from a real-life application show the
benefits of the Magic Set method within an application scenario that
has received considerable attention in recent years, the problem of
answering user queries over possibly inconsistent databases
originating from integration of autonomous sources of information.
\end{abstract}

\begin{keyword}
Logic Programming \sep Stable Models \sep Magic Sets \sep Answer Set Programming \sep Data Integration
\end{keyword}
\end{frontmatter}

\section{Introduction}\label{sec:introduction}

Disjunctive $\dat$ is a language that has been proposed for modeling incomplete data \cite{lobo-etal-92}.
Together with a light version of negation, in this paper stratified negation, this language can in fact express
any query of the complexity class $\Sigma^P_2$ (i.e., $\NP^{\NP}$) \cite{eite-etal-97f},
under the stable model semantics.
It turns out that disjunctive $\dat$ with stratified negation is strictly more expressive (unless the polynomial hierarchy collapses to its first level) than normal logic
programming (i.e., non-disjunctive Datalog with unstratified negation), as the latter can express ``only'' queries
in $\NP$.
As shown in \cite{eite-etal-97f}, the high expressive power of disjunctive $\dat$ has also some positive practical
implications in terms of modelling knowledge, since many problems in NP can be represented more simply
and naturally in stratified disjunctive Datalog than in normal logic programming.
For this reason, it is not surprising that disjunctive $\dat$ has found several real-world applications
\cite{leon-etal-2005,manna-etal-2011-tldks,manna-etal-2011-jcss,ricca-etal-2010-IDUM,ricca-etal-2011-tplp},
also encouraged by the availability of some efficient inference engines, such as \dlv{}
\cite{leon-etal-2002-dlv}, GnT \cite{janh-etal-2005-tocl}, Cmodels \cite{lier-2005-lpnmr}, or ClaspD \cite{dres-etal-2008-KR}. As a matter of fact, these systems are
continuously enhanced to support novel optimization strategies, enabling them to be effective over increasingly
larger application domains. In this paper, we contribute to this
development by providing a novel optimization technique, inspired by
deductive database optimization techniques, in particular
the Magic Set method \cite{banc-etal-86,beer-rama-91,ullm-89}.

The goal of the original Magic Set method (defined for non-disjunctive
$\dat$ programs) is to exploit the presence of constants in a query
for restricting the possible search space by considering only a subset
of a hypothetical program instantiation that is sufficient to answer the
query in question. In order to do this, a top-down computation for
answering the query is simulated in an abstract way. This top-down
simulation is then encoded by means of rules, defining new Magic Set
predicates. The extensions of these predicates (sets of ground atoms)
will contain the tuples that are calculated during a top-down
computation. These predicates are inserted into the original program
rules and can then be used by bottom-up computations to narrow the
computation to what is needed for answering the query.

Extending these ideas to disjunctive $\dat$ faces a major challenge:
While non-disjunctive $\dat$ programs are deterministic, which in
terms of the stable model semantics means that any non-disjunctive
$\dat$ program has exactly one stable model, disjunctive $\dat$
programs are nondeterministic in the sense that they may have multiple
stable models. Of course, the main goal is still isolating a subset of
a hypothetical program instantiation, upon which the considered query
will be evaluated in an equivalent way. There are two basic
possibilities how this nondeterminism can be dealt with in the context
of Magic Sets: The first is to consider \emph{static} Magic Sets, in
the sense that the definition of the Magic Sets is still
deterministic, and therefore the extension of the Magic Set predicates
is equal in each stable model. This static behavior is automatic for
Magic Sets of non-disjunctive $\dat$ programs. The second possibility
is to allow \emph{dynamic} Magic Sets, which also introduce
non-deterministic definitions of Magic Sets. This means that the
extension of the Magic Set predicates may differ in various stable
models, and thus can be viewed as being specialized for each stable
model.

While the nature of dynamic Magic Sets intuitively seems to be more
fitting for disjunctive $\dat$ than static Magic Sets, considering the
architecture of modern reasoning systems for disjunctive $\dat$
substantiates this intuition: These systems work in two phases, which
may be considered as a deterministic (grounding) and a
non-deterministic (model search) part. The interface between these two
is by means of a ground program, which is produced by the
deterministic phase. Static Magic Sets will almost exclusively have an
impact on the grounding phase, while dynamic Magic Sets also have the
possibility to influence the model search phase. In particular, some
assumptions made during the model search may render parts of the
program irrelevant to the query, which may be captured by dynamic
Magic Sets, but not (or only under very specific circumstances) by
static Magic Sets.

In the literature, apart from our own work in
\cite{cumb-etal-2004-iclp}, there is only one previous attempt for
defining a Magic Set method for disjunctive $\dat$, reported in
\cite{grec-99,grec-2003}, which will be referred to as Static Magic Sets ($\SMS$) in this work. The basic idea of $\SMS$ is
that bindings need to be propagated not only from rule heads to rule
bodies (as in traditional Magic Sets), but also from one head
predicate to other head predicates. In addition to producing
definitions for the predicates defining Magic Sets, the method also
introduces additional auxiliary predicates called \emph{collecting}
predicates. These collecting predicates however have a peculiar
effect: Their use keeps the Magic Sets \emph{static}. Indeed, both
magic and collecting predicates are guaranteed to have deterministic
definitions, which implies that disjunctive $\dat$ systems can exploit
the Magic Sets only during the grounding phase. Most systems will
actually produce a ground program which does contain neither magic nor
collecting predicates.

In this article, we propose a \emph{dynamic} Magic Set method for
disjunctive $\dat$ with stratified negation under the stable model
semantics, provide an implementation of it in the system \dlv, and
report on an extensive experimental evaluation. In more detail, the
contributions are:

\vspace{-4mm}\begin{description}

\item[$\blacktriangleright$]

  We present a dynamic Magic Set method for disjunctive $\dat$
  programs with stratified negation, referred to as Dynamic Magic
  Sets ($\DMS$). Different from the previously proposed static method
  $\SMS$, existing systems can exploit the information provided by the
  Magic Sets also during their nondeterministic model search phase.
  This feature allows for potentially exponential performance gains
  with respect to the previously proposed static method.

 \item[$\blacktriangleright$]

  We formally establish the correctness of $\DMS$. In particular, we
  prove that the program obtained by the transformation $\DMS$ is
  query-equivalent to the original program. This result holds for both
  brave and cautious reasoning.

 \item[$\blacktriangleright$]

  We highlight a strong relationship between Magic Sets and unfounded
  sets, which characterize stable models. We can show that the atoms
  which are relevant for answering a query are either true or form an
  unfounded set, which eventually allows us to prove the
  query-equivalence results.

 \item[$\blacktriangleright$]

  Our results hold for a disjunctive $\dat$ language with stratified
  negation under the stable model semantics. In the literature,
  several works deal with non-disjunctive $\dat$ with stratified
  negation under the well-founded or the perfect model
  semantics, which are special cases of our language. For the static
  method $\SMS$, an extension to disjunctive $\dat$ with stratified
  negation has previously only been sketched in \cite{grec-2003}.

 \item[$\blacktriangleright$]

  We have implemented a $\DMS$ optimization module inside the \dlv
  system \cite{leon-etal-2002-dlv}. In this way, we could exploit the
  internal data-structures of the \dlv system and embed $\DMS$ in the
  core of \dlv. As a result, the technique is completely transparent
  to the end user. The system is available at \url{http://www.dlvsystem.com/magic/}.

 \item[$\blacktriangleright$]

  We have conducted extensive experiments on synthetic domains that
  highlight the potential of $\DMS$. We have compared the performance
  of the \dlv system without Magic Set optimization with $\SMS$ and
  with $\DMS$. The results show that in many cases the Magic Set
  methods yield a significant performance benefit. Moreover, we can
  show that the dynamic method $\DMS$ can yield drastically better
  performance than the static $\SMS$. Importantly, in cases in which
  $\DMS$ cannot be beneficial (if all or most of the instantiated
  program is relevant for answering a query), the overhead incurred is
  very light.

 \item[$\blacktriangleright$]

  We also report on experiments which evaluate the impact of $\DMS$
  on an industrial application scenario on real-world data. The
  application involves data integration and builds on several results
  in the literature (for example
  \cite{ArBC00,BaBe02,BrBe03,CaLR03b,chom-marc-2004b,GrGZ01}), which
  transform the problem of query answering over inconsistent databases
  (in this context stemming from integrating autonomous data sources)
  into query answering over disjunctive $\dat$ programs. By leveraging
  these results, $\DMS$ can be viewed as a query optimization method
  for inconsistent databases or for data integration systems. The
  results show that $\DMS$ can yield significant performance gains
  for queries of this application.

\end{description}

\paragraph*{Organization.}

The main body of this article is organized as follows. In
Section~\ref{sec:preliminaries}, preliminaries on disjunctive $\dat$
and on the Magic Set method for non-disjunctive $\dat$ queries are
introduced. Subsequently, in
Section~\ref{sec:magic_set-disjunctive_datalog} the extension $\DMS$
for the case of disjunctive $\dat$ programs is presented, and we show its
correctness.  In Section~\ref{sec:system} we discuss the
implementation and integration of the Magic Set method within the \dlv
system.  Experimental results on synthetic benchmarks are reported in
Section~\ref{sec:experiments}, while the application to data
integration and its experimental evaluation is discussed in
Section~\ref{sec:dataintegration}.  Finally, related work is
discussed in Section~\ref{sec:relatedwork}, and in
Section~\ref{sec:conclusion} we draw our conclusions.

\section{Preliminaries}\label{sec:preliminaries}

In this section, (disjunctive) $\dat$ programs with (stratified) negation are briefly described,
and the standard Magic Set method is
presented together with the notion of sideways information passing strategy (SIPS) for $\dat$ rules.

\subsection{Disjunctive $\dat$ Programs with Stratified Negation}

In this paper, we adopt the standard $\dat$ name convention: Alphanumeric strings starting with a lowercase
character are predicate or constant symbols, while alphanumeric strings starting with an uppercase character are
variable symbols; moreover, we allow the use of positive integer constant symbols. Each predicate symbol is associated
with a non-negative integer, referred to as its \emph{arity}.
An \emph{atom} $p(\t)$ is composed of a predicate symbol $p$ and a list $\t$= $t_1,\ldots,t_k
\ (k\geq 0) $
of terms, each of which is either a constant or a variable.
A \emph{literal} is an atom $p(\t)$ or a negated atom $\naf\ p(\t)$;
in the first case the literal is \emph{positive}, while in the second it is \emph{negative}.

A \emph{disjunctive $\dat$ rule with negation} (short: $\datdnr$ rule)
\R{} is of the form \vspace{-2mm}

\begin{dlvcode}
p_1(\t_1) \ \Or\ \cdots \ \Or\ p_n(\t_n) \derives
    q_1(\s_1),\ \ldots,\ q_j(\s_j), \\
    \quad\quad\quad\quad\quad\quad\quad\quad\quad\quad\quad\ \
    \naf~q_{j+1}(\s_{j+1}),\ \ldots,\ \naf~q_m(\s_m).
\end{dlvcode}

\vspace{-3mm}where $p_1(\t_1), \ldots,\ p_n(\t_n),\ q_1(\s_1), \ldots,\ q_m(\s_m)$
are atoms and $n\geq 1,$  $m\geq j\geq 0$.  The
disjunction $p_1(\t_1) \ \Or\ \cdots \ \Or\ p_n(\t_n)$ is the {\em head} of \R{},
while the conjunction $q_1(\s_1),\ \ldots,\ q_j(\s_j),\ \naf~q_{j+1}(\s_{j+1}),\ \ldots,\ \naf~q_m(\s_m)$ is the {\em body} of~\R{}.
Moreover, $\HR$ denotes the set of head atoms, while $\BR$ denotes the set of body literals.
We also use $\posbody{\R}$ and $\negbody{\R}$ for denoting
the sets of atoms appearing in positive and negative body literals, respectively.
If $\R$ is disjunction-free, that is $n = 1$, and negation-free,
that is $\negbody{\R}$ is empty, then we say that \R{} is a $\dat$ rule;
if $\posbody{\R}$ is empty in addition,
then we say that \R{} is a \emph{fact}. A {\em disjunctive $\dat$ program}
$\p$ is a finite set of rules; if all the rules in it are disjunction- and negation-free,
then $\P$ is a (standard) $\dat$ program.

Given a $\datdnr$ program $\P$, a predicate belongs to the {\em Intensional Database} (IDB)
if it is either in the head of a rule with non-empty body, or in the head of a
disjunctive rule; otherwise, it belongs to the
{\em Extensional Database} (EDB). The set of rules having IDB predicates in their heads
is denoted by $\IDB{\p}$, while $\EDB{\p}$ denotes the remaining rules, that is,
$\EDB{\p} = \P \setminus \IDB{\p}$. 
For simplicity, we assume that predicates will always be of the same type (EDB or IDB) in any program.

The set of all constants appearing in a program $\P$ is the \emph{universe} of $\P$ and is denoted by $\UP$,%
\footnote{
If $\p$ has no constants, an arbitrary constant is added to $\UP$.}
while the set of ground atoms constructable from predicates in $\P$ with constants in $\UP$ is the \emph{base}
of $\P$, denoted by $\BP$. We call an atom (rule, or program) \emph{ground} if it does not contain any
variables.
A substitution $\vartheta$ is a function from variables to elements of \UP.
For an expression $S$ (atom, literal, rule), by $S\vartheta$ we denote the expression obtained from $S$ by substituting all occurrences of each variable $X$ in $S$ with $\vartheta(X)$.
A ground atom $p(\t)$ (resp.\ ground rule $\R_g$) is
an instance of an atom $p(\t')$ (resp.\ rule $\R$) if there is a
substitution $\vartheta$ from the variables in $p(\t')$ (resp.\ in $\R$)
to $\UP$ such that ${p(\t)} = {p(\t')}\vartheta$
(resp.\ $\R_g = \R\vartheta$).
Given a program $\p$, $\ground{\p}$ denotes the set of all possible instances
of rules in $\p$.

Given an atom $p(\t)$ and a set of ground atoms $A$,
by $A|_{p(\t)}$ we denote the set of ground instances of $p(\t)$ belonging to $A$.
For example, $\BP|_{p(\t)}$ is the set of all ground atoms obtained by applying to
$p(\t)$ all the possible substitutions from the variables in $p(\t)$ to $\UP$,
that is, the set of all the instances of $p(\t)$.
Abusing notation, if $B$ is a set of atoms, by $A|_B$ we denote the union
of all $A|_{p(\t)}$, for each ${p(\t)} \in B$.

A desirable property of $\datdnr$ programs is {\em safety}. A $\datdnr$ rule $\R$ is safe
if each variable appearing in $\R$ appears in at least one atom of $\posbody{\R}$.
A $\datdnr$ program is safe if all its rules are safe.
Moreover, programs without recursion over negated literals constitute an interesting class of
$\datdnr$ programs. Without going into details, a predicate $p$ in the head of a rule $\R$
{\em depends} on all the predicates $q$ in the body of $\R$; 
$p$ depends on $q$ positively if $q$ appears in $\posbody{\R}$, and
$p$ depends on $q$ negatively if $q$ appears in $\negbody{\R}$. A program has
recursion over negation if a cycle of dependencies with at least one
negative dependency exists. If a program has no recursion over negation, then the
program is stratified (short: $\datdn$).
In this work only safe programs without recursion over negation are considered.

An \emph{interpretation} for a program $\P$ is a subset $I$ of $\BP$. A positive ground
literal $p(\t)$ is true with respect to an
interpretation $I$ if ${p(\t)}\in I$; otherwise, it is false.
A negative ground literal $\naf\ p(\t)$ is true with respect to $I$
if and only if $p(\t)$ is false with respect to $I$, that is, if and only if ${p(\t)} \not\in I$.
The body of a ground rule $\R$ is true with respect to $I$
if and only if all the body literals of $\R$ are true with respect to $I$, that is,
if and only if $\posbody{\R} \subseteq I$ and $\negbody{\R} \cap I = \emptyset$.
An interpretation $I$ {\em satisfies} a ground rule $\R\in \GP$ if at least one atom
in $\head{\R}$ is true with respect to $I$ whenever the body of $\R$ is true with respect to $I$. An interpretation $I$ is a
\emph{model} of a $\datdnr$ program $\P$ if $I$ satisfies all the rules in $\GP$.
Since an interpretation is a set of atoms, if $I$ is an interpretation
for a program $\p$, and $\p'$ is another program,
then by $I|_{B_{\p'}}$ we denote the restriction of $I$ to the
base of $\p'$.

Given an interpretation $I$ for a program $\P$, the reduct of $\P$ with respect to $I$,
denoted by $\ground{\p}^{I}$, is obtained by deleting from $\GP$ all
the rules $\R_g$ with $\negbody{\R_g} \cap I \neq \emptyset$,
and then by removing all the negative literals from the remaining rules.

The semantics of a $\datdnr$ program $\P$ is given by the set $\SM(\P)$ of stable models of
$\P$, where an interpretation $M$ is a stable model for $\P$ if and only if
$M$ is a subset-minimal model of $\ground{\P}^M$.
 It is well-known that there is exactly one stable model for any $\dat$ program,
also in presence of stratified negation. However, for a $\datdn$
program $\P$, $|\SM(\P)| \ge 1$ holds
($\datdnr$ programs, instead, can also have no stable model).

Given a ground atom ${p(\t)}$ and a $\datdnr$ program $\P$, $p(\t)$ is a \emph{cautious} (or \emph{certain})
consequence of $\P$, denoted by $\P \ccons p(\t)$, if ${p(\t)} \in M$ for each $M \in \SM(\P)$; $p(\t)$ is a
\emph{brave} (or \emph{possible}) consequence of $\P$, denoted by $\P \bcons p(\t)$, if ${p(\t)} \in M$
for some $M \in \SM(\P)$. Note that brave and cautious consequences coincide for
$\dat$ programs, as these programs have a unique stable model.
Moreover, cautious consequences of a $\datdn$ program $\P$ are also brave
consequences of $\P$ because $|\SM(\P)| \ge 1$ holds in this case.

Given a \emph{query} $\Q = {g(\t)?}$ (an atom),%
\footnote{
Note that more complex queries can still be expressed using appropriate rules.
We assume that each constant appearing in $\Q$ also appears in $\p$;
if this is not the case, then we can add to $\p$ a fact $p(\t)$
such that $p$ is a predicate not occurring in $\p$
and $\t$ are the arguments of $\Q$.
Question marks will be usually omitted when referring to queries in the text.}
$\Ans_c(\Q,\P)$ denotes the set of all substitutions $\vartheta$
for the variables of ${g(\t)}$
such that $\P \ccons {g(\t)}\vartheta$, while $\Ans_b(\Q,\P)$ denotes the set of substitutions $\vartheta$
for the variables of ${g(\t)}$ such that $\P \bcons {g(\t)}\vartheta$.

Let $\P$ and $\P'$ be two $\datdnr$ programs and $\Q$ a query. Then $\P$ and $\P'$ are \emph{brave-equivalent}
with respect to $\Q$, denoted by $\P\bqequiv{\Q} \P'$, if $\Ans_b(\Q,\P \cup \F) = \Ans_b(\Q,\P' \cup \F)$ is guaranteed for each
set of facts $\F$ defined over predicates which are EDB predicates of $\p$ or $\p'$;
similarly, $\P$ and $\P'$ are \emph{cautious-equivalent} with respect to $\Q$, denoted by $\P\cqequiv{\Q} \P'$, if
$\Ans_c(\Q,\P \cup \F) = \Ans_c(\Q,\P' \cup \F)$ is guaranteed for each set of facts $\F$
defined over predicates which are EDB predicates of $\p$ or $\p'$.

\subsection{Bottom-up Disjunctive Datalog Computation}

Many $\datdnr$ systems implement a two-phase computation.
The first phase, referred to as {\em program instantiation} or {\em grounding}, is bottom-up.
For an input program $\p$, it produces a ground program which is equivalent to
$\GP$, but significantly smaller.
Most of the techniques used in this phase stem from bottom-up
methods developed for classic and deductive databases;
see for example \cite{abit-etal-95} or
\cite{gebs-etal-2007-lpnmr,leon-etal-2002-dlv} for details.
Essentially, predicate instances which are known to be true or known to be false
are identified and this knowledge is used for deriving further
instances of this kind. Eventually, the truth values obtained in this
way are used to produce rule instances which are not satisfied already.
It is important to note that this phase
behaves in a deterministic way with respect to stable models. No
assumptions about truth or falsity of atoms are made, only definite
knowledge is derived, which must hold in all stable models. For this
reason, programs with multiple stable models cannot be solved by
grounding.

The second phase is often referred to as \emph{stable model search} and takes
care of the non-deterministic computation. Essentially, one undefined atom is selected
and its truth or falsity is assumed. The assumption might imply truth or
falsity of other undefined atoms. Hence, the process is
repeated until either an inconsistency is derived or all atoms have been interpreted.
In the latter case an additional check is performed to ensure stability of the model.
Details on this process can be found for example in \cite{fabe-2002}.
Query answering is typically handled by storing all admissible answer substitutions
as stable models are computed. For brave reasoning, each stable model can
contribute substitutions to the set of answers. In this case the set of answers
is initially empty. For cautious reasoning, instead, each stable model may
eliminate some substitutions from the set of admissible answers. Therefore, in
this case all possible substitutions for the input query are initially contained
in the set of answers.

\subsection{Sideways Information Passing for $\dat$ Rules}
\label{sec:sip}

The Magic Set method aims at simulate a top-down evaluation of a query $\Q$, like for instance the one adopted
by Prolog. According to this kind of evaluation, all the rules $\R$ such that
${p(\t)} \in \HR$ and $\HR\vartheta = \{\Q\vartheta'\}$ (for some
substitution $\vartheta$ for all the variables of $\R$ and some
substitution $\vartheta'$ for all the variables of $\Q$)
are considered in a first step. Then the atoms in $\posbody{\R}\vartheta$ are taken as subqueries
(we recall that standard $\dat$ rules have empty negative body), and
the procedure is iterated.
Note that, according to this process, if a (sub)query has some
argument that is \emph{bound} to a constant value, this
information is ``passed'' to the atoms in the body. Moreover, the body
is considered to be processed in a certain sequence, and processing a
body atom may bind some of its arguments for subsequently considered
body atoms, thus ``generating'' and ``passing'' bindings within the
body. Whenever a body atom is processed, each of its argument is
therefore considered to be either \emph{bound} or \emph{free}. We
illustrate this mechanism by means of an example.

\begin{example}\label{ex:sip1}
\em Let $\tt path(1,5)$ be a query for a program having the
following inference rules:

{
\small
\[
\begin{array}{l}
\vspace{-2mm} \R_1: \quad \tt path(X,Y) \derives edge(X,Y).\\
\vspace{-2mm} \R_2: \quad \tt path(X,Y) \derives edge(X,Z),\ path(Z,Y).
\end{array}
\]
}
Since this is a $\dat$ program, brave and cautious consequences coincide.
Moreover, let $\F_1 = \{ {\tt edge(1,3), edge(2,4), edge(3,5)} \}$ be the EDB of the program. A top-down evaluation scheme
considers $\R_1$ and $\R_2$ with $\tt X$ and $\tt Y$ bound to $\tt 1$ and $\tt 5$, respectively. In particular,
when considering $\R_1$, the information about the binding of the two variables is passed to $\tt edge(X,Y)$,
which is indeed the only query atom occurring in $\R_1$. Thus, the evaluation fails since $\tt edge(1,5)$ does
not occur in $\F_1$.

When considering $\R_2$, instead, the binding information can be passed either to $\tt path(Z,Y)$ or to $\tt
edge(X,Z)$. Suppose that atoms are evaluated according to their ordering in the rule
(from left to right); then $\tt edge(X,Z)$ is
considered before $\tt path(Z,Y)$. In particular, $\F_1$ contains the atom $\tt edge(1,3)$, which leads us
to map $\tt Z$ to $\tt 3$. Eventually, this inferred binding information might be propagated to the remaining
body atom $\tt path(Z,Y)$, which hence becomes $\tt path(3,5)$.

The process has now to be repeated by looking for an answer to $\tt path(3,5)$. Again, rule $\R_1$ can be
considered, from which we conclude that this query is true since $\tt edge(3,5)$ occurs in $\F_1$. Thus, $\tt
path(1,5)$ holds as well due to $\R_2$. \hfill $\Box$
\end{example}

Note that in the example above we have two degrees of freedom in the specification of the top-down evaluation
scheme. The first one concerns which ordering is used for processing the body
atoms. While Prolog systems are usually required to follow the ordering in which the program is written, $\dat$ has a purely declarative semantics which is independent of the body ordering, allowing for an arbitrary ordering to be adopted. The second degree of freedom is slightly more subtle, and concerns the selection of the terms to be
considered bound to constants from previous evaluations. Indeed, while we have considered the propagation of all
the binding information that originates from previously processed body atoms, it is in general possible to
restrict the top-down evaluation to partially propagate this information. For instance, one may desire to
propagate only information generated from the evaluation of EDB predicates, or even just the information that is passed on via the head atom.

The specific propagation strategy adopted in the top-down evaluation scheme is called {\em sideways information
passing strategy} (SIPS), which is just a way of formalizing a partial ordering over the atoms of each rule
together with the specification of how the bindings originated and propagate
\cite{beer-rama-91,grec-2003}.
To formalize this concept, in what follows, for each IDB atom $p(\t)$, we shall denote its associated
binding information (originated in a certain step of the top-down evaluation) by means of a string $\alpha$
built over the letters $b$ and $f$, denoting ``bound'' and ``free'', respectively, for each argument of
$p(\t)$.

\begin{definition}[SIPS for $\dat$ rules]
\label{def:sip}
A {\em SIPS} for a $\dat$ rule $\R$ with respect to a binding $\alpha$
for the atom ${p(\t)} \in \HR$ is a pair $(\prec^{\alpha}_r,f^{\alpha}_r)$, where:
\begin{enumerate}
  \item $\prec^{\alpha}_r$ is a strict partial order over the atoms in $\HR \cup \posbody{\R}$,
  such that ${p(\t)} \prec^{\alpha}_r {q(\s)}$, for all atoms ${q(\s)}\in \posbody{\R}$; and,

  \item  $f^{\alpha}_r$ is a function assigning to each atom ${q(\s)}\in \HR \cup \posbody{\R}$
  a subset of the variables in $\s$---intuitively,
  those made bound when processing ${q(\s)}$.
\end{enumerate}
\end{definition}

Intuitively, for each atom $q(\s)$ occurring in $\R$, the strict partial
order $\prec^{\alpha}_r$ specifies those atoms that have to be processed
before processing atom $q(\s)$. Eventually, an argument $X$ of $q(\s)$
is bound to a constant if there exists an atom $q'(\s')$ such that
${q'(\s')} \prec^{\alpha}_r {q(\s)}$ and
${X} \in f^{\alpha}_r({q'(\s')})$.
Note that the head atom $p(\t)$ precedes
all other atoms in $\prec^{\alpha}_r$.

\begin{example}\label{ex:sip2}\em
The SIPS we have adopted in Example~\ref{ex:sip1} for $\R_1$ with respect to the binding $\tt bb$ (originating from the
query $\tt path(1,5)$) can be formalized as the pair $(\prec^{\tt bb}_{\R_1},f^{\tt bb}_{\R_1})$, where
${\tt path(X,Y)}\prec^{\tt bb}_{\R_1} {\tt edge(X,Y)}$,
$f^{\tt bb}_{\R_1}({\tt path(X,Y)})=\{{\tt X,Y}\}$, and
$f^{\tt bb}_{\R_1}({\tt edge(X,Y)})=\emptyset$.
Instead, the SIPS we have adopted for $\R_2$ with respect to the binding $\tt bb$ can be formalized
as the pair $(\prec^{\tt bb}_{\R_2},f^{\tt bb}_{\R_2})$, where ${\tt path(X,Y)} \prec^{\tt bb}_{\R_2} {\tt
edge(X,Z)} \prec^{\tt bb}_{\R_2} {\tt path(Z,Y)}, f^{\tt bb}_{\R_2}({\tt path(X,Y)}) = \{{\tt X,Y}\}, f^{\tt
bb}_{\R_2}({\tt edge(X,Z)}) = \{{\tt Z}\}$, and $f^{\tt bb}_{\R_2}({\tt path(Z,Y)}) = \emptyset$.
%
\hfill $\Box$
\end{example}

All the algorithms and techniques we shall develop in this paper are orthogonal with respect to the underlying SIPSes to
be used in the top-down evaluation. Thus, in Section~\ref{sec:msdat}, we shall assume that $\dat$ programs are provided in
input together with some arbitrarily defined SIPS $(\prec^{\alpha}_r,f^{\alpha}_r)$, for each rule $r$
and for each possible adornment $\alpha$ for the head atom in $\HR$.

\subsection{Magic Sets for $\dat$ Programs}
\label{sec:msdat}

The Magic Set method is a strategy for simulating the top-down evaluation of a query by modifying the original
program by means of additional rules, which narrow the computation to what is relevant for answering the query.
We next provide a brief and informal description of the Magic Set rewriting technique. The reader is referred to
\cite{ullm-89} for a detailed presentation.

The method is structured in four main phases, which are informally illustrated below by means of
Example~\ref{ex:sip1}.

\textbf{(1) Adornment.} The key idea is to materialize the binding information for IDB predicates that would be
propagated during a top-down computation.
In particular, the fact that an IDB predicate  $p(\t)$ is associated with a binding information $\alpha$
(i.e., a string over the letters $b$ and $f$, one for each term in $\t$) is denoted by the atom obtained
\emph{adorning} the predicate symbol with the binding at hand, that is, by $p^{\alpha}(\t)$. In what
follows, the predicate $p^{\alpha}$ is said to be an adorned predicate.

First, adornments are created for query predicates so that an argument occurring in the query is adorned with
the letter $b$ if it is a constant, or with the letter $f$ if it is a variable. For instance, the
adorned version of the query atom $\tt path(1,5)$ is $\tt path^{\tt bb}(1,5)$, which gives rise to the adorned
predicate $\tt path^{\tt bb}$.

Each adorned predicate is eventually used to propagate its information into the body of the rules defining it
according to a SIPS, thereby simulating a top-down evaluation. In particular, assume that the binding $\alpha$ has to be propagated into a rule $r$ whose head is $p(\t)$. Thus, the associated SIPS $(\prec^{\alpha}_r,f^{\alpha}_r)$ determines which variables will be bound in the evaluation of the
various body atoms. Indeed, a variable $X$ of an atom $q(\s)$ in $\R$ is bound if and only if either

\begin{enumerate}
\item ${X}\in f^{\alpha}_r({q(\s)})$ with ${q(\s)} = {p(\t)}$; or,

\item ${X}\in f^{\alpha}_r({b(\bar z)})$ for an atom ${b(\bar z)}\in \posbody{\R}$ such that
${b(\bar z)} \prec^{\alpha}_r {q(\s)}$ holds.
\end{enumerate}

Adorning a rule $r$ with respect to an adorned predicate $p^{\alpha}$ means propagating the binding information $\alpha$, starting from the head predicate ${p(\t)}\in \HR$,
thereby creating
a novel \emph{adorned rule} where all the IDB predicates in $r$ are substituted by the adorned
predicates originating from the binding according to (1) and (2).

\begin{example}\label{ex:adornament}\em
Adorning the query $\tt path(1,5)$ generates $\tt path^{\tt bb}(1,5)$. Then, propagating the
binding information $\tt bb$ into the rule $\R_1$, i.e., when adorning $\R_1$ with $\tt path^{\tt bb}$, produces the following adorned rule (recall here that adornments apply only to IDB predicates, whereas $\tt edge$
is an EDB predicate): \vspace{-3mm} {\small
\[
\begin{array}{l}
\vspace{-2mm} \R_1^a: \quad \tt path^{\tt bb}(X,Y) \derives edge(X,Y).\\
\end{array}
\]}

Instead, when propagating $\tt bb$ into the rule $\R_2$ according to the SIPS
$(\prec^{\tt bb}_{\R_2},f^{\tt bb}_{\R_2})$ defined in Example~\ref{ex:sip2},
we obtain the following adorned rule: \vspace{-3mm} {\small
\[
\begin{array}{l}
\vspace{-2mm} \R_2^a: \quad \tt path^{\tt bb}(X,Y) \derives edge(X,Z),\ path^{bb}(Z,Y).
\end{array}
\]}
\vspace{-10mm}\hfill $\Box$
\end{example}

While adorning rules, novel binding information in the form of yet unseen adorned predicates may be generated, which should be used
for adorning other rules. In fact, the adornment step is repeated until all bindings have been processed,
yielding the \emph{adorned program}, which is the set of all adorned rules created during the computation. For
instance, in the above example, the adorned program just consists of $\R_1^a$ and $\R_2^a$ for no adorned
predicate different from $\tt path^{\tt bb}$ is generated.

\textbf{(2) Generation.} In the second step of the Magic Set method, the adorned program is used to generate
{\em magic rules}, which are used to simulate the top-down evaluation scheme and
to single out the atoms relevant for answer the input query.
For an adorned atom $p^\alpha(\bar t)$, let $\magic{p^\alpha(\bar t)}$ be its \emph{magic version}
defined as the atom $magic\_p^\alpha(\bar t')$, where $\bar t'$ is obtained from $\bar t$ by
eliminating all arguments corresponding to an $f$ label in $\alpha$, and where $magic\_p^\alpha$ is
a new predicate symbol (for simplicity denoted by attaching the prefix ``$magic\_$'' to the predicate symbol $p^\alpha$).
Intuitively, $magic\_p^\alpha(\bar t')\vartheta$ ($\vartheta$ a substitution)
is inferred by the rules of the rewritten program whenever a top-down evaluation
of the original program would process a subquery of the form $p^\alpha(\bar t'')$,
where $\bar t''$ is obtained from $\bar t$ by applying $\vartheta$ to
all terms in $\bar t'$.

Thus, if $q_i^{\beta_i}(\s_i)$ is an adorned atom
(i.e., $\beta_i$ is not the empty string)
in the body of an adorned rule $\R^a$ having $p^\alpha(\t)$ in head, a
magic rule $\R^\magica$ is generated such that (i) $\head{\R^\magica} = \{\magic{q_i^{\beta_i}(\s_i)}\}$ and (ii) $\body{\R^\magica}$
is the union of $\{\magic{p^\alpha(\t)}\}$ and the set of all the atoms
${q_j^{\beta_j}(\s_j)} \in \posbody{\R}$ such that ${q_j(\s_j)} \prec^{\alpha}_r {q_i(\s_i)}$.

\begin{example}\label{ex:magic1}\em
In our running example, only one magic rule is generated,

\vspace{-3mm}
{\small
\[
\begin{array}{l}
\vspace{-2mm} \R_2^\magica: \quad \tt magic\_path^{\tt bb}(Z,Y) \derives magic\_path^{\tt bb}(X,Y),\ edge(X,Z).\\
\end{array}
\]}

\vspace{-5mm} In fact, the adorned rule $\R_1^a$ does not produce any magic rule, since there is no adorned
predicate in $\posbody{\R_1^a}$.\hfill $\Box$
\end{example}

\textbf{(3) Modification.} The adorned rules are subsequently modified by adding magic atoms to their bodies.
These magic atoms limit the range of the head variables avoiding the inference of facts which cannot contribute
to the derivation of the query. In particular, each adorned rule $\R^a$, whose head atom is $p^\alpha(\t)$, is
modified by adding the atom $\magic{p^\alpha(\t)}$ to its body. The resulting rules are called {\em modified
rules}.

\begin{example}\label{ex:magic2}\em
In our running example, the following modified rules are generated:

\vspace{-3mm}
{\small
\[
\begin{array}{l}
\vspace{-2mm} \R_1': \quad \tt path^{\tt bb}(X,Y) \derives magic\_path^{\tt bb}(X,Y),\ edge(X,Y).\\
\vspace{-2mm} \R_2': \quad \tt path^{\tt bb}(X,Y) \derives magic\_path^{\tt bb}(X,Y),\ edge(X,Z),\ path^{\tt bb}(Z,Y).
\end{array}
\]}

\vspace{-10mm}\hfill $\Box$
\end{example}

\textbf{(4) Processing the Query.} Finally, given the adorned predicate $g^{\alpha}$ obtained when adorning a query $g(\t)$,
(1) a \emph{magic seed} $\magic{g^{\alpha}(\t)}$ (a fact) and (2) a rule $g(\t) \derives g^{\alpha}(\t)$ are produced.
In our example, $\tt magic\_path^{\tt bb}(1,5)$ and $\tt path(X,Y) \derives path^{\tt bb}(X,Y)$ are generated.

The complete rewritten program according to the Magic Set method consists of the magic, modified, and query
rules (together with the original EDB).  Given a $\dat$ program $\P$, a query $\Q$, and the rewritten program $\P'$, it is well-known that $\P$
and $\P'$ are equivalent with respect to $\Q$, i.e., $\P\bqequiv{\Q} {\P'}$ and $\P\cqequiv{\Q} {\P'}$ hold
\cite{ullm-89}.

\begin{example}\em
The complete rewriting of our running example is as follows:\footnote{%
The Magic Set rewriting of a program $\p$ affects only $\IDB{\p}$, so we usually
omit $\EDB{\p}$ in examples.}

\vspace{-3mm}
{\small
\[
\begin{array}{l}
\vspace{-2mm}\phantom{\R_1':}\ \quad \tt magic\_path^{\tt bb}(1,5).\\
\vspace{-2mm}\phantom{\R_1':}\ \quad \tt path(X,Y) \derives path^{\tt bb}(X,Y).\\
\vspace{-2mm}\R_2^\magica: \quad \tt magic\_path^{\tt bb}(Z,Y) \derives magic\_path^{\tt bb}(X,Y), edge(X,Z).\\
\vspace{-2mm}\R_1': \quad \tt path^{\tt bb}(X,Y) \derives magic\_path^{\tt bb}(X,Y),\ edge(X,Y).\\
\vspace{-2mm}\R_2': \quad \tt path^{\tt bb}(X,Y) \derives magic\_path^{\tt bb}(X,Y),\ edge(X,Z),\ path^{\tt bb}(Z,Y).
\end{array}
\]}

\vspace{-5mm} In this rewriting, $\tt magic\_path^{\tt bb}(X,Y)$ represents a potential sub-path of the paths from
$\tt 1$ to $\tt 5$. Therefore, when answering the query, only these sub-paths will be actually considered in the
bottom-up computation.
One can check that this rewriting is in fact equivalent to the original program with respect to the query
$\tt path(1,5)$. \hfill $\Box$
\end{example}

\section{Magic Set Method for $\datdn$ Programs}\label{sec:magic_set-disjunctive_datalog}

In this section we present the Dynamic Magic Set algorithm ($\DMS$) for the optimization of disjunctive
programs with stratified negation. Before discussing the details of the algorithm,
we informally present the main ideas that have been
exploited for enabling the Magic Set method to work on disjunctive programs
(without negation).

\subsection{Overview of Binding Propagation in $\datd$ Programs}\label{sec:intuition}

As first observed in \cite{grec-2003}, while in non-disjunctive programs bindings are propagated only
head-to-body, a Magic Set transformation for disjunctive programs has to propagate bindings also head-to-head in order
to preserve soundness. Roughly, suppose that a predicate $p$ is relevant for the query, and a disjunctive
rule $r$ contains $p(X)$ in the head. Then, besides propagating the binding from $p(X)$ to the body of
$r$ (as in the non-disjunctive case), the binding must also be
propagated from $p(X)$ to the other head atoms of $r$. The reason is that any atom which is true in a stable model needs a supporting rule, which is a rule with a true body and in which the atom in question is the \emph{only} true head atom. Therefore, $r$ can yield support to the
truth of $p(X)$ only if all other head atoms are false, which is due to the implicit minimality
criterion in the semantics.

Consider, for instance, a $\datd$ program $\P$ consisting of the rule $\ \tt p(X)\, \Or\, q(Y)
\derives\, a(X,Y),\, b(X)$, and the query $\tt p(1)$. Even though the query propagates the binding for the
predicate $\tt p$, in order to correctly answer the query we also need to evaluate the truth value of $\tt
q(Y)$, which indirectly receives the binding through the body predicate $\tt a(X,Y)$. For instance, suppose that
the program contains the facts $\tt a(1,2)$ and $\tt b(1)$;
then the atom $\tt q(2)$ is relevant for the query $\tt p(1)$
(i.e., it should belong to the Magic Set of the query), since the truth of $\tt q(2)$ would invalidate the
derivation of $\tt p(1)$ from the above rule, due to the minimality of the semantics. It follows that, while
propagating the binding, the head atoms of disjunctive rules must be all adorned as well.

However, the adornment of the head of one disjunctive rule $r$ may give rise to multiple rules, having different
adornments for the head predicates. This process can be somehow seen as ``splitting'' $r$ into multiple rules.
While this is not a problem in the non-disjunctive case, the semantics of a disjunctive program may be affected.
Consider, for instance, the program consisting of the rule $\ \tt p(X,Y)\ \Or\ q(Y,X) \derives a(X,Y)$, in which $\tt p$ and $\tt q$
are mutually exclusive (due to minimality) since they do not appear in any other rule head. Assuming the
adornments $\tt p^{\tt bf}$ and $\tt q^{\tt bf}$ to be propagated, we might obtain rules whose heads have the form $\tt
p^{\tt bf}(X,Y)\ \Or\ q^{\tt fb}(Y,X)$ (derived while propagating $\tt p^{\tt bf}$) and $\tt p^{\tt fb}(X,Y)\ \Or\ q^{\tt bf}(Y,X)$
(derived while propagating $\tt q^{\tt bf}$). These rules could support two atoms $\tt p^{\tt bf}(m,n)$ and $\tt
q^{\tt bf}(n,m)$, while in the original program $\tt p(m,n)$ and $\tt p(n,m)$ could not hold simultaneously (due to
semantic minimality), thus changing the original semantics.

The method proposed in \cite{grec-2003} circumvents this problem by using some auxiliary predicates that collect
all facts coming from the different adornments. For instance, in the above example, two rules of the form $\tt
collect\_p(X,Y) \derives p^{\tt fb}(X,Y)$ and $\tt collect\_p(X,Y) \derives p^{\tt bf}(X,Y)$ are added for the predicate
$\tt p$.
The main deficiency of this approach is that collecting predicates will store a sizable superset
of all the atoms relevant to answer the given query.

An important observation is that these collecting predicates are defined in a deterministic way.
Since these predicates are used for restricting the computation in \cite{grec-2003}, a consequence is that assumptions during the computation
cannot be exploited for determining the relevant part of the program. In terms of bottom-up
systems, this implies that the optimization affects only the grounding portion of the
solver. Intuitively, it would be beneficial to also have a form of conditional relevance,
exploiting also relevance for assumptions. In fact, in Section~\ref{sec:experiments}, we provide
experimental evidence for this intuition.

In the following, we propose a novel Magic Set method that guarantees query equivalence
and also allows for the exploitation of conditional or dynamic relevance,
overcoming a major drawback of $\SMS$.

\subsection{$\DMS$ Algorithm}
\label{sec:dms}

Our proposal to enhance the Magic Set method for disjunctive $\dat$ programs has two crucial features compared to the one of \cite{grec-2003}:

\begin{enumerate}
  \item  First, the semantics of the program is preserved by stripping off the
  adornments from non-magic predicates in modified rules, and not by introducing collecting
  predicates that can introduce overhead in the grounding process, as discussed in Section~\ref{sec:intuition}.

  \item  Second, the proposed Magic Set technique is not just a way to cut irrelevant
  rules from the ground program; in fact, it allows for dynamic determination
  of relevance, thus optimizing also the nondeterministic computation by disabling parts
  of the programs which are not relevant in any extension of the current computation state.
\end{enumerate}

\begin{figure}[t]
 \centering
 \fbox{\hspace{2mm}\parbox{0.88\textwidth}{\scriptsize
  \begin{description}
  \item[Algorithm DMS($\Q$,$\p$)]
  \item[Input:] A \datdn\ program $\P$, and a query $\Q=g(\t)?$
  \item[Output:] The rewritten program $\DMS(\Q,\P)$;
  \item[var:]  $S$, $D$: \textbf{set} of adorned predicates;\ \modifiedRules$_{\Q,\P}$,\magicRules$_{\Q,\P}$: \textbf{set} of rules;
  \item[begin] \
  \item[]\emph{\phantom{1}1.}\ \ $S$ := $\emptyset$;\ \ $D$ := $\emptyset$;\ \ \modifiedRules$_{\Q,\P}$ := $\emptyset$;\ \ \magicRules$_{\Q,\P}$ := \{\textbf{\emph{BuildQuerySeed}}($\Q,S$)\};
  \item[]\emph{\phantom{1}2.}\ \ \textbf{while} $S\neq \emptyset$ \textbf{do}
  \item[]\emph{\phantom{1}3.}\ \ \hs $p^\alpha$\ := an element of $S$;\hs remove $p^\alpha$ from $S$;\hs add $p^\alpha$ to $D$;
  \item[]\emph{\phantom{1}4.}\ \ \hs \textbf{for each} rule $\R \in \P$ and \textbf{for each} atom $p(\t)$ $\in \HR$ \textbf{do}
  \item[]\emph{\phantom{1}5.}\, \ \hs \hs $\R^{a}$:=\textbf{\emph{Adorn}}$(\R,{p^\alpha(\t)},S,D)$;
  \item[]\emph{\phantom{1}6.}\, \ \hs \hs \magicRules$_{\Q,\P}$\ := \magicRules$_{\Q,\P}$ \ $\cup$ \textbf{\emph{Generate}}$(\R,{p^\alpha(\t)},\R^a)$;
  \item[]\emph{\phantom{1}7.} \ \hs \hs \modifiedRules$_{\Q,\P}$\ := \modifiedRules$_{\Q,\P}$ \ $\cup$ $\{$\,\textbf{\emph{Modify}}$(\R,\R^a)$\,$\}$;
  \item[]\emph{\phantom{1}8.}\ \ \hs \textbf{end for}
  \item[]\emph{\phantom{1}9.}\ \ \textbf{end while}
  \item[]\emph{10.}  $\DMS(\Q,\P)$:=\magicRules$_{\Q,\P}$ \ $\cup$ \modifiedRules$_{\Q,\P}$ $\cup$ \EDB{\p};
  \item[]\emph{11.}  \textbf{return} $\DMS(\Q,\P)$;
  \item[end.] \
  \end{description}
 }}
 \caption{Dynamic Magic Set algorithm ($\DMS$) for $\datdn$ programs}\label{fig:DMS}
\end{figure}

The algorithm $\DMS$ implementing these strategies is reported in Figure~\ref{fig:DMS} as pseudo-code. We assume that all variables are passed to functions by reference, in particular the variable $S$ is modified inside \textbf{\emph{BuildQuerySeed}} and \textbf{\emph{Adorn}}.
Its input is a $\datdn$ program\footnote{Note that the algorithm can be used for non-disjunctive
and/or positive programs as a special case.}
$\P$ and a query $\Q$.
The algorithm uses two sets, $S$ and $D$, to store adorned predicates
to be propagated and already processed, respectively.
After all the adorned predicates have been processed, the method
outputs a rewritten program $\DMS(\Q,\P)$ consisting of a set of
\emph{modified} and \emph{magic} rules, stored by means of the sets
\modifiedRules$_{\Q,\P}$ and \magicRules$_{\Q,\P}$, respectively
(together with the original EDB).
The main steps of the algorithm are illustrated by means of the following running
example.

\vspace{-2mm}\begin{example}[Strategic Companies \cite{cado-etal-97}]\label{running}
\em
Let $C = \{{c_1, \ldots, c_m}\}$ be a collection
of companies producing some goods in a set $G$, such that each company ${c_i}\in C$ is controlled by a set of
other companies $O_{i} \subseteq C$. A subset of the companies $C' \subseteq C$ is a \emph{strategic set} if it is
a minimal set of companies satisfying the following conditions:
Companies in $C'$ produce all the goods in $G$; and
$O_{i} \subseteq C'$ implies ${c_i} \in C'$, for each $i=1,\ldots,m$.

We assume that each product is produced by at most two companies and that each company is controlled by at most three companies. It is known that the problem retains its hardness (for the second level of the polynomial hierarchy; see \cite{cado-etal-97}) under these restrictions.
We assume that production of goods is represented by an EDB containing a fact ${\tt produced\_by(p,c_1,c_2})$ for each product $p$ produced by companies $c_1$ and $c_2$, and that the control is represented by facts ${\tt controlled\_by(c,c_1,c_2,c_3)}$ for each company $c$ controlled by companies $c_1$, $c_2$, and $c_3$.\footnote{If a product is produced by only one company, ${c_2} = {c_1}$, and similarly for companies controlled by fewer than three companies.}
This problem can be modeled via the following disjunctive program
$\mathcal{P}_{sc}$:

\vspace{-2mm}\begin{dlvcode}
\R_3:\ \tt sc(C_1)\ \Or\ sc(C_2) \derives produced\_by(P,C_1,C_2). \\
\R_4:\ \tt sc(C) \derives controlled\_by(C,C_1,C_2,C_3),\ sc(C_1),\ sc(C_2),\ sc(C_3).
\end{dlvcode}
\vspace{-2mm}Moreover, given a company ${c} \in C$, we consider a query $\Q_{sc}=\tt sc(c)$
asking whether $c$ belongs to some strategic set of $C$. \hfill $\Box$
\end{example}

\begin{figure}[t]
 \centering
 \fbox{\hspace{2mm}\parbox{0.88\textwidth}{\scriptsize
  \begin{description}
  \item[Function BuildQuerySeed($\Q$, $S$)]
  \item[Input:] $\Q$: query;\hs $S$\,: {\bf set} of adorned predicates;
  \item[Output:] The query seed (a magic atom);
  \item[var:] $\alpha$: adornment string;
  \item[begin] \
  \item[]\emph{\phantom{1}1.}\hs
      Let $p(\t)$ be the atom in $\Q$.
  \item[]\emph{\phantom{1}2.}\hs
      $\alpha\ :=\ \epsilon$;
  \item[]\emph{\phantom{1}3.}\hs
      {\bf for each} argument $t$ in $\t$ {\bf do}
  \item[]\emph{\phantom{1}4.}\hs
      \hs \textbf{if} $t$ is a constant {\bf then}\hs $\alpha\ :=\ \alpha{b}$;\hs {\bf else}\hs $\alpha\ :=\ \alpha{f}$;\hs \textbf{end if}
  \item[]\emph{\phantom{1}5.}\hs
      \textbf{end for}
  \item[]\emph{\phantom{1}6.}\hs
      add $p^\alpha$ to $S$;
  \item[]\emph{\phantom{1}7.}\hs
      {\bf return} $\magic{p^\alpha(\t)}$;
  \item[end.] \
  \end{description}
 }}
 \caption{BuildQuerySeed function}\label{fig:BuildQuerySeed}
\end{figure}

The computation starts in step \emph{1} by initializing $S$, $D$, and \modifiedRules$_{\Q,\P}$ to the empty set. Then, the
function \textbf{\emph{BuildQuerySeed}}$(\Q,S)$ is used for storing in \magicRules$_{\Q,\P}$ the magic seed, and
inserting in the set $S$ the adorned predicate of $\Q$. Note that we do not generate any query rules because
standard atoms in the transformed program will not contain adornments.
Details of \textbf{\emph{BuildQuerySeed}}$(\Q,S)$ are reported in Figure~\ref{fig:BuildQuerySeed}.

\begin{example}\label{runningseed}\em
Given the query $\Q_{sc}=\tt sc(c)$ and the program $\mathcal{P}_{sc}$, function \textbf{\emph{BuildQuerySeed}}$(\Q_{sc},S)$ creates
the fact $\tt magic\_sc^{\tt b}(c)$ and inserts $\tt sc^{\tt b}$ in $S$. \hfill $\Box$
\end{example}

\begin{figure}[t]
 \centering
 \fbox{\hspace{2mm}\parbox{0.88\textwidth}{\scriptsize
  \begin{description}
  \item[Function Adorn(\R, $p^\alpha(\t)$, $S$, $D$)]
  \item[Input:] \R: rule;\hs $p^\alpha(\t)$: adorned atom;\hs $S$, $D$\,: {\bf set} of adorned predicates;
  \item[Output:] an adorned rule;
  \item[var:] $\R^a$: adorned rule;\hs $\alpha_i$: adornment string;
  \item[begin] \
  \item[]\emph{\phantom{1}1.}\hs
      Let $(\SIPSprecR{p^\alpha(\t)}, f_\R^{p^\alpha(\t)})$ be the SIPS associated with $\R$ and $p^\alpha(\t)$.
  \item[]\emph{\phantom{1}2.}\hs
      $\R^a\ :=\ \R$;
  \item[]\emph{\phantom{1}3.}\hs
      {\bf for each} IDB atom $p_i(\t_i)$ in $\HR \cup \BpR \cup \BnR$ {\bf do}
  \item[]\emph{\phantom{1}4.}\hs
      \hs $\alpha_i\ :=\ \epsilon$;
  \item[]\emph{\phantom{1}5.}\hs
      \hs {\bf for each} argument $t$ in $\t$ {\bf do}
  \item[]\emph{\phantom{1}6.}\hs
      \hs\hs {\bf if} $t$ is a constant {\bf then}
  \item[]\emph{\phantom{1}7.}\hs
      \hs\hs\hs $\alpha_i\ :=\ \alpha_i b$;
  \item[]\emph{\phantom{1}8.}\hs
      \hs\hs {\bf else}
  \item[]\emph{\phantom{1}9.}\hs
      \hs\hs\hs Argument $t$ is a variable. Let $X$ be this variable.
  \item[]\emph{\phantom{}10.}\hs
      \hs\hs\hs {\bf if} ${X} \in f_{\R}^{p^\alpha(\t)}({p(\t)})$ {\bf or}
      there is ${q(\s)}$ in $\posbody{\R}$ such that
  \item[]\emph{\phantom{}11.}\hs
      \hs\hs\hs \phantom{{\bf if} ${X} \in f_{\R}^{p^\alpha(\t)}({p(\t)})$ {\bf or}}
      ${q(\s)} \SIPSprecR{p^\alpha(\t)} {p_i(\t_i)}$ {\bf and}
      ${X}\in f^{p^\alpha(\t)}_\R({q(\s)})$ {\bf then}
  \item[]\emph{\phantom{}12.}\hs
      \hs\hs\hs\hs $\alpha_i\ :=\ \alpha_i b$;
  \item[]\emph{\phantom{}13.}\hs
      \hs\hs\hs {\bf else}
  \item[]\emph{\phantom{}14.}\hs
      \hs\hs\hs\hs $\alpha_i\ :=\ \alpha_i f$;
  \item[]\emph{\phantom{}15.}\hs
      \hs\hs\hs {\bf end if}
  \item[]\emph{\phantom{}16.}\hs
      \hs\hs {\bf end if}
  \item[]\emph{\phantom{}17.}\hs
      \hs {\bf end for}
  \item[]\emph{\phantom{}18.}\hs
      \hs substitute $p_i(\t_i)$ in $\R^a$ with $p_i^{\alpha_i}(\t_i)$;
  \item[]\emph{\phantom{}19.}\hs
      \hs {\bf if} set $D$ does not contain ${p_i^{\alpha_i}}$ {\bf then}\quad
            add $p_i^{\alpha_i}$ to $S$;\quad {\bf end if}
  \item[]\emph{\phantom{}20.}\hs
      {\bf end for}
  \item[]\emph{\phantom{}21.}\hs
      {\bf return} $\R^a$;
  \item[end.] \
  \end{description}
 }}
 \caption{Adorn function}\label{fig:adorn}
\end{figure}

The core of the algorithm (steps \emph{3--8}) is repeated until the set $S$ is empty, i.e., until there is no
further adorned predicate to be propagated. In particular, an adorned predicate $p^\alpha$ is moved from
$S$ to $D$ in step \emph{3}, and its binding is propagated in each (disjunctive) rule $\R \in \P$ of the form
\begin{dlvcode}
\R: \ p(\t) \ \Or\ p_1(\t_1) \ \Or\ \cdots \ \Or\ p_n(\t_n) \derives
    q_1(\s_1),\ \ldots,\ q_j(\s_j), \\
    \quad\quad\quad\quad\quad\quad\quad\quad\quad\quad\quad\ \
    \naf~q_{j+1}(\s_{j+1}),\ \ldots,\ \naf~q_m(\s_m).
\end{dlvcode}
(with $n\geq 0$) having an atom $p(\t)$ in the head (note that the rule
$\R$ is processed a number of times that equals the number of head atoms with predicate
$p$; steps \emph{4--8}).

\noindent \textbf{(1) Adornment.} Step \emph{5} in Figure~\ref{fig:DMS} implements the adornment of the rule.
Different from the case of non-disjunctive positive programs,
the binding of the predicate $p^\alpha$ needs to be also propagated to the atoms
$p_1(\t_1),\ldots,p_n(\t_n)$ in the head.  Therefore, binding propagation has to be extended to the head
atoms different from $p(\t)$, which are therefore adorned according to a SIPS specifically conceived for
disjunctive programs.
Notation gets slightly more involved here: Since in non-disjunctive
rules there is a single head atom, it was sufficient to specify an
order and a function for each of its adornments (omitting the head
atom in the notation). With disjunctive rules, an order and a function
need to be specified for each adorned head atom, so it is no longer
sufficient to include only the adornment in the notation, but we
rather include the full adorned atom.

\begin{definition}[SIPS for $\datdn$ rules]
\label{def:sip2} A {\em SIPS} for a $\datdn$ rule $\R$ with respect to a binding $\alpha$ for an atom ${p(\t)} \in
\HR$ is a pair $(\prec^{p^\alpha(\t)}_r,f^{p^\alpha(\t)}_r)$, where:
\begin{enumerate}
  \item $\prec^{p^\alpha(\t)}_r$ is a strict partial order over the atoms in $\HR \cup \BpR \cup \BnR$, such that:
  \begin{enumerate}
  \item ${p(\t)} \prec^{p^\alpha(\t)}_r {q(\s)}$, for all atoms ${q(\s)}\in \HR \cup \BpR \cup \BnR$
        different from ${p(\t)}$;
  \item for each pair of atoms ${q(\s)} \in (\HR \setminus \{{p(\t)}\}) \cup \BnR$
        and ${b(\z)} \in \HR \cup \BpR \cup \BnR$,
        ${q(\s)} \prec^{p^\alpha(\t)}_r {b(\z)}$ does not hold; and,
  \end{enumerate}

  \item  $f^{p^\alpha(\t)}_r$ is a function assigning to each atom ${q(\s)}\in \HR \cup \BpR \cup \BnR$ a subset of the variables in $\s$---intuitively,
  those made bound when processing ${q(\s)}$.
\end{enumerate}
\end{definition}

As for $\dat$ rules, for each atom $q(\s)$ occurring in $\R$, the strict partial
order $\prec^{p^\alpha(\t)}_r$ specifies those atoms that have to be processed
before processing atom $q(\s)$, and an argument $X$ of $q(\s)$
is bound to a constant if there exists an atom $q'(\s')$ occurring in $\R$ such that
${q'(\s')} \prec^{p^\alpha(\t)}_r {q(\s)}$ and
${X} \in f^{p^\alpha(\t)}_r({q'(\s')})$.
The difference with respect to SIPSes for $\dat$ rules is precisely in the dependency from $p(\t)$
in addition to $\alpha$, and in condition \emph{(1.b)} stating that head atoms different from $p(\t)$
and negative body literals
cannot provide bindings to variables of other atoms.

The underlying idea is that a rule which is used to ``prove'' the
truth of an atom in a top-down method will be a rule which supports
that atom. This implies that all other head atoms in that rule must be
false and that the body must be true. Head atoms and atoms occurring
in the negative body cannot ``create'' bindings (that is, restrict
the values of variables), but these atoms are still relevant to the
query, which leads to the restrictions in Definition~\ref{def:sip2}.

Note that this definition considers each rule in isolation and is
therefore independent of the inter-rule structure of a program. In
particular, it is not important for the SIPS definition whether a
program is cyclic or contains head cycles.

In the following, we shall assume that each $\datdn$ program is
provided in input together with some arbitrarily defined SIPS for $\datdn$ rules \mbox{$(\prec^{p^\alpha(\t)}_r,f^{p^\alpha(\t)}_r)$}. In fact, armed with $(\prec^{p^\alpha(\t)}_r,f^{p^\alpha(\t)}_r)$, the adornment can be carried out precisely as we discussed for $\dat$ programs; in particular,
we recall here that a variable $X$ of an atom $q(\s)$ in $\R$ is bound if and only if either:

\begin{enumerate}
\item ${X}\in f^{p^\alpha(\t)}_r({q(\s)})$ with ${q(\s)} = {p(\t)}$; or,

\item ${X}\in f^{p^\alpha(\t)}_r({b(\bar z)})$ for an atom
${b(\bar z)}\in \posbody{\R}$ such that
${b(\bar z)} \prec^{p^\alpha(\t)}_r {q(\s)}$ holds.
\end{enumerate}

The function \textbf{\emph{Adorn}}$(\R,{p^\alpha(\t)},S,D)$ produces an adorned disjunctive rule $\R^a$
from an adorned atom $p^\alpha(\t)$ and a suitable
unadorned rule $\R$ (according to the bindings defined in the points (1) and (2) above),
by inserting all newly adorned
predicates in $S$. Hence, in step \emph{5} the rule $\R^a$ is of the form

\vspace{-1mm}
\begin{dlvcode}
\R^a: \ p^\alpha(\t)\,\Or\,p_1^{\alpha_1}(\t_1)\,\Or\,\cdots\,\Or\,p_n^{\alpha_n}(\t_n) \derives
    q_1^{\beta_1}(\s_1),\ \ldots,\ q_j^{\beta_j}(\s_j), \\
    \quad\quad\quad\quad\quad\quad\quad\quad\quad\quad\quad\quad\quad\ \
    \naf~q_{j+1}^{\beta_{j+1}}(\s_{j+1}),\ \ldots,\ \naf~q_m^{\beta_m}(\s_m).
\end{dlvcode}%
Details of \textbf{\emph{Adorn}}$(\R,{p^\alpha(\t)},S,D)$ are reported in Figure~\ref{fig:adorn}.

\begin{example}\label{running2}\em
Let us resume from Example~\ref{runningseed}.
We are supposing that the adopted SIPS is passing the bindings
via $\tt produced\_by$ and $\tt controlled\_by$ to the variables
of $\tt sc$ atoms, in particular
\begin{eqnarray*}
{\tt sc(C_1)} & \prec^{\tt sc^{\tt b}(C_1)}_{\R_3} & {\tt produced\_by(P,C_1,C_2)}\\
{\tt sc(C_1)} & \prec^{\tt sc^{\tt b}(C_1)}_{\R_3} & {\tt sc(C_2)}\\
{\tt produced\_by(P,C_1,C_2)} &\prec^{\tt sc^{\tt b}(C_1)}_{\R_3}& {\tt sc(C_2)}\\\\
{\tt sc(C_2)} & \prec^{\tt sc^{\tt b}(C_2)}_{\R_3} & {\tt produced\_by(P,C_1,C_2)}\\
{\tt sc(C_2)} & \prec^{\tt sc^{\tt b}(C_2)}_{\R_3} & {\tt sc(C_1)}\\
{\tt produced\_by(P,C_1,C_2)} &\prec^{\tt sc^{\tt b}(C_2)}_{\R_3}& {\tt sc(C_1)}\\\\
{\tt sc(C)} & \prec^{\tt sc^{\tt b}(C)}_{\R_4} & {\tt controlled\_by(C,C_1,C_2,C_3)}\\
{\tt sc(C)} & \prec^{\tt sc^{\tt b}(C)}_{\R_4} & {\tt sc(C_1)}\\
{\tt sc(C)} & \prec^{\tt sc^{\tt b}(C)}_{\R_4} & {\tt sc(C_2)}\\
{\tt sc(C)} & \prec^{\tt sc^{\tt b}(C)}_{\R_4} & {\tt sc(C_3)}\\
{\tt controlled\_by(C,C_1,C_2,C_3)} & \prec^{\tt sc^{\tt b}(C)}_{\R_4} & {\tt sc(C_1)}\\
{\tt controlled\_by(C,C_1,C_2,C_3)} & \prec^{\tt sc^{\tt b}(C)}_{\R_4} & {\tt sc(C_2)}\\
{\tt controlled\_by(C,C_1,C_2,C_3)} & \prec^{\tt sc^{\tt b}(C)}_{\R_4} & {\tt sc(C_3)}
\end{eqnarray*}
\[
\begin{array}{l}
f^{\tt sc^{\tt b}(C_1)}_{\R_3}({\tt sc(C_1)}) =  \{{\tt C_1}\}\\
f^{\tt sc^{\tt b}(C_1)}_{\R_3}({\tt produced\_by(P,C_1,C_2)}) = \{{\tt P,C_2}\}\\
f^{\tt sc^{\tt b}(C_1)}_{\R_3}({\tt sc(C_2)}) =  \emptyset \vspace{1em}\\
f^{\tt sc^{\tt b}(C_2)}_{\R_3}({\tt sc(C_2)}) = \{{\tt C_2}\}\\
f^{\tt sc^{\tt b}(C_2)}_{\R_3}({\tt produced\_by(P,C_1,C_2)}) = \{{\tt P,C_1}\}\\
f^{\tt sc^{\tt b}(C_2)}_{\R_3}({\tt sc(C_1)}) = \emptyset \vspace{1em}\\
f^{\tt sc^{\tt b}(C)}_{\R_4}({\tt sc(C)}) = \{{\tt C}\}\\
f^{\tt sc^{\tt b}(C)}_{\R_4}({\tt controlled\_by(C,C_1,C_2,C_3)}) =\{{\tt C_1,C_2,C_3}\}\\
f^{\tt sc^{\tt b}(C)}_{\R_4}({\tt sc(C_1)}) =
f^{\tt sc^{\tt b}(C)}_{\R_4}({\tt sc(C_2)}) =
f^{\tt sc^{\tt b}(C)}_{\R_4}({\tt sc(C_3)}) = \emptyset
\end{array}
\]
When $\tt sc^{\tt b}$ is removed from the set $S$, we first select rule $\R_3$
and the head predicate $\tt sc(C_1)$. Then
the adorned version is
\vspace{-2mm}\begin{dlvcode}
\R_{3,1}^a:\ \tt sc^{\tt b}(C_1)\ \Or\ sc^{\tt b}(C_2) \derives produced\_by(P,C_1,C_2). \\
\end{dlvcode}
\vspace{-2mm}Next, $r_3$ is processed again, this time with head predicate $\tt sc(C_2)$, producing
\vspace{-2mm}\begin{dlvcode}
\R_{3,2}^a:\ \tt sc^{\tt b}(C_2)\ \Or\ sc^{\tt b}(C_1) \derives produced\_by(P,C_1,C_2). \\
\end{dlvcode}
\vspace{-2mm}Finally, processing $r_4$ we obtain
\vspace{-2mm}\begin{dlvcode}
\R_4^a:\ \tt sc^{\tt b}(C) \derives controlled\_by(C,C_1,C_2,C_3),\,sc^{\tt b}(C_1),\,sc^{\tt b}(C_2),\,sc^{\tt b}(C_3).\\
\end{dlvcode}
\vspace{-13mm}
\hfill $\Box$
\end{example}

\begin{figure}[t]
 \centering
 \fbox{\hspace{2mm}\parbox{0.88\textwidth}{\scriptsize
  \begin{description}
  \item[Function Generate(\R, $p^\alpha(\t)$, $\R^a$)]
  \item[Input:] $\R$: rule;\hs $p^\alpha(\t)$: adorned atom;\hs $\R^a$: adorned rule;
  \item[Output:] a {\bf set} of magic rules;
  \item[var:] $R$: {\bf set} of rules;\hs $\R^\magica$: rule;
  \item[begin] \
  \item[]\emph{\phantom{1}1.}\hs
      Let $(\SIPSprecR{p^\alpha(\t)}, f_\R^{p^\alpha(\t)})$ be the SIPS associated with $\R$ and $p^\alpha(\t)$.
  \item[]\emph{\phantom{1}2.}\hs
      $R\ :=\ \emptyset$;
  \item[]\emph{\phantom{1}3.}\hs
      {\bf for each} atom $p_i^{\alpha_i}(\t_i)$ in $\head{\R^a} \cup \posbody{\R^a} \cup \negbody{\R^a}$ different from $p^\alpha(\t)$ {\bf do}
  \item[]\emph{\phantom{1}4.}\hs
      \hs {\bf if} $\alpha_i \neq \epsilon$ {\bf then}
  \item[]\emph{\phantom{1}5.}\hs
      \hs\hs $\R^\magica\ :=\ \magic{p_i^{\alpha_i}(\t_i)} \derives \magic{p^\alpha(\t)}$;
  \item[]\emph{\phantom{1}6.}\hs
      \hs\hs {\bf for each} atom $p_j(\t_j)$ in \posbody{\R} such that ${p_j(\t_j)} \SIPSprecR{p^\alpha(\t)} {p_i(\t_i)}$ {\bf do}
  \item[]\emph{\phantom{1}7.}\hs
      \hs\hs\hs add atom $p_j(\t_j)$ to \posbody{\R^\magica};
  \item[]\emph{\phantom{1}8.}\hs
      \hs\hs {\bf end for}
  \item[]\emph{\phantom{1}9.}\hs
      \hs\hs $R\ :=\ R \cup \{\R^\magica\}$;
  \item[]\emph{\phantom{}10.}\hs
      \hs {\bf end if}
  \item[]\emph{\phantom{}11.}\hs
      {\bf end for}
  \item[]\emph{\phantom{}12.}\hs
      {\bf return} $R$;
  \item[end.] \
  \end{description}
 }}
 \caption{Generate function}\label{fig:generate}
\end{figure}

\noindent \textbf{(2) Generation.} The algorithm uses the adorned rule $\R^a$ for generating and collecting the
magic rules in step \emph{6} (Figure~\ref{fig:DMS}).
More specifically, \textbf{\emph{Generate}}$(\R,{p^\alpha(\t)},\R^a)$
produces magic rules according to the following schema:
if $p_i^{\alpha_i}(\t_i)$ is an adorned atom
(i.e., $\alpha_i$ is not the empty string) occurring in $\R^a$ and
different from $p^\alpha(\t)$,
a magic rule $\R^\magica$ is generated such that
(i) $\head{\R^\magica} = \{\magic{p_i^{\alpha_i}(\t_i)}\}$ and
(ii) $\body{\R^\magica}$ is the union of $\{\magic{p^\alpha(\t)}\}$ and the
set of all the atoms ${q_j^{\beta_j}(\s_j)} \in \posbody{\R}$ such that
${q_j(\s_j)} \prec^{\alpha}_r {p_i(\t_i)}$.
Details of \textbf{\emph{Generate}}$(\R,{p^\alpha(\t)},\R^a)$ are reported in Figure~\ref{fig:generate}.

\vspace{-2mm}\begin{example}\label{running3}\em Continuing with our running example,
by invoking \textbf{\emph{Generate}}$(\R_3,{\tt sc^{\tt b}(C_1)},\R_{3,1}^a)$, the
following magic rule is produced:
\vspace{-2mm}\begin{dlvcode}
\R_{3,1}^\magica:\ \tt magic\_sc^{\tt b}(C_2) \derives magic\_sc^{\tt b}(C_1),\ produced\_by(P,C_1,C_2). \\
\end{dlvcode}
\vspace{-2mm}Similarly, by invoking \textbf{\emph{Generate}}$(\R_3,{\tt sc^{\tt b}(C_2)},\R_{3,2}^a)$,
the following magic rule is produced:
\vspace{-2mm}\begin{dlvcode}
\R_{3,2}^\magica:\ \tt magic\_sc^{\tt b}(C_1) \derives magic\_sc^{\tt b}(C_2),\ produced\_by(P,C_1,C_2). \\
\end{dlvcode}
\vspace{-2mm} Finally, the following magic rules are produced by
\textbf{\emph{Generate}}$(\R_4,{\tt sc^{\tt b}(C)},\R_4^a)$:
\vspace{-2mm}\begin{dlvcode}
\R_{4,1}^\magica:\ \tt magic\_sc^{\tt b}(C_1) \derives magic\_sc^{\tt b}(C),\ controlled\_by(C,C_1,C_2,C_3).\\
\R_{4,2}^\magica:\ \tt magic\_sc^{\tt b}(C_2) \derives magic\_sc^{\tt b}(C),\ controlled\_by(C,C_1,C_2,C_3).\\
\R_{4,3}^\magica:\ \tt magic\_sc^{\tt b}(C_3) \derives magic\_sc^{\tt b}(C),\ controlled\_by(C,C_1,C_2,C_3).\\
\end{dlvcode}
\vspace{-13mm}\hfill $\Box$
\end{example}

\noindent \textbf{(3) Modification.} In step \emph{7} the modified rules are generated and collected. The only difference
with respect to the $\dat$ case is that the adornments are stripped off the original atoms.
Specifically, given an adorned rule $\R^a$ associated with a rule $\R$,
a modified rule $\R'$ is obtained from $\R$ by adding to its body an
atom $\magic{p^\alpha(\t)}$ for each atom $p^\alpha(\t)$ occurring in
$\head{\R^a}$.
Hence, the function
\textbf{\emph{Modify}}$(\R,\R^a)$, reported in Figure~\ref{fig:modify}, constructs a rule $\R'$ of the form
\vspace{-2mm}\begin{dlvcode}
\R': \ p(\t)\,\Or\,p_1(\t_1)\,\Or\,\cdots\,\Or\,p_n(\t_n) \derives \magic{p^\alpha(\t)},
     \magic{p_1^{\alpha_1}(\t_1)}, \ldots, \\
     \quad\quad\ \
     \magic{p_n^{\alpha_n}(\t_n)},
     q_1(\s_1),\ldots,q_j(\s_j),
     \naf~q_{j+1}(\s_{j+1}),\ \ldots,\ \naf~q_m(\s_m).
\end{dlvcode}

\begin{figure}[t]
 \centering
 \fbox{\hspace{2mm}\parbox{0.88\textwidth}{\scriptsize
  \begin{description}
  \item[Function Modify($\R$, $\R^a$)]
  \item[Input:] \R: rule;\hs $\R^a$: adorned rule;
  \item[Output:] a modified rule;
  \item[var:] $\R'$: rule;
  \item[begin] \
  \item[]\emph{\phantom{1}1.}\hs
      $\R'\ :=\ \R$;
  \item[]\emph{\phantom{1}2.}\hs
      {\bf for each} atom $p^\alpha(\t)$ in $\head{\R^a}$ {\bf do}
  \item[]\emph{\phantom{1}3.}\hs
      \hs add \magic{p^\alpha(\t)} to $\posbody{\R'}$;
  \item[]\emph{\phantom{1}4.}\hs
      {\bf end for}
  \item[]\emph{\phantom{1}5.}\hs
      {\bf return} $\R'$;
  \item[end.] \
  \end{description}
 }}
 \caption{Modify function}\label{fig:modify}
\end{figure}

\noindent Finally, after all the adorned predicates have been processed, the algorithm outputs the program
$\DMS(\Q,\P)$.

\begin{example}\label{running4}\em
In our running example, we derive the
following set of modified rules:
\vspace{-2mm}{\small
\begin{dlvcode}
\begin{array}{ll}
\R_{3,1}':\ & \tt sc(C_1)\ \Or\ sc(C_2) \derives
          magic\_sc^{\tt b}(C_1),\ magic\_sc^{\tt b}(C_2),\\
          & \tt \phantom{sc(C_1)\ \Or\ sc(C_2) \derives} produced\_by(P,C_1,C_2). \\
\R_{3,2}':\ & \tt sc(C_2)\ \Or\ sc(C_1) \derives
          magic\_sc^{\tt b}(C_2),\ magic\_sc^{\tt b}(C_1),\\
          & \tt \phantom{sc(C_1)\ \Or\ sc(C_2) \derives} produced\_by(P,C_1,C_2). \\
\R_4':\ & \tt sc(C) \derives magic\_sc^{\tt b}(C),\
          controlled\_by(C,C_1,C_2,C_3),\\
          & \tt \phantom{sc(C) \derives} sc(C_1),\ sc(C_2),\ sc(C_3).
\end{array}
\end{dlvcode}}
\vspace{-2mm}\noindent Here, $\R_{3,1}'$ (resp. $\R_{3,2}'$, $\R_4'$)
is derived by adding magic predicates and
stripping off adornments for the rule $\R_{3,1}^a$ (resp. $\R_{3,2}^a$, $\R_4^a$).
Thus, the optimized program
$\DMS(\Q_{sc},\P_{sc})$ comprises the above modified rules as well as  the magic rules in
Example~\ref{running3}, and the magic seed $\tt magic\_sc^{\tt b}(c)$
(together with the original EDB). \hfill $\Box$
\end{example}

Before establishing the correctness of the technique, we briefly present an
example of the application of $\DMS$ on a program containing disjunction and
stratified negation.

\begin{example}\label{running5}\em
Let us consider a slight variant of the Strategic Companies problem described
in Example~\ref{running} in which we have to determine whether a given company
$\tt c$ does not belong to any strategic set. We can thus consider the query
$\tt nsc(c)$ for the program $\p_{nsc}$ obtained by adding to $\p_{sc}$
the following rule:
\vspace{-2mm}{\small
\begin{dlvcode}
\begin{array}{ll}
\R_{nsc}:\ & \tt nsc(C) \derives company(C),\ \naf~sc(C).
\end{array}
\end{dlvcode}}
where $\tt company$ is an EDB predicate. Company $\tt c$ does not belong to any
strategic set if the query is cautiously false.

In this case, processing the query produces the query seed $\tt magic\_nsc^{\tt b}(c)$
(a fact) and the adorned predicate $\tt nsc^{\tt b}$ (which is added to set $S$).
After that, $\tt nsc^{\tt b}$ is moved from $S$ to $D$ and rule $\R_{nsc}$ is considered.
Assuming the following SIP:
\[\begin{array}{c}
{\tt nsc(C)} \prec^{\tt nsc^{\tt b}(C)}_{\R_{nsc}} {\tt company(C)}
\qquad\qquad
{\tt nsc(C)} \prec^{\tt nsc^{\tt b}(C)}_{\R_{nsc}} {\tt sc(C)}\\
f^{\tt nsc^{\tt b}(C)}_{\R_{nsc}}({\tt nsc(C)}) = \{{\tt C}\}
\quad\quad
f^{\tt nsc^{\tt b}(C)}_{\R_{nsc}}({\tt company(P)}) =
f^{\tt nsc^{\tt b}(C)}_{\R_{nsc}}({\tt sc(C)}) = \emptyset
\end{array}
\]
by invoking \textbf{\emph{Adorn}}$(\R_{nsc},{\tt nsc^{\tt b}(C)},S,D)$ we obtain the
following adorned rule:
\vspace{-2mm}{\small
\begin{dlvcode}
\begin{array}{ll}
\R_{nsc}^a:\ & \tt nsc^{\tt b}(C) \derives company(C),\ \naf~sc^{\tt b}(C).
\end{array}
\end{dlvcode}}
The new adorned predicate $\tt sc^{\tt b}$ is added to $S$. Then,
\textbf{\emph{Generate}}$(\R_{nsc},{\tt nsc^{\tt b}(C)},\R_{nsc}^a)$ and
\textbf{\emph{Modify}}$(\R_{nsc},\R_{nsc}^a)$ produce the following magic
and modified rules:
\vspace{-2mm}{\small
\begin{dlvcode}
\begin{array}{ll}
\R_{nsc}^\magica:\ & \tt magic\_sc^{\tt b}(C) \derives magic\_nsc^{\tt b}(C).\\
\R_{nsc}':\ & \tt nsc(C) \derives magic\_nsc^{\tt b}(C),\ company(C),\ \naf~sc(C).
\end{array}
\end{dlvcode}}
The algorithm then processes the adorned atom $\tt sc^{\tt b}$.
Hence, if the SIPS presented in Example~\ref{running2} is assumed, the rewritten
program comprises the following rules:
$\R_{nsc}'$,
$\R_{3,1}'$,
$\R_{3,2}'$,
$\R_4'$,
$\R_{nsc}^\magica$,
$\R_{3,1}^\magica$,
$\R_{3,2}^\magica$,
$\R_{4,1}^\magica$,
$\R_{4,2}^\magica$ and
$\R_{4,3}^\magica$.  \hfill $\Box$
\end{example}

\subsection{Query Equivalence Result}\label{sec:teoria}

We conclude the presentation of the $\DMS$ algorithm by formally proving its correctness. 
We would like to point out that all of these results hold for any kind of SIPS, as long as it conforms to Definition~\ref{def:sip2}. Therefore, in the remainder of this section, we assume that any program comes with some associated SIPS.
In the proofs, we
use the well established notion of unfounded set for disjunctive $\dat$ programs (possibly with negation)
defined in \cite{leon-etal-97b}.
Before introducing unfounded sets, however, we have to define partial interpretations,
that is, interpretations for which some atoms may be undefined.

\begin{definition}[Partial Interpretation]\label{def:partial_interpretation}
Let $\p$ be a $\datdnr$ program. A {\em partial interpretation} for $\p$ is a pair
$\tuple{T,N}$ such that $T \subseteq N \subseteq \BP$. The atoms in $T$
are interpreted as true, while the atoms in $N$ are not false and those in $N \setminus T$ are undefined.
All other atoms are false.
\end{definition}

Note that total interpretations are a special case in which $T = N$.
We can then formalize the notion of unfounded set.

\begin{definition}[Unfounded Sets]
\label{def:unfoundedset} Let $\tuple{T,N}$ be a partial interpretation for a $\datdnr$ program $\p$, and $X \subseteq \BP$
be a set of  atoms. Then, $X$ is an \emph{unfounded set} for $\p$ with respect to $\tuple{T,N}$  if and only if, for each ground rule
$\R_g \in \ground{\p}$ with $X \cap \head{\R_g} \neq \emptyset$,
at least one of the following conditions holds:
$(1.a)$ $\posbody{\R_g} \not\subseteq N$;
$(1.b)$ $\negbody{\R_g} \cap T \neq \emptyset$;
$(2)$ $\posbody{\R_g} \cap X \neq \emptyset$;
$(3)$ $\head{\R_g} \cap (T \setminus X) \neq \emptyset$.
\end{definition}

Intuitively, conditions $(1.a)$, $(1.b)$ and $(3)$ check if the rule is satisfied by $\tuple{T,N}$ regardless of the atoms in $X$,
while condition $(2)$ checks whether the rule can be satisfied by taking the atoms in $X$ as false.

\begin{example}\label{ex:unfounded}\em
Consider again the program $\p_{sc}$ of Example~\ref{running} and assume
$\EDB{\p_{sc}} = \{{\tt produced\_by(p,c,c_1)}\}$.
Then $\ground{\p_{sc}}$ consists of the rule

\vspace{-3mm}\begin{dlvcode}
\R_{sc}: \ \tt sc(c) \Or sc(c_1) \derives produced\_by(p,c,c_1).
\end{dlvcode}
\vspace{-7mm}

\noindent
(together with facts, and rules having some ground instance of EDB predicate not occurring
in $\EDB{\p_{sc}}$, omitted for simplicity).
Consider now a partial interpretation $\tuple{M_{sc},B_{\p_{sc}}}$ such that $M_{sc} = \{{\tt produced\_by(p,c,c_1), sc(c)}\}$.
Thus, $\{{\tt sc(c_1)}\}$ is an unfounded set for $\p_{\sc}$ with respect to $\tuple{M_{sc},B_{\p_{sc}}}$
($\R_{sc}$ satisfies condition $(3)$ of Definition~\ref{def:unfoundedset}),
while $\{{\tt sc(c), sc(c_1)}\}$ is not ($\R_{sc}$ violates all conditions).
\hfill $\Box$
\end{example}

The following is an adaptation of Theorem~4.6 in \cite{leon-etal-97b} to our notation.

\begin{theorem}[\cite{leon-etal-97b}]\label{theo:unfounded}
Let $\tuple{T,N}$ be a partial interpretation for a $\datdnr$ program $\p$. Then, for any stable
model $M$ of $\p$ such that $T \subseteq M \subseteq N$, and for each unfounded set $X$ of $\p$ with respect to $\tuple{T,N}$,
$M \cap X = \emptyset$ holds.
\end{theorem}

\begin{example}\label{ex:unfoundAndAnswerSet}\em
In Example~\ref{ex:unfounded}, we have shown that $\{{\tt sc(c_1)}\}$ is an
unfounded set for $\p_{\sc}$ with respect to $\tuple{M_{sc},B_{\p_{sc}}}$.
Note that the total interpretation $M_{sc}$ is a stable model of $\p_{sc}$,
and that the unfounded set $\{{\tt sc(c_1)}\}$ is disjoint from
$M_{sc}$. \hfill $\Box$
\end{example}

Equipped with these notions and Theorem~\ref{theo:unfounded}, we now proceed to prove the correctness of the $\DMS$ strategy. In
particular, we shall first show that the method is \emph{sound} in that, for each stable model $M$ of $\dmsqp$,
there is a stable model $M'$ of $\p$ such that $M'|_\Q = M|_\Q$ (i.e., the two models coincide when restricted
to the query). Then, we prove that the method is also \emph{complete}, i.e., for each stable model $M'$ of
$\p$, there is a stable model $M$ of $\dmsqp$ such that $M'|_\Q = M|_\Q$.

In both parts of the proof, we shall exploit the following (syntactic) relationship between the original program
and the transformed one.

\begin{lemma}\label{lem:mappingGroundNonground}
Let $\p$ be a $\datdn$ program, $\Q$ a query, and let $\magic{{p^\alpha(\t)}}$ be a ground atom%
\footnote{Note that in this way the lemma refers only to rules that contain a head atom for which
a magic predicate has been generated during the transformation.}
in $\ensuremath{B_{\dmsqp}}$ (the base of the transformed program). Then the ground rule

\begin{dlvcode}
\R_g: \ p(\t)\,\Or\,p_1(\t_1)\,\Or\,\cdots\,\Or\,p_n(\t_n) \derives
    q_1(\s_1),\ \ldots,\ q_j(\s_j), \\
    \quad\quad\quad\quad\quad\quad\quad\quad\quad\quad\quad\ \
    \naf~q_{j+1}(\s_{j+1}),\ \ldots,\ \naf~q_m(\s_m).
\end{dlvcode}

\vspace{-6mm}belongs to $\ground{\p}$ if and only if the ground rule

\vspace{-3mm}\begin{dlvcode}
\R_g': \ p(\t)\,\Or\,p_1(\t_1)\,\Or\,\cdots\,\Or\,p_n(\t_n) \derives \magic{p^\alpha(\t)},
     \magic{p_1^{\alpha_1}(\t_1)}, \ldots, \\
     \quad\quad\ \
     \magic{p_n^{\alpha_n}(\t_n)},
     q_1(\s_1),\ldots,q_j(\s_j),
     \naf~q_{j+1}(\s_{j+1}),\ \ldots,\ \naf~q_m(\s_m).
\end{dlvcode}

\vspace{-5mm}belongs to $\ground{\dmsqp}$.
\end{lemma}
\begin{proof}
$(\Rightarrow)$ Consider the following rule $\R\in \p$ such that $\R_g=\R\vartheta$ for some substitution $\vartheta$:

\begin{dlvcode}
\R: \ p(\t')\,\Or\,p_1(\t_1')\,\Or\,\cdots\,\Or\,p_n(\t_n') \derives
    q_1(\s_1'),\ \ldots,\ q_j(\s_j'), \\
    \quad\quad\quad\quad\quad\quad\quad\quad\quad\quad\quad\ \
    \naf~q_{j+1}(\s_{j+1}'),\ \ldots,\ \naf~q_m(\s_m').
\end{dlvcode}

Since $\magic{{p^\alpha(\t)}}$ is a ground atom in $\ensuremath{B_{\dmsqp}}$, $p^{\alpha}$ has been
inserted in the set $S$ at some point of the Magic Set transformation, and it has eventually been used to adorn and
modify $\R$, thereby producing the following rule $\R'\in\dmsqp$:

\begin{dlvcode}
\R': \ p(\t')\,\Or\,p_1(\t_1')\,\Or\,\cdots\,\Or\,p_n(\t_n') \derives \magic{p^\alpha(\t')},
     \magic{p_1^{\alpha_1}(\t_1')}, \ldots, \\
     \quad\quad\ \
     \magic{p_n^{\alpha_n}(\t_n')},
     q_1(\s_1'),\ldots,q_j(\s_j'),
     \naf~q_{j+1}(\s_{j+1}'),\ \ldots,\ \naf~q_m(\s_m').
\end{dlvcode}

Clearly enough, the substitution $\vartheta$ mapping $\R$ into $\R_g$ can also be used to map $\R'$ into $\R_g'$, since the
magic atoms added into the positive body of $\R'$ are defined over a subset of the variables occurring in head atoms.

$(\Leftarrow)$ Let $\R' \in \dmsqp$ be a rule such that $\R_g'=\R'\vartheta$ for some substitution $\vartheta$:

\begin{dlvcode}
\R': \ p(\t')\,\Or\,p_1(\t_1')\,\Or\,\cdots\,\Or\,p_n(\t_n') \derives \magic{p^\alpha(\t')},
     \magic{p_1^{\alpha_1}(\t_1')}, \ldots, \\
     \quad\quad\ \
     \magic{p_n^{\alpha_n}(\t_n')},
     q_1(\s_1'),\ldots,q_j(\s_j'),
     \naf~q_{j+1}(\s_{j+1}'),\ \ldots,\ \naf~q_m(\s_m').
\end{dlvcode}

By the construction of $\dmsqp$, $\R'$ is a modified rule produced by adding some magic atom
to the positive body of a rule $\R \in \p$ of the form:

\begin{dlvcode}
\R: \ p(\t')\,\Or\,p_1(\t_1')\,\Or\,\cdots\,\Or\,p_n(\t_n') \derives
    q_1(\s_1'),\ \ldots,\ q_j(\s_j'), \\
    \quad\quad\quad\quad\quad\quad\quad\quad\quad\quad\quad\ \
    \naf~q_{j+1}(\s_{j+1}'),\ \ldots,\ \naf~q_m(\s_m').
\end{dlvcode}

Thus, the substitution $\vartheta$ mapping $\R'$ to $\R_g'$ can also be used to map $\R$ to $\R_g$, since $\R$ and $\R'$
have the same variables.
\end{proof}

\subsubsection{Soundness of the Magic Set Method}
\label{sec:soundness}

Let us now start with the first part of the proof, in particular, by stating some further definitions and
notations. Given a model $M'$ of $\dmsqp$, and a model $N' \subseteq M'$ of $\ground{\dmsqp}^{M'}$,
we next define the set of atoms which are relevant for $\Q$ but are false with respect to $N'$.

\begin{definition}[Killed Atoms]
\label{def:killed} Given a model $M'$ for $\dmsqp$, and a model $N' \subseteq M'$ of $\ground{\dmsqp}^{M'}$,
the set $\killedmpnp$ of the \emph{killed atoms}
with respect to $M'$ and $N'$ is defined as:
$$
\begin{array}{ll}
\{{k(\t)} \in \BP \setminus N' \ | & \mbox{ either }\, {k}\, \mbox{ is an EDB predicate, or }\\
& \mbox{ there is a binding } \alpha \mbox{ such that } \magic{{k^\alpha(\t)}} \in N' \}.
\end{array}
$$
\end{definition}
\vspace{-1em}

\begin{example}\label{ex:killed}\em
We consider the program $\DMS(\Q_{sc},\p_{sc})$ presented in Section~\ref{sec:dms}
(we recall that $\Q_{sc} = \tt sc(c)$), the EDB $\{{\tt produced\_by(p,c,c_1)}\}$
introduced in Example~\ref{ex:unfounded},
and a stable model $M_{sc}' = \{{\tt produced\_by(p,c,c_1), sc(c)}, {\tt magic\_sc^{\tt b}(c), magic\_sc^{\tt b}(c_1)}\}$ for $\DMS(\Q_{sc},\p_{sc})$.
Thus, $\ground{\DMS(\Q_{sc},\p_{sc})}^{M_{sc}}$ consists of the following rules:

\vspace{-3mm}\begin{dlvcode}
\tt magic\_sc^{\tt b}(c). \quad\quad magic\_sc^{\tt b}(c_1) \derives magic\_sc^{\tt b}(c).\\
\tt sc(c) \Or sc(c_1) \derives magic\_sc^{\tt b}(c),\ magic\_sc^{\tt b}(c_1),\ produced\_by(p,c,c_1).
\end{dlvcode}
\vspace{-7mm}

\noindent
Since $M_{sc}'$ is also a model of the program above, we can compute
$\killed{M_{sc}'}{M_{sc}'}{\Q_{sc}}{\p_{sc}}$ and check that $\tt sc(c_1)$
belongs to it because of $\tt magic\_sc^{\tt b}(c_1)$ in $M_{sc}'$.
Note that, by definition, also false ground instances of EDB predicates like
$\tt produced\_by(p,c_1,c)$ or $\tt controlled\_by(c,c_1,c_1,c_1)$ belong
to $\killed{M_{sc}'}{M_{sc}'}{\Q_{sc}}{\p_{sc}}$.
Moreover, note that no other atom belongs to this set.
\hfill $\Box$
\end{example}

The intuition underlying the definition above is that killed atoms are either
false ground instances of some EDB predicate, or false atoms which are relevant
with respect to $\Q$ (for there exists an associated magic atom in the model $N'$);
since $N'$ is a model of $\ground{\dmsqp}^{M'}$ contained in $M'$,
we expect that these atoms are also false in any stable model for $\p$ containing $\Mpp$
(which, we recall here, is the model $M'$ restricted on the atoms originally occurring in $\P$).

\begin{example}\em
Let us resume from Example~\ref{ex:killed}.
We have that $M_{sc}'|_{\p_{sc}} = \{{\tt produced\_by(p,c,c_1), sc(c)}\}$, which
coincides with model $M_{sc}$ of Example~\ref{ex:unfounded}. Hence,
we already know that $\{{\tt sc(c_1)}\}$ is an unfounded set for $\p_{sc}$
with respect to $\tuple{M_{sc},B_{\p_{sc}}}$.
Since each other atom $k(\t)$ in $\killed{M_{sc}'}{M_{sc}'}{\Q_{sc}}{\p_{sc}}$ is
such that $k$ is an EDB predicate,
we also have that $\killed{M_{sc}'}{M_{sc}'}{\Q_{sc}}{\p_{sc}}$
is an unfounded set for $\p_{sc}$ with respect to $\tuple{M_{sc},B_{\p_{sc}}}$.
Therefore, as a consequence of Theorem~\ref{theo:unfounded},
each stable model $M$ of $\p_{sc}$ such that $M_{sc} \subseteq M \subseteq B_{\p_{sc}}$
(in this case only $M_{sc}$ itself) is disjoint from
$\killed{M_{sc}'}{M_{sc}'}{\Q_{sc}}{\p_{sc}}$.
\hfill $\Box$
\end{example}

This intuition is formalized below.


\begin{proposition}
\label{prop:killed_unfounded} Let $M'$ be a model for $\dmsqp$,
and $N' \subseteq M'$ be a model of $\ground{\dmsqp}^{M'}$. Then, $\killedmpnp$ is an unfounded set
for $\p$ with respect to $\tuple{\Mpp,\BP}$.
\end{proposition}
\begin{proof}
According to Definition~\ref{def:unfoundedset} of unfounded sets (for $\p$ with respect to $\tuple{\Mpp,\BP}$),
given any rule $\R_g$ in $\ground{\p}$ of the form

\begin{dlvcode}
\R_g: \ k(\t)\,\Or\,p_1(\t_1)\,\Or\,\cdots\,\Or\,p_n(\t_n) \derives
    q_1(\s_1),\ \ldots,\ q_j(\s_j), \\
    \quad\quad\quad\quad\quad\quad\quad\quad\quad\quad\quad\ \
    \naf~q_{j+1}(\s_{j+1}),\ \ldots,\ \naf~q_m(\s_m).
\end{dlvcode}

\vspace{-3mm}we have to show that if ${k(\t)} \in \killedmpnp\cap \head{\R_g}$, then at least one of the following conditions holds: $(1.a)$ $\posbody{\R_g}
\not\subseteq \BP$; 
$(1.b)$ $\negbody{\R_g} \cap \Mpp \neq \emptyset$;
$(2)$ $\posbody{\R_g} \cap \killedmpnp \neq \emptyset$; $(3)$ $\head{\R_g} \cap (\Mpp
\setminus \killedmpnp) \neq \emptyset$.

Note that the properties above refer to the original program $\p$. However, our hypothesis is
formulated over the transformed one $\dmsqp$ (for instance, we know that $M'$ is a model of $\dmsqp$). The line of the proof
is then to analyze $\dmsqp$ in the light of its syntactic relationships with $\p$ established via
Lemma~\ref{lem:mappingGroundNonground}. In particular, recall first that, by Definition~\ref{def:killed}, there is a binding $\alpha$ such that $\magic{{\tt
k^\alpha(\t)}} \in N'$ (and, hence, $\magic{{k^\alpha(\t)}}$ is a ground atom in $\ensuremath{B_{\dmsqp}}$).
Thus, we can apply Lemma~\ref{lem:mappingGroundNonground} and conclude the existence of a ground rule
$\R_g'\in\ground{\dmsqp}$ such that:

\begin{dlvcode}
\R_g': \ k(\t)\,\Or\,p_1(\t_1)\,\Or\,\cdots\,\Or\,p_n(\t_n) \derives \magic{k^\alpha(\t)},
     \magic{p_1^{\alpha_1}(\t_1)}, \ldots, \\
     \quad\quad\ \
     \magic{p_n^{\alpha_n}(\t_n)},
     q_1(\s_1),\ldots,q_j(\s_j),
     \naf~q_{j+1}(\s_{j+1}),\ \ldots,\ \naf~q_m(\s_m).
\end{dlvcode}

Since $M'$ is a model of $\dmsqp$, the proof is just based on analyzing the following three scenarios that exhaustively cover all possibilities (concerning
the fact that the rule $\R_g'$ is satisfied by $M'$):

\vspace{-4mm}
\begin{description}
  \item[(S1)] $\negbody{\R_g'} \cap M' \neq \emptyset$, i.e., the negative body of $\R_g'$ is false with respect to $M'$;
  \item[(S2)] $\posbody{\R_g'} \not\subseteq M'$, i.e., the positive body of $\R_g'$ is false with respect to $M'$;
  \item[(S3)] $\negbody{\R_g'} \cap M' = \emptyset$, $\posbody{\R_g'} \subseteq M'$, and $\head{\R_g'} \cap M' \neq \emptyset$, i.e., none of the previous cases holds,
  and hence the head of $\R_g'$ is true with respect to $M'$.
\end{description}

In the remaining, we shall show that {\bf (S1)} implies condition $(1.b)$, {\bf (S2)} implies condition (2), and {\bf (S3)} implies either (2) or (3). In fact, note that
condition $(1.a)$ cannot hold.

\begin{description}
\item[(S1)] Assume that $\negbody{\R_g'} \cap M' \neq \emptyset$. Since $\negbody{\R_g} = \negbody{\R_g'}$ and
$\negbody{\R_g} \subseteq \BP$,
from $\negbody{\R_g'} \cap M' \neq \emptyset$
we immediately conclude $\negbody{\R_g} \cap \Mpp \neq \emptyset$,
i.e., $(1.b)$ holds.

\vspace{2mm}
\item[(S2)] Assume that $\posbody{\R_g'} \not\subseteq M'$, and let $\R' \in \dmsqp$ be a modified rule such that $\R'_g = \R'\vartheta$
for some substitution $\vartheta$:

\begin{dlvcode}
\hspace{-5mm}
\R': \ k(\t')\,\Or\,p_1(\t_1')\,\Or\,\cdots\,\Or\,p_n(\t_n') \derives \magic{k^\alpha(\t')},
     \magic{p_1^{\alpha_1}(\t_1')}, \ldots, \\
\hspace{-5mm}
     \quad\quad\ \
     \magic{p_n^{\alpha_n}(\t_n')},
     q_1(\s_1'),\ldots,q_j(\s_j'),
     \naf~q_{j+1}(\s_{j+1}'),\ \ldots,\ \naf~q_m(\s_m').
\end{dlvcode}

\noindent
We first claim that $\posbody{\R'_g}|_{\BP} \not\subseteq N'$ must hold in this case. To prove the claim, observe that
during the {\em Generation} step preceding the production of $\R'$,
a magic rule $\R^\magica_{i}$
such that $\head{\R^\magica_{i}} = \{\magic{p_i^{\alpha_i}(\t_i')}\}$
and $\posbody{\R^\magica_{i}} \subseteq \{\magic{k^\alpha(\t')},$ ${q_1(\s_1'),\ldots,q_j(\s_j')}\}$
has been produced
for each $1\leq i\leq n$ (we recall that magic rules have empty negative bodies).
Hence,
since the variables of $\R^\magica_{i}$ are a subset
of the variables of $\R'$,
by applying the substitution $\vartheta$ to
$\R^\magica_{i}$ we obtain a ground rule $\R^\magica_{{i},g}$
such that $\head{\R^\magica_{{i},g}}=\{\magic{p_i^{\alpha_i}(\t_i)}\}$ and
$\posbody{\R^\magica_{{i},g}}\subseteq \{\magic{{k^\alpha(\t)}},$
${q_1(\s_1),\ldots,q_j(\s_j)} \}= \{\magic{{k^\alpha(\t)}}\} \cup \posbody{r_g'}|_{\BP}$.
Thus, if $\posbody{\R'_g}|_{\BP} \subseteq N'$, from the above magic rules
and since $N'$ is a model containing ${\magic{k^\alpha(\t)}}$ by assumption, then
we would conclude that $\posbody{\R'_g} \subseteq N'$. However,
this is impossible, since $N' \subseteq M'$ and $\posbody{\R'_g} \not\subseteq M'$ imply $\posbody{\R'_g} \not\subseteq N'$.

Now, $\posbody{\R'_g}|_{\BP} \not\subseteq N'$ implies the existence of an atom ${q_i(\s_i)} \in \posbody{\R_g'}|_{\BP}$ such that ${q_i(\s_i)} \not\in N'$, that is,
${q_i(\s_i)}\in \BP\setminus N'$.
In particular, we can assume w.l.o.g. that, for any ${q(\s)} \in \posbody{\R_g'}|_{\BP}$ with ${q(\s')} \SIPSprecR{k^\alpha(\t')} {q_i(\s_i')}$, it is the case that
${q(\s)} \in N'$, where $\R$ is the rule in $\p$ from which the modified rule $\R'$ has been generated (just take a $\SIPSprecR{ k^\alpha(\t')}$-minimum element in $\posbody{\R'_g}|_{\BP}\setminus N'$).
If $q_i$ is an EDB predicate, the atom $q_i(\s_i)$ belongs to $\killedmpnp$ by the definition of killed atoms.
Otherwise, $q_i$ is an IDB predicate.
In this case, there is a magic rule $\R_{i}^\magica$, produced during the
{\em Generation} step preceding the production of $\R'$, such that
$\head{\R_{i}^\magica} = \{{\magic{q_i^{\beta_i}(\s_i')}}\}$ and
$\body{\R_{i}^\magica} = \{{\magic{k^\alpha(\t')}}\} \cup \{{q(\s')} \in \posbody{\R} \mid {q(\s')} \SIPSprecR{k^\alpha(\t')} {q_i(\s_i')}\}$.
Thus, $\R_{{i},g}^\magica = \R_{i}^\magica\vartheta$ belongs to $\ground{\dmsqp}$.
In particular, $\posbody{\R_{{i},g}^\magica} \subseteq N'$ holds because $\magic{k^\alpha(\t)}$ belongs to $N'$ and by the properties of $q_i(\s_i)$.
Therefore, since $N'$ is a model of $\ground{\dmsqp}^{M'}$, $\magic{q_i^{\beta_i}(\s_i)}$ belongs to $N'$, from which ${q_i(\s_i)} \in \killedmpnp$ follows from the definition of killed atoms.
Thus, independently of the type (EDB, IDB) of $q_i$, (2) holds.

\vspace{2mm}
\item[(S3)] Assume that $\posbody{\R'_g} \subseteq M'$,
$\negbody{\R'_g} \cap M' = \emptyset$, and $\head{\R_g'} \cap M' \neq \emptyset$.
First, observe that from $\negbody{\R'_g} \cap M' = \emptyset$ we can conclude that there is a rule
in $\ground{\dmsqp}^{M'}$ obtained from $\R'_g$ by removing its negative body literals.
Consider now the rules $\R^\magica_{{i},g}$ produced during the {\em Generation} step,
for each $1 \leq i \leq n$ (as in {\bf (S2)}). We distinguish two cases.

If $\{{q_1(\s_1),\ldots,q_j(\s_j)}\} \subseteq N'$,
since $\magic{{k^\alpha(\t)}} \in N'$, we can conclude that
$\posbody{\R^\magica_{{i},g}} \subseteq N'$, for each $1 \leq i \leq n$.
Moreover, since $N'$ is a model of $\ground{\dmsqp}^{M'}$,
the latter implies that $\magic{{p_i^{\alpha_i}(\t_i)}} \in N'$,
for each $1 \leq i \leq n$.
Then $\posbody{\R'_g} \subseteq N'$ holds, and so $\head{\R'_g} \cap N' \neq \emptyset$
(because $N'$ is a model of $\ground{\dmsqp}^{M'}$).
We now observe that $\head{\R_g'} \cap (\Mpp \setminus\ \killedmpnp) \neq \emptyset$
is equivalent to $(\head{\R_g'} \cap \Mpp) \setminus \killedmpnp \neq \emptyset$.
Moreover, the latter is equivalent to $(\head{\R_g} \cap M') \setminus \killedmpnp \neq \emptyset$
because $\head{\R_g'}$ contains only standard atoms and $\head{\R_g'} = \head{\R_g}$.
In addition, from $N' \subseteq M'$ we conclude $\head{\R_g} \cap N' \subseteq \head{\R_g} \cap M'$,
and by Definition~\ref{def:killed}, $N' \cap \killedmpnp = \emptyset$ holds.
Hence, $(\head{\R_g} \cap M') \setminus \killedmpnp \supseteq \head{\R_g} \cap N'$,
which is not empty, and so condition $(3)$ holds.

Otherwise, $\{{q_1(\s_1),\ldots,q_j(\s_j)}\} \not\subseteq N'$.
Let $i \in \{1, \ldots, j\}$ be such that ${q_i(\s_i)} \not\in N'$ and,
for any ${q(\s)} \in \posbody{\R_g'}|_{\BP}$, ${q(\s')} \SIPSprecR{k^\alpha(\t')} {q_i(\s_i')}$ implies ${q(\s)} \in N'$
(where $\R$ is the rule in $\p$ from which the modified rule $\R'$ has been generated).
If $q_i$ is an EDB predicate, the atom $q_i(\s_i)$ belongs to $\killedmpnp$ by the definition of killed atoms.
Otherwise, $q_i$ is an IDB predicate and there is a magic rule $\R_{{i},g}^\magica \in \ground{\dmsqp}$ having an atom $\magic{q_i^{\beta_i}(\s_i)}$ in head, and such that $\posbody{\R_{{i},g}^\magica} \subseteq N'$.
Therefore, $\magic{q_i^{\beta_i}(\s_i)}$ belongs to $N'$, from which ${q_i(\s_i)} \in \killedmpnp$ follows from the definition of killed atoms.
Thus, independently of the type (EDB, IDB) of $q_i$, (2) holds.
\end{description}
\vspace{-9mm}
\end{proof}

We can now complete the first part of the proof.


\begin{lemma}
\label{lem:one_minimal_model} For each stable model $M'$ of $\dmsqp$, there is a stable model $M$ of $\p$ such
that $M\supseteq \Mpp$.
\end{lemma}
\begin{proof}
Let $M$ be a stable model of $\p \cup \Mpp$, the program obtained by adding to $\p$ a fact for each atom in
$\Mpp$.
We shall show that $M$ is in
fact a stable model of $\p$ such that $M \supseteq \Mpp$. Of course,
$M$ is a model of $\p$ such that $M \supseteq \Mpp$.
So, the line of the proof is to show that if $M$ is not stable, then it is possible to build a model $N'$ of $\ground{\dmsqp}^{M'}$ such that
$N'\subset M'$, thereby contradicting the minimality of $M'$ over the models of $\ground{\dmsqp}^{M'}$.

Assume, for the sake of contradiction, that $M$ is not stable and let $N\subset M$ be a model of $\ground{\p}^{M}$.
Define $N'$ as the interpretation $(N\cap \Mpp)\cup (M'\setminus \BP)$. By construction, note that $N'\subseteq
M'$, since $M'$ coincides with $\Mpp\cup (M'\setminus \BP)$. In fact, in the case where $N'=M'$, we would have
that $N \supseteq \Mpp$, since $(N\cap \Mpp)$ and $(M'\setminus \BP)$ are disjoint. Hence, $N$ would not only be
a model for $\ground{\p}^{M}$ but also a model for $\ground{\p \cup \Mpp}^{M}$, while on the other hand $N\subset M$ holds. However, this
is impossible, since $M$ is a stable model of $\p \cup \Mpp$.
So, $N'\subset M'$ must hold. Hence, to complete the proof and get a contradiction, it remains to show that $N'$ is actually a model of $\ground{\dmsqp}^{M'}$, i.e.,
it satisfies all the rules in $\ground{\dmsqp}^{M'}$.
To this end, we have to consider the following two kinds of rules:

\begin{description}
\item[(1)] Consider a ground magic rule $\R_g^\magica\in\ground{\dmsqp}^{M'}$ such that $\posbody{\R_g^\magica} \subseteq
N'$, and let $\magic{{p^\alpha(\t)}}$ be the (only) atom in $\head{\R_g^\magica}$. Since $N' \subset M'$,
$\posbody{\R_g^\magica} \subseteq N'$ implies that $\posbody{\R_g^\magica} \subset M'$. In fact, since $M'$ is a model of $\dmsqp$
and $|\head{\R_g^\magica}| = 1$, $\magic{{p^\alpha(\t)}} \in M'$ must hold
(we recall that $\negbody{\R_g^\magica} = \emptyset$).
Moreover, since $\BP$ does not contain
any magic atom, $\magic{{p^\alpha(\t)}}$ is also contained in $M' \setminus \BP$. Thus, by the construction of
$N'$, we can conclude that $\head{\R_g^\magica} \cap N' \neq \emptyset$.

\vspace{2mm}
\item[(2)] Consider a rule obtained by removing the negative literals
from a ground modified rule $\R_g'\in\ground{\dmsqp}$ where

\begin{dlvcode}
\hspace{-5mm}
\R_g': \ p(\t)\,\Or\,p_1(\t_1)\,\Or\,\cdots\,\Or\,p_n(\t_n) \derives \magic{p^\alpha(\t)},
     \magic{p_1^{\alpha_1}(\t_1)}, \ldots, \\
\hspace{-5mm}
     \quad\quad\ \
     \magic{p_n^{\alpha_n}(\t_n)},
     q_1(\s_1),\ldots,q_j(\s_j),
     \naf~q_{j+1}(\s_{j+1}),\ \ldots,\ \naf~q_m(\s_m).
\end{dlvcode}

\noindent and where $\posbody{\R_g'} \subseteq N'$. Observe that $\negbody{\R_g'} \cap M' = \emptyset$ holds by the definition of reduct.
Moreover, let $\R_g$ be the rule of $\ground{\p}$ associated with $\R_g'$ (according to
Lemma~\ref{lem:mappingGroundNonground}):

\begin{dlvcode}
\R_g: \ p(\t)\,\Or\,p_1(\t_1)\,\Or\,\cdots\,\Or\,p_n(\t_n) \derives
    q_1(\s_1),\ \ldots,\ q_j(\s_j), \\
    \quad\quad\quad\quad\quad\quad\quad\quad\quad\quad\quad\ \
    \naf~q_{j+1}(\s_{j+1}),\ \ldots,\ \naf~q_m(\s_m).
\end{dlvcode}

We have to show that $\head{\R_g'}\cap N'\neq \emptyset$. The proof is based on establishing the following properties on
$\R_g'$ and $\R_g$:

\vspace{-3mm}\begin{eqnarray}\label{eq:beforemain} \bullet \ M\cap \killedmpmp=\emptyset;
\end{eqnarray}

\vspace{-3mm}\begin{eqnarray}\label{eq:main} \bullet\ (\head{\R_g'}\setminus M')\cap M=\emptyset;
\end{eqnarray}

\vspace{-3mm}\begin{eqnarray}\label{eq:main1} \bullet\  \negbody{\R_g'}\cap M=\emptyset;
\end{eqnarray}

\vspace{-3mm}\begin{eqnarray}\label{eq:main2} \bullet\ \head{\R_g'} \cap M' = \head{\R_g'} \cap \Mpp = \head{\R_g'} \cap M;
\end{eqnarray}

\vspace{-3mm}\begin{eqnarray}\label{eq:main3} \bullet\ \head{\R_g}\cap N\neq \emptyset.
\end{eqnarray}

In particular, we shall directly prove (\ref{eq:beforemain}), and show the following implications: (\ref{eq:beforemain})$\rightarrow$(\ref{eq:main})$\wedge$(\ref{eq:main1}),
(\ref{eq:main})$\rightarrow$(\ref{eq:main2}), and (\ref{eq:main1})$\rightarrow$(\ref{eq:main3}).
Eventually, based on (\ref{eq:main2}) and (\ref{eq:main3}), the fact that $\head{\R_g'}\cap N'\neq \emptyset$ can be easily derived as follows:
Since $\head{\R_g} \subseteq \BP$, by the definition of $N'$ we can conclude that $\head{\R_g} \cap
N' = \head{\R_g} \cap (N \cap \Mpp) = (\head{\R_g} \cap N) \cap (\head{\R_g} \cap \Mpp)$. Moreover, because of (\ref{eq:main2}) and the fact that $\head{\R_g} = \head{\R_g'}$,
$\head{\R_g} \cap N'$ coincides in turn with  $(\head{\R_g} \cap N) \cap (\head{\R_g} \cap M)$.
Then, recall that $N\subset M$. Thus, $\head{\R_g} \cap N'=\head{\R_g} \cap N$, which is not empty by (\ref{eq:main3}).

In order to complete the proof, we have to show that all the above equations actually hold.

\medskip

\emph{Proof of (\ref{eq:beforemain}).} We recall that, by Proposition~\ref{prop:killed_unfounded}, we already know
that $\killedmpmp$ is an unfounded set for $\p$ with respect to $\tuple{\Mpp,\BP}$. In fact, one may notice that $\killedmpmp$ is an
unfounded set for $\p\cup \Mpp$ with respect to $\tuple{\Mpp,\BP}$ too, since the rules added to $\p$  are facts corresponding to the
atoms in $\Mpp$ and $\Mpp \cap \killedmpmp = \emptyset$ by Definition~\ref{def:killed}.
Thus, since $M\supseteq \Mpp$ and $M$ is a stable model of $\p\cup \Mpp$, we can apply
Theorem~\ref{theo:unfounded} in order to conclude that $M\cap \killedmpmp=\emptyset$.

\medskip

\emph{Proof of (\ref{eq:main}).} After (\ref{eq:beforemain}), we can just show that $\head{\R_g'}\setminus M' \subseteq \killedmpmp$.
In fact, since $N'\subset M'$, we note that $\posbody{\R_g'} \subseteq N'$ implies $\posbody{\R_g'} \subset M'$.
Thus, $\head{\R_g'}\setminus M' \subseteq \killedmpmp$ follows by Definition~\ref{def:killed} and the form of rule $\R_g'$.

\medskip

\emph{Proof of (\ref{eq:main1}).} After (\ref{eq:beforemain}), we can just show that $\negbody{\R_g'} \subseteq \killedmpmp$.
Actually, we show that the IDB atoms in $\negbody{\R_g'}$ belong to
$\killedmpmp$, as EDB atoms in $\negbody{\R_g'}$ clearly belong to $\killedmpmp$
because $\negbody{\R_g'} \cap M' = \emptyset$ by assumption.
To this end, consider a modified rule $\R' \in \dmsqp$ such that $\R'_g = \R'\vartheta$
for some substitution $\vartheta$:

\begin{dlvcode}
\hspace{-5mm}
\R': \ p(\t')\,\Or\,p_1(\t_1')\,\Or\,\cdots\,\Or\,p_n(\t_n') \derives \magic{p^\alpha(\t')},
     \magic{p_1^{\alpha_1}(\t_1')}, \ldots, \\
\hspace{-5mm}
     \quad\quad\ \
     \magic{p_n^{\alpha_n}(\t_n')},
     q_1(\s_1'),\ldots,q_j(\s_j'),
     \naf~q_{j+1}(\s_{j+1}'),\ \ldots,\ \naf~q_m(\s_m').
\end{dlvcode}

\noindent
During the {\em Generation} step preceding the production of $\R'$,
a magic rule $\R^\magica_{i}$
with $\head{\R^\magica_{i}} = \{\magic{q_i^{\beta_i}(\s_i')}\}$
and where $\posbody{\R^\magica_{i}} \subseteq \posbody{\R'}$
has been produced
for each $j+1\leq i\leq m$ such that $q_i$ is an IDB predicate.
Hence,
since the variables of $\R^\magica_{i}$ are a subset
of the variables of $\R'$,
the substitution $\vartheta$ can be used to map
$\R^\magica_{i}$ to a ground rule $\R^\magica_{{i},g} = \R^\magica_{i}\vartheta$
with $\head{\R^\magica_{{i},g}}=\{\magic{q_i^{\beta_i}(\s_i)}\}$
and $\posbody{\R^\magica_{{i},g}}\subseteq \posbody{\R_g'}$.
%
Now, since $\posbody{\R_g'} \subseteq N'\subset M'$, we can conclude that $\posbody{\R^\magica_{{i},g}}$ is in turn contained in
$M'$. Thus, the head of $\R^\magica_{{i},g}$ must be true
with respect to $M'$ (we recall that magic rules have empty negative bodies). That is, $\magic{{q_i^{\beta_i}(\s_i)}}\in M'$ holds, for each $j+1\leq i\leq m$
such that $q_i$ is an IDB predicate.
Moreover,  $\negbody{\R_g'} \cap M' = \emptyset$ implies that ${q_i^{\beta_i}(\s_i)}\in \BP\setminus M'$, as ${q_i^{\beta_i}(\s_i)}\in \negbody{\R_g'}$. Thus,
by Definition~\ref{def:killed}, ${q_i^{\beta_i}(\s_i)}\in \killedmpmp$.

\medskip

\emph{Proof of (\ref{eq:main2}).} The property immediately follows from (\ref{eq:main}) and the fact that
$\head{\R_g'} \subseteq \BP$ and $M \supseteq \Mpp$.

\medskip

\emph{Proof of (\ref{eq:main3}).} Note that $\negbody{\R_g} = \negbody{\R_g'}$,
and so (\ref{eq:main1}) implies that there is a rule in $\ground{\p}^{M}$ obtained from $\R_g$
by removing the atoms in $\negbody{\R_g}$. Note also that
$\posbody{\R_g}=\posbody{\R'_g}\cap \BP\subseteq
N'\cap \BP$ (since $\posbody{\R_g'} \subseteq N'$). Thus, by the definition of $N'$, $\posbody{\R_g}\subseteq N$ (more
specifically, $\posbody{\R_g} \subseteq N \cap \Mpp$). Moreover, since $N$ is a model of $\ground{\p}^{M}$,
the latter entails that $\head{\R_g}\cap N\neq \emptyset$.
\end{description}
\vspace{-9mm}
\end{proof}

\begin{theorem}
\label{thm:extending_minimal_models} Let $\Q$ be a query for a $\datdn$ program $\p$. Then, for each stable
model $M'$ of $\dmsqp$, there is a stable model $M$ of $\p$ such that $M'|_\Q = M|_\Q$.
\end{theorem}
\begin{proof}
Because of Lemma~\ref{lem:one_minimal_model}, for each stable model $M'$ of $\dmsqp$, there is a stable model
$M$ of $\p$ such that $M\supseteq \Mpp$. Thus, we trivially have that $M|_\Q \supseteq M'|_\Q$ holds. We now
show that the inclusion cannot be proper.

In fact, by the definition of $\dmsqp$, the magic seed is associated to any ground instance of $\Q$. Then
$\BP|_{\Q} \setminus M' \subseteq \killedmpmp$ by Definition~\ref{def:killed}
(we recall that $\BP|_{\Q}$ denotes the ground instances of $\Q$). By
Proposition~\ref{prop:killed_unfounded}, $\killedmpmp$ is an unfounded set for $\p$ with respect to $\tuple{\Mpp,\BP}$. Hence, by
Theorem~\ref{theo:unfounded}, we have that $M\cap \killedmpmp=\emptyset$. It follows that $M\cap (\BP|_{\Q}
\setminus M')=\emptyset$. Thus, \ $M|_\Q \setminus M'|_\Q=\emptyset$, \ which \ combined \ with \ $M|_\Q \supseteq M'|_\Q$
implies $M|_\Q = M'|_\Q$.
\end{proof}

\subsubsection{Completeness of the Magic Set Method}

For the second part of the proof, we construct an interpretation for $\dmsqp$ based on one for $\p$.

\begin{definition}[Magic Variant]
\label{def:magic_variant} Let $I$ be an interpretation for $\p$. We define an interpretation $\variantqpi{\infty}$ for
$\dmsqp$, called the magic variant of $I$ with respect to $\Q$ and $\p$, as the limit of the following sequence:
$$
\begin{array}{l}
\begin{array}{lcl}
\variantqpi{0} & = & \EDB{\p}; \mbox{ and} \\
\variantqpi{i+1} & = & \variantqpi{i}\ \cup \\
\end{array}\\
\quad\quad \{ {p(\t)} \in I \ \mid \mbox{ there is a binding }
              \alpha \mbox{ such that }\\
\quad\quad\quad\quad\quad\quad\quad\quad {\magic{p^\alpha(\t)}} \in \variantqpi{i} \} \ \cup \\
\quad\quad \{ {\magic{p^\alpha(\t)}} \ \mid \ \exists \ \R_g^\magica \in \ground{\dmsqp} \mbox{ such that }\\
\quad\quad\quad\quad {\magic{p^\alpha(\t)}} \in \head{\R_g^\magica}
          \mbox{ and } \posbody{\R_g^\magica} \subseteq \variantqpi{i} \}, \ \ \ \forall i\geq 0.
\end{array}
$$
\end{definition}

\begin{example}\label{ex:magic_variant}\em
Consider the program $\DMS(\Q_{sc},\p_{sc})$ presented in Section~\ref{sec:dms},
the EDB $\{{\tt produced\_by(p,c,c_1)}\}$
and the interpretation $M_{sc} = \{{\tt produced\_by(p,c,c_1), sc(c)}\}$.
We next compute the magic variant $\variant{\Q_{sc}}{\p_{sc}}{\infty}{M_{sc}}$
of $M_{sc}$ with respect to $\Q_{sc}$ and $\p_{sc}$.
We start the sequence with the original EDB:
$\variant{\Q_{sc}}{\p_{sc}}{0}{M_{sc}} = \{{\tt produced\_by(p,c,c_1)}\}$.
For $\variant{\Q_{sc}}{\p_{sc}}{1}{M_{sc}}$, we add $\tt magic\_sc^{\tt b}(c)$ (the query seed),
while for $\variant{\Q_{sc}}{\p_{sc}}{2}{M_{sc}}$, we add $\tt sc(c)$ (because
${\tt sc(c)} \in M_{sc}$ and
${\tt magic\_sc^{\tt b}(c)} \in \variant{\Q_{sc}}{\p_{sc}}{0}{M_{sc}}$),
and $\tt magic\_sc^{\tt b}(c_1)$ (because
$\tt magic\_sc^{\tt b}(c_1) \derives magic\_sc^{\tt b}(c).$ is a rule of $\ground{\DMS(\Q_{sc},\p_{sc})}$
and ${\tt magic\_sc^{\tt b}(c)} \in \variant{\Q_{sc}}{\p_{sc}}{0}{M_{sc}}$).
Any other element of the sequence coincides with $\variant{\Q_{sc}}{\p_{sc}}{2}{M_{sc}}$,
and so also $\variant{\Q_{sc}}{\p_{sc}}{\infty}{M_{sc}}$.
\hfill $\Box$
\end{example}

By definition, for a magic variant $\variantqpi{\infty}$ of an interpretation $I$ with respect to $\Q$ and $\p$, $\variantqpi{\infty}|_{\BP} \subseteq
I$ holds. More interestingly, the magic variant of a stable model for $\p$ is in turn a stable model for $\dmsqp$.

\begin{example}\em
The magic variant of $M_{sc}$ with respect to $\Q_{sc}$ and $\p_{sc}$
(see Example~\ref{ex:magic_variant})
coincides with the interpretation $M_{sc}'$ introduced in Example~\ref{ex:killed}.
From previous examples, we know that $M_{sc}$ is a stable model of $\p_{sc}$,
and $M_{sc}'$ is a stable model of $\DMS(\Q_{sc},\p_{sc})$.
\hfill $\Box$
\end{example}

The following two lemmas formalize the intuition above, with the latter being the counterpart of Lemma~\ref{lem:one_minimal_model}.

\begin{lemma}
\label{lem:magic_variant_minimal_model}
For each stable model $M$ of $\p$, the magic variant $M'=\variantqpm{\infty}$ of $M$ is a model of $\ground{\dmsqp}^{M'}$ with $M\supseteq \Mpp$.
\end{lemma}
\begin{proof}
As $M'$ is the magic variant of the stable model $M$, we trivially have that $M \supseteq \Mpp$ holds. We next show that
$M'$ is a model of $\ground{\dmsqp}^{M'}$. To this end, consider a rule
in $\ground{\dmsqp}^{M'}$ having the body true,
that is, a rule obtained by removing the negative body literals from
a rule $\R_g'\in \ground{\dmsqp}$ such that $\negbody{\R_g'} \cap M' = \emptyset$
and $\posbody{\R_g'}\subseteq M'$ hold. We have to show that $\head{\R_g'} \cap M'\neq \emptyset$.

In the case where $\R_g'$ is a magic rule, then $\posbody{\R_g'} \subseteq M'$ implies that the (only) atom in
$\head{\R_g'}$ belongs to $M'$ (by Definition~\ref{def:magic_variant}).
The only remaining (slightly more involved) case to be analyzed is where $\R_g'$ is a modified rule of the form

\begin{dlvcode}
\hspace{-5mm}
\R_g': \ p(\t)\,\Or\,p_1(\t_1)\,\Or\,\cdots\,\Or\,p_n(\t_n) \derives \magic{p^\alpha(\t)},
     \magic{p_1^{\alpha_1}(\t_1)}, \ldots, \\
\hspace{-5mm}
     \quad\quad\ \
     \magic{p_n^{\alpha_n}(\t_n)},
     q_1(\s_1),\ldots,q_j(\s_j),
     \naf~q_{j+1}(\s_{j+1}),\ \ldots,\ \naf~q_m(\s_m).
\end{dlvcode}

\noindent
In this case, we first apply as usual Lemma~\ref{lem:mappingGroundNonground} in order to conclude the existence of a rule $\R_g\in
\ground{\p}$ of the form

\begin{dlvcode}
\hspace{-5mm}
\R_g: \ p(\t)\,\Or\,p_1(\t_1)\,\Or\,\cdots\,\Or\,p_n(\t_n) \derives
    q_1(\s_1),\ \ldots,\ q_j(\s_j), \\
\hspace{-5mm}
    \quad\quad\quad\quad\quad\quad\quad\quad\quad\quad\quad\ \
    \naf~q_{j+1}(\s_{j+1}),\ \ldots,\ \naf~q_m(\s_m).
\end{dlvcode}

\noindent
Then, we claim that the following two properties hold:
\vspace{-3mm}
\begin{eqnarray}
\label{eq:correct} \bullet && \negbody{\R_g} \cap M = \emptyset;\\
\label{eq:correct1} \bullet &&  \posbody{\R_g} \subseteq M.
\end{eqnarray}
These properties are in fact what we just need to establish the result. Indeed, since $M$
is a model of $\ground{\p}^{M}$, (\ref{eq:correct}) and (\ref{eq:correct1}) imply $\head{\R_g} \cap M \neq \emptyset$. So, we can recall that
$\head{\R_g} = \head{\R_g'}$, and hence let $p_i(\t_i)$ be an atom in
$\head{\R_g} \cap M=\head{\R_g'} \cap M$
and $\magic{p_i^{\alpha_i}(\t_i)}$ be its corresponding magic atom in
$\posbody{\R_g'}$ ($i \in \{\epsilon, 1, \ldots, n\}$, where $\epsilon$ is the
empty string).
Since $\posbody{\R_g'} \subseteq
M'$ (by hypothesis) and since ${p_i(\t_i)}\in M$, we can then conclude that
${p_i(\t_i)}$ is in $M'$ as well by
Definition~\ref{def:magic_variant}. That is, $\head{\R_g'} \cap M'\neq \emptyset$.

Let now finalize the proof, by showing that the above properties actually hold.

\emph{Proof of (\ref{eq:correct}).} Consider a modified rule $\R' \in \dmsqp$ such that $\R'_g = \R'\vartheta$ for a substitution $\vartheta$:

\begin{dlvcode}
\hspace{-5mm}
\R': \ p(\t')\,\Or\,p_1(\t_1')\,\Or\,\cdots\,\Or\,p_n(\t_n') \derives \magic{p^\alpha(\t')},
     \magic{p_1^{\alpha_1}(\t_1')}, \ldots, \\
\hspace{-5mm}
     \quad\quad\ \
     \magic{p_n^{\alpha_n}(\t_n')},
     q_1(\s_1'),\ldots,q_j(\s_j'),
     \naf~q_{j+1}(\s_{j+1}'),\ \ldots,\ \naf~q_m(\s_m').
\end{dlvcode}

\noindent
and the rule $\R \in \p$ from which $\R'$ is produced
(such that $\R_g = \R\vartheta$):

\begin{dlvcode}
\hspace{-5mm}
\R: \ p(\t')\,\Or\,p_1(\t_1')\,\Or\,\cdots\,\Or\,p_n(\t_n') \derives
    q_1(\s_1'),\ \ldots,\ q_j(\s_j'), \\
\hspace{-5mm}
    \quad\quad\quad\quad\quad\quad\quad\quad\quad\quad\quad\ \
    \naf~q_{j+1}(\s_{j+1}'),\ \ldots,\ \naf~q_m(\s_m').
\end{dlvcode}

\noindent
During the {\em Generation} step preceding the production of $\R'$,
a magic rule $\R^\magica_{i}$
such that $\head{\R^\magica_{i}} = \{\magic{q_i^{\beta_i}(\s_i')}\}$
has been produced
for each $j+1\leq i\leq m$ such that $q_i$ is an IDB predicate.
Hence,
since the variables of $\R^\magica_{i}$ are a subset
of the variables of $\R'$,
the substitution $\vartheta$ can be used to map
$\R^\magica_{i}$ to a ground rule $\R^\magica_{{i},g} = \R^\magica_{i}\vartheta$
such that $\head{\R^\magica_{{i},g}}=\{\magic{q_i^{\beta_i}(\t_i)}\}$ and
$\posbody{\R^\magica_{{i},g}}\subseteq
\posbody{\R_g'}$
(we recall that magic rules have empty negative body).
Now, since $\posbody{\R_g'} \subseteq M'$, we can conclude that $\posbody{\R^\magica_{{i},g}}$ is in turn contained in
$M'$. Thus, by the construction of $M'$, the head of $\R^\magica_{{i},g}$ must be true
with respect to $M'$, that is, $\magic{{q_i^{\beta_i}(\t_i)}}\in M'$ holds for each $j+1\leq i\leq m$
such that $q_i$ is an IDB predicate.
So, if some (IDB) atom ${q_i(\s_i)} \in \negbody{\R_g}$ belongs
to $M$, by Definition~\ref{def:magic_variant} we can conclude that
${q_i(\s_i)} \in M'$, which contradicts the assumption
that $\negbody{\R_g'} \cap M' = \emptyset$ (we recall that $\negbody{\R_g} = \negbody{\R_g'}$).
%
This proves that IDB predicates in $\negbody{\R_g}$ do not occur in $M$. The same trivially holds for EDB predicates too, since
$\negbody{\R_g}\cap M' = \negbody{\R_g'}\cap M'=\emptyset$ and $M'\supseteq \EDB{\p}$ (by the definition of magic variant).

\medskip

\emph{Proof of (\ref{eq:correct1}).} The equation straightforwardly follows from the fact that $\posbody{\R_g} = \posbody{\R_g'}|_{\BP}$, and
since $M \supseteq \Mpp$ and $\posbody{\R_g'} \subseteq M'$ hold by the construction of $M'$ and by the initial hypothesis on the choice of $\R_g'$, respectively.
\end{proof}

\medskip

\begin{lemma}
\label{lem:magic_variant_minimal_modelDUE}
For each stable model $M$ of $\p$, there is a stable model $M'$ of $\dmsqp$ (which is the magic variant of $M$) such that $M\supseteq \Mpp$.
\end{lemma}
\begin{proof}
After Lemma~\ref{lem:magic_variant_minimal_model},
we can show that $M'=\variantqpm{\infty}$ is also minimal over all the models of $\ground{\dmsqp}^{M'}$.
Let $N' \subseteq M'$ be a minimal model of $\ground{\dmsqp}^{M'}$. We prove by
induction on the definition of the magic variant that $M'$ is in turn contained in $N'$. The base case (i.e., $\variantqpm{0} \subseteq
N'$) is clearly true, since $\variantqpm{0}$ contains only EDB facts.
Suppose $\variantqpm{i} \subseteq N'$ in order to prove that
$\variantqpm{i+1} \subseteq N'$ holds as well.

While considering an atom in $\variantqpm{i+1} \setminus \variantqpm{i}$, we distinguish two cases:

\begin{description}
\item[(a)] For a magic atom $\magic{p^\alpha(\t)}$ in $\variantqpm{i+1} \setminus \variantqpm{i}$, by
Definition~\ref{def:magic_variant} there must be a rule $\R_g^\magica \in \ground{\dmsqp}$
having $\head{\R_g^\magica} = \{\magic{p^\alpha(\t)}\}$ and
$\posbody{\R_g^\magica} \subseteq \variantqpm{i}$
(we recall that magic rules have empty negative body and so
$\R_g^\magica \in \ground{\dmsqp}^{M'}$ holds).
We can then conclude that $\posbody{\R_g^\magica}
\subseteq N'$ holds by the induction hypothesis and so $\magic{{p^\alpha(\t)}} \in N'$ (because $N'$ is a model
of $\ground{\dmsqp}^{M'}$).

\item[(b)] For a standard atom $p(\t)$ in $\variantqpm{i+1} \setminus \variantqpm{i}$, by
Definition~\ref{def:magic_variant} there is a binding $\alpha$ such that ${\magic{p^\alpha(\t)}} \in
\variantqpm{i}$ and the atom $p(\t)$ belongs to $M$. Assume for the sake of contradiction that ${p(\t)}\not\in
N'$. Since $M'$ is a model of $\dmsqp$ and $N'$ is a model of $\ground{\dmsqp}^{M'}$,
we can compute the set $\killedmpnp$ as introduced in Section~\ref{sec:soundness}
and note, in particular, that ${p(\t)}\in \killedmpnp$ holds (by definition).
Moreover, by Proposition~\ref{prop:killed_unfounded}, $\killedmpnp$ is an unfounded set for $\p$ with respect to $\tuple{\Mpp,\BP}$.
In addition, $M \supseteq \Mpp$ holds by Definition~\ref{def:magic_variant}.
Thus, $M$ is a stable model for $\p$ such that $M\supseteq \Mpp$, and we can
hence apply Theorem~\ref{theo:unfounded} in order to conclude that $M \cap \killedmpnp =
\emptyset$. The latter is in contradiction with ${p(\t)}\in \killedmpnp$ and
${p(\t)} \in M$. Hence, ${p(\t)}\in N'$.
\end{description}
\vspace{-9mm}
\end{proof}

We can then prove the correspondence of stable models with respect to queries.

\begin{theorem}
\label{theo:extending_minimal_models} Let $\Q$ be a query for a $\datdn$ program $\p$. Then, for each stable
model $M$ of $\p$, there is a stable model $M'$ of $\dmsqp$ (which is the
magic variant of $M$)
such that $M'|_\Q = M|_\Q$.
\end{theorem}
\begin{proof}
Let $M$ be a stable model of $\p$ and $M' = \variantqpm{\infty}$ its magic variant. Because of
Lemma~\ref{lem:magic_variant_minimal_modelDUE}, $M'$ is a stable model of $\dmsqp$ such that $M \supseteq
\Mpp$. Thus, we trivially have that $M|_\Q \supseteq M'|_\Q$ holds. We now show the reverse inclusion.

Since $M'$ is a stable model of $\dmsqp$, we can determine the set $\killedmpmp$ as defined in
Section~\ref{sec:soundness}. Hence, by Definition~\ref{def:killed} we can conclude that (a) $\BP|_{\Q}
\setminus M' \subseteq \killedmpmp$ because $M'$ contains the magic seed by construction
(we recall that $\BP|_{\Q}$ denotes the ground instances of $\Q$). Moreover, since $M$ is
a stable model of $\p$ with $M\supseteq \Mpp$ and $\killedmpmp$ is an unfounded set for $\p$ with respect to $\tuple{\Mpp,\BP}$ by
Proposition~\ref{prop:killed_unfounded}, we can conclude that (b) $M \cap
\killedmpmp = \emptyset$ by Theorem~\ref{theo:unfounded}. Thus, by combining (a) and (b) we obtain that $(\BP|_{\Q} \setminus M') \cap M =
\emptyset$, which is equivalent to $M|_\Q \subseteq M'|_\Q$.
\end{proof}

Finally, we show the correctness of the Magic Set method with respect to query 
answering, that is, we prove that the original and rewritten programs
provide the same answers for the input query on all possible EDBs.

\begin{theorem}
\label{theo:dms_equivalence} Let $\p$ be a $\datdn$ program, and let $\Q$ be a query. Then $\dmsqp \bqequiv{\Q}
\p$ and $\dmsqp \cqequiv{\Q} \p$ hold.
\end{theorem}
\begin{proof}
We want to show that, for any set of facts $\F$ defined over the EDB predicates
of $\p$ (and $\DMS(\Q,\p)$),
$\Ans_b(\Q, \DMS(\Q,\p) \cup \F) = \Ans_b(\Q, \p \cup \F)$ and
$\Ans_c(\Q, \DMS(\Q,\p) \cup \F) = \Ans_c(\Q, \p \cup \F)$ hold.
We first observe that the Magic Set rewriting does not depend on EDB facts;
thus, $\dmsqp \cup \F = \DMS(\Q, \p \cup \F)$ holds.
Moreover, note that $\datdn$ programs always have stable models.
Therefore, as a direct consequence of Theorem~\ref{thm:extending_minimal_models}
and Theorem~\ref{theo:extending_minimal_models}, we can conclude
$\Ans_b(\Q, \DMS(\Q,\p \cup \F)) = \Ans_b(\Q, \p \cup \F)$ and
$\Ans_c(\Q, \DMS(\Q,\p \cup \F)) = \Ans_c(\Q, \p \cup \F)$.
\end{proof}

\subsection{Magic Sets for Stratified $\dat$ Programs without Disjunction}
\label{sec:no_disjunction}

Stratified $\dat$ programs without disjunction have exactly one stable model \cite{gelf-lifs-88}.
However, the Magic Set transformation can introduce new dependencies between predicates,
possibly resulting in unstratified programs (we refer to the analysis in \cite{kemp-etal-95}).
Clearly, original and rewritten programs agree on the query, as proved in the previous section,
but the question whether the rewritten program admits a unique stable model
is also important.
In fact, for programs having the unique stable model property,
brave and cautious reasoning coincide and a solver can immediately answer the
query after the first (and unique) stable model is found.
The following theorem states that the rewritten program of a stratified program indeed has a unique stable model.

\begin{theorem}
\label{theo:stratified_unique}
Let $\p$ be a disjunction-free $\dat$ program with stratified negation
and $\Q$ a query.
Then $\dmsqp$ has a unique stable model.
\end{theorem}
\begin{proof}
Let $M$ be the unique stable model of $\p$,
and $M' = \variantqpm{\infty}$ its magic variant as presented in Definition~\ref{def:magic_variant}.
By Lemma~\ref{lem:magic_variant_minimal_modelDUE} we already know that
$M'$ is a stable model of $\dmsqp$.
We now show that any stable model $N'$ of $\dmsqp$ contains $M'$
by induction on the structure of $M'$.
The base case ($\variantqpm{0} \subseteq N'$) is clearly true, since $\variantqpm{0}$ contains only
EDB facts.
Suppose $\variantqpm{i} \subseteq N'$ in order to prove that $\variantqpm{i+1} \subseteq N'$
holds as well.
Thus, while considering an atom in $\variantqpm{i+1} \setminus \variantqpm{i}$, two cases are
possible:
\begin{description}
\item[(1)] For a magic atom $\magic{p^\alpha(\t)}$ in $\variantqpm{i+1} \setminus \variantqpm{i}$, by
Definition~\ref{def:magic_variant} there must be a rule $\R_g^\magica \in \ground{\dmsqp}$ having $\head{\R_g^\magica} =
\{\magic{p^\alpha(\t)}\}$ and $\posbody{\R_g^\magica} \subseteq \variantqpm{i}$
(we recall that magic rules have empty negative bodies and so
$\R_g^\magica \in \ground{\dmsqp}^{N'}$ holds).
We can then conclude that $\posbody{\R_g^\magica}
\subseteq N'$ holds by the induction hypothesis and so $\magic{{p^\alpha(\t)}} \in N'$
(because $N'$ is a model of $\ground{\dmsqp}^{N'}$).

\item[(2)] For a standard atom $p(\t)$ in $\variantqpm{i+1} \setminus \variantqpm{i}$, by
Definition~\ref{def:magic_variant} there is a binding $\alpha$ such that ${\magic{p^\alpha(\t)}} \in
\variantqpm{i}$ and the atom $p(\t)$ belongs to $M$. Assume for the sake of contradiction that ${p(\t)}\not\in
N'$. Since $N'$ is a stable model of $\dmsqp$, we can compute the set $\killed{N'}{N'}{\Q}{\p}$ as introduced in
Section~\ref{sec:soundness} and note, in particular, that ${p(\t)}\in \killed{N'}{N'}{\Q}{\p}$ holds, by definition.
Moreover, by Proposition~\ref{prop:killed_unfounded}, $\killed{N'}{N'}{\Q}{\p}$ is an unfounded set for $\p$ with respect to $\tuple{\Npp,\BP}$.
In addition, by Lemma~\ref{theo:extending_minimal_models} there is a stable model $N$ of $\p$
such that $N \supseteq \Npp$, which would mean that ${p(\t)} \not\in N$ holds.
Hence, we can conclude that $N$ and $M$ are two different stable models of $\p$,
obtaining a contradiction, as $\p$ has a unique stable model.
\end{description}

Since stable models are incomparable with respect to containment, $M' \subseteq N'$
implies $M' = N'$. Hence, $M'$ is the unique stable model of $\dmsqp$.
\end{proof}

\section{Implementation}\label{sec:system}

The Dynamic Magic Set method ($\DMS$) has been implemented and
integrated into the core of the \dlv \cite{leon-etal-2002-dlv}
system. In this section, we shall first briefly describe the
architecture of the system and its usage. We then
briefly present an optimization for eliminating redundant rules, which
are sometimes introduced during the Magic Set rewriting.

\subsection{System Architecture and Usage}

We have created a prototype system by implementing the Magic Set technique described in Section~\ref{sec:magic_set-disjunctive_datalog}
inside \dlv, as shown in the architecture reported in Figure~\ref{fig:architecture}.
\dlv supports both brave and cautious reasoning, and for a completely ground
query it can be also used for computing all stable models in which the query is true.
\dlv performs brave reasoning if invoked with
the command-line option {\tt -FB}, while {\tt -FC} indicates cautious reasoning.

In our prototype, the $\DMS$ algorithm is applied automatically by default when the user invokes \dlv with {\tt -FB} or {\tt -FC}
together with a (partially) bound query. Magic Sets are not applied by default if the query does not contain any constant.
The user can modify this default behavior by specifying the command-line options {\tt -ODMS} (for applying Magic Sets) or {\tt -ODMS-} (for disabling Magic Sets).

If a completely bound query is specified, \dlv can print the magic variant of the stable model (not displaying magic predicates), which witnesses the truth
(for brave reasoning) or the falsity (for cautious reasoning) of the query,
by specifying the command-line option {\tt --print-model}.

Within \dlv, \DMS\ is applied immediately after parsing the program
and the query by the {\em Magic Set Rewriter} module.  The rewritten
(and optimized as described in Section~\ref{sec:redundant_rules})
program is then processed by the {\em Intelligent Grounding} module
and the {\em Model Generator} module using the implementation of \dlv.
The only other modification is for the output and its filtering: For ground queries, the
witnessing stable model is no longer printed by default, but only if
{\tt --print-model} is specified, in which case the magic predicates
are omitted from the output.

The SIPS schema\footnote{Since technically a SIPS has a definition for
  every single rule, implementations use a schema for creating the
  SIPS for a given rule.}  implemented in the prototype is as follows:
For a rule $r$, head atom $p(\t)$ and binding $\alpha$,
$\prec^{p^\alpha(\t)}_r$ satisfies the conditions of
Definition~\ref{def:sip2}, in particular ${p(\t)} \prec^{
  p^\alpha(\t)}_r {q(\s)}$ holds for all ${q(\s)} \neq {
  p(\t)}$ in $r$, and ${q(\s)} \not \prec^{p^\alpha(\t)}_r
{b(\z)}$ holds for all head or negative body atoms ${q(s)}
\neq {p(\t)}$ and any atom $b(\z)$ in $r$.  Moreover, all the
positive body literals of $r$ form a chain in $\prec^{p^\alpha(\t)}_r$. This chain is constructed by iteratively inserting
those atoms containing most bound arguments (considering $\alpha$ and
also the partially formed chain and $f^{p^\alpha(\t)}_r$) into the
chain. Among the atoms with most bindings an arbitrary processing
order (usually the order appearing in the original rule body) is used.
Furthermore, $f^{p^\alpha(\t)}_r({q(\s)}) = {X}$ holds if
and only if $q(\s)$ belongs to the positive body of $r$, has at
least one bound argument and $X$ occurs in $\s$.

This means that apart from the head atom via which the rule is
adorned, only positive body atoms can yield variable bindings and only
if at least one of their arguments is bound, but both atoms with EDB
and IDB predicates can do so. Moreover, atoms with more bound
arguments will be processed before those with fewer bound arguments.

Note that in this work we did not study the impact of trying different
SIPS schemas, as we wanted to focus on showing the impact that our
technique can have, rather than fine-tuning its parameters. While we
believe that the SIPS schema employed is well-motivated, there
probably is quite a bit of room for improvement, which we leave for
future work.

\begin{figure}
 \centering
 \includegraphics[width=0.75\textwidth]{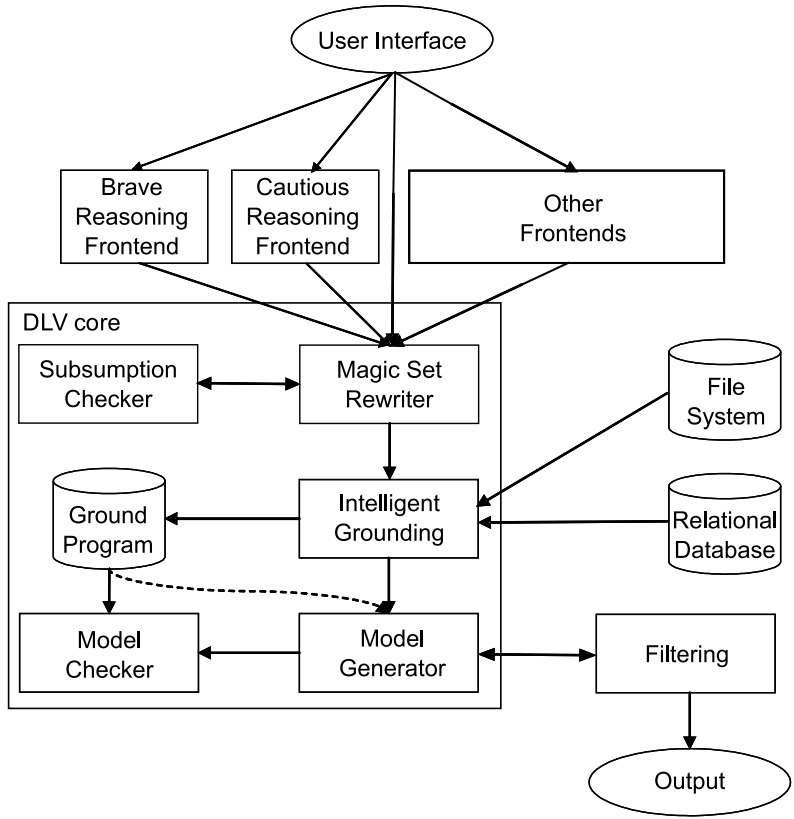}
 \caption{Prototype system architecture}
 \label{fig:architecture}
\end{figure}

An executable of the \dlv system  supporting the Magic Set optimization is available at \url{http://www.dlvsystem.com/magic/}.

\subsection{Dealing with Redundant Rules}\label{sec:redundant_rules}

Even though our rewriting algorithm keeps the amount of generated rules low,
it might happen that
some redundant rules are generated when adorning disjunctive rules, thereby somewhat deteriorating the optimization effort.
For instance, in Example~\ref{running4} the first two modified rules are semantically equivalent, and this might happen
even if the two head predicates differ. In general not only duplicated rules might be created, but
also rules which are logically subsumed by other rules in the program. Let us first give the definition of subsumption for $\datdn$ rules.

\begin{definition}\label{def:subsumed}
Let $\P$ be a $\datdnr$ program, and let $\R$ and $\R'$ be two rules of $\P$. Then, $\R$ is
\emph{subsumed} by $\R'$ (denoted by $\R \sqsubseteq \R'$) if there exists a substitution $\vartheta$ for
the variables of $\R'$, such that $\head{\R'} \vartheta \subseteq \head{\R}$ and $\body{\R'} \vartheta
\subseteq \body{\R}$. A rule $\R$ is \emph{redundant} if there exists a rule $\R'$ such that
$\R \sqsubseteq \R'$.
\end{definition}

Ideally, a Magic Set rewriting algorithm should be capable of identifying all the possible redundant rules and
removing them from the output. Unfortunately, this approach is unlikely to be feasible in polynomial time, given
that subsumption checking on first-order expressions is \NP-complete (problem [LO18] in \cite{gare-john-79}).

Thus, in order to identify whether a rule $\R$ produced during the Magic Set transformation is redundant, we
pragmatically apply a greedy subsumption algorithm in our implementation, for checking whether $\R \sqsubseteq \R'$
holds for some rule $\R'$. In particular, the employed heuristics aims at building the substitution $\vartheta$ (as
in Definition~\ref{def:subsumed}) by iteratively choosing an atom $p(\t)$ (which is not yet processed) from
$\R'$ and by matching it (if possible) with some atom of $\R$. The greedy approach prefers
those atoms of $\R'$ with the maximum number of variables not yet matched.

To turn on subsumption checking (applied once after the Magic Set rewriting), \dlv\ has to be
invoked with the command-line option {\tt -ODMS+}.

\section{Experiments on Standard Benchmarks}\label{sec:experiments}

We performed several experiments for assessing the
effectiveness of the proposed technique.
In this section we present the results obtained on various standard benchmarks, most of
which have been directly adopted from the literature.
Further experiments on an application scenario using real-world data
will be discussed in detail in
Section~\ref{sec:dataintegration}. 
We also refer to \cite{lian-etal-2009,moti-satt-2006} that contain
performance evaluations involving $\DMS$; in \cite{lian-etal-2009}
$\dlv$ with $\DMS$ was tested on Semantic Web reasoning tasks and
confronted with a heterogeneous set of systems, in
\cite{moti-satt-2006} the system KAON2, which includes a version of
$\DMS$, is confronted against other ontology systems. In both
publications the impact of magic sets is stated explicitly.

\subsection{Compared Methods, Benchmark Problems and Data}
\label{sec:experiments-settings}

In order to evaluate the impact of the proposed method, we have
compared $\DMS$ (using the SIPS defined outlined in Section~\ref{sec:system}) both with the traditional \dlv evaluation without
\emph{Magic Sets} and with the $\SMS$ method proposed in
\cite{grec-2003}. Concerning $\SMS$, we were not able to obtain an
implementation, and have therefore performed the rewriting
manually. As a consequence, the runtime measures obtained for $\SMS$
do not contain the time needed for rewriting, while it is included for
$\DMS$.

For the comparison, we consider the following benchmark problems. The first three of them had been already used to assess $\SMS$ in
\cite{grec-2003}, to which we refer for details:

\begin{itemize}
\item \emph{Simple Path:} Given a directed graph $G$ and two nodes $a$ and
$b$, does there exist a unique path connecting $a$ to $b$ in $G$?
The instances are encoded by facts $\tt edge(v_1,v_2)$ for each arc $(v_1,v_2)$ in $G$, while the problem itself is encoded by the program\footnote{The first rule of the program models
that for each node $X$ of $G$, a unique path connecting $X$ with itself can
either exist or not.}
\begin{dlvcode}
\tt sp(X,X)\ \Or\ not\_\, sp(X,X) \derives edge(X,Y).\\
\tt sp(X,Y)\ \Or\ not\_\, sp(X,Y) \derives sp(X,Z),\ edge(Z,Y).\\
\tt path(X,Y) \derives sp(X,Y).\\
\tt path(X,Y) \derives not\_\, sp(X,Y).\\
\tt not\_\, sp(X,Z) \derives path(X,Y_1),\ path(X,Y_2),\ Y_1<>Y_2,\\
\tt \phantom{not\, sp(X,Z) \derives} edge(Y_1,Z), edge(Y_2,Z).
\end{dlvcode}
with the query $\tt sp(a,b)$.
The structure of the graph, which is the same as the one reported in
\cite{grec-2003}, consists of a square matrix of nodes connected as shown in Figure~\ref{fig:istanzeSP},
and the instances have been generated by varying
of the number of nodes.

\begin{figure}[t]
 \centering
 \subfigure{\centering \includegraphics[height=145pt]{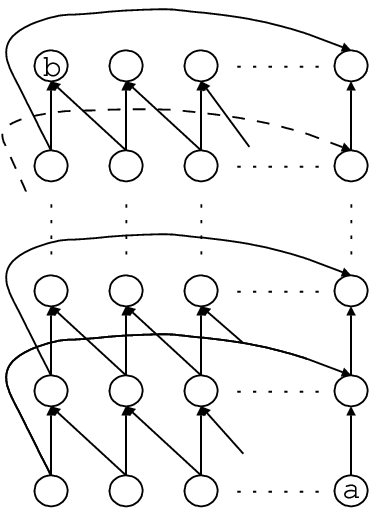}}
 \hspace{1.5cm}
 \subfigure{\centering \includegraphics[height=145pt]{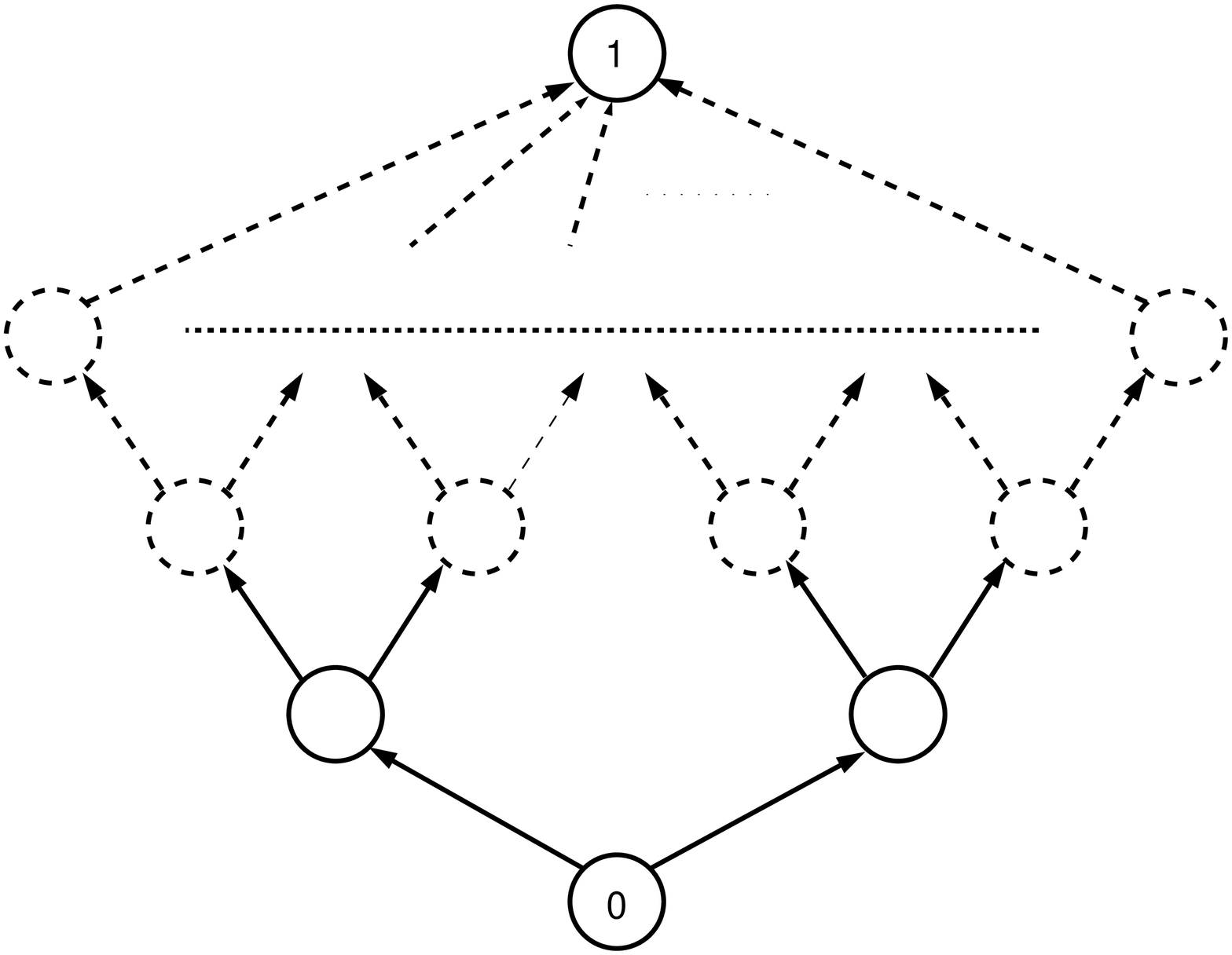}}
 \caption{Instances structure of \emph{Simple Path} and \emph{Related} (left)
  and of \emph{Conformant Plan Checking} (right)} \label{fig:istanzeSP}
\end{figure}

\item \emph{Related:} Given a genealogy graph storing information about
relationships (father/brother) among people and given two people $p_1$ and
$p_2$, is $p_1$ an ancestor of $p_2$?
The instances are encoded by facts $\tt related(p_1,p_2)$ when $p_1$ is known to be related to $p_2$, that is, when $p_1$ is the father or a brother of $p_2$.
The problem can be encoded by the program
\begin{dlvcode}
\tt father(X,Y)\ \Or\ brother(X,Y) \derives related(X,Y). \\
\tt ancestor(X,Y) \derives father (X,Y). \\
\tt ancestor(X,Y) \derives father(X,Z),\ ancestor(Z,Y).
\end{dlvcode}
and the query is $\tt ancestor(p_1,p_2)$.
The structure of the ``genealogy'' graph is the same as the one presented in
\cite{grec-2003} and coincides with the one used for testing \emph{Simple Path}.
Also in this case, the instances are generated by varying
the number of nodes (thus the number of persons in the genealogy) of the
graph.

\item \emph{Strategic Companies:} This is a slight variant of the problem domain used in the running example. The description here is of the problem as posed in the Third ASP Competition. 
We consider a collection $C$  of  companies, where each company produces some
goods in a set $G$ and each company ${c_i}\in C$ is controlled by a set of owner companies $O_{i} \subseteq
C$. A subset of the companies $C' \subset C$ is a \emph{strategic set} if it is a minimal set of companies
producing all the goods in $G$, such that if $O_{i} \subseteq C'$ for some $i=1,\ldots,m$ then ${c_i} \in C'$
must hold. As in the Second Answer Set
Competition,\footnote{\url{http://www.cs.kuleuven.be/~dtai/events/ASP-competition/index.shtml}} we assume that
each product is produced by at most four companies, and that each company is controlled by at most four
companies (the complexity of the problem under these restrictions is as hard as without them). Given two distinct
companies ${c_i}, {c_j} \in C$, is there a strategic set of $C$ which contains both $c_i$ and $c_j$?
The instances are encoded by facts $\tt produced\_by(p,c_1,c_2,c_3,c_4)$ when product $p$ is produced by companies $c_1, c_2, c_3,$ and $c_4$; if $p$ is produced by fewer than four companies (but at least one), then $c_1,c_2,c_3,c_4$ contains repetitions of companies. Moreover, facts $\tt controlled\_by(c,c_1,c_2,c_3,c_4)$ represent that company $c$ is controlled by companies $c_1, c_2, c_3,$ and $c_4$; again, if $c$ is controlled by fewer than four companies, then $c_1,c_2,c_3,c_4$ contains repetitions.
The problem can be encoded by the program
\begin{dlvcode}
\tt st(C_1)\ \Or\ st(C_2)\ \Or\ st(C_3)\ \Or\ st(C_4) \derives produced\_by(P,C_1,C_2,C_3,C_4). \\
\tt st(C) \derives controlled\_by(C,C_1,C_2,C_3,C_4), st(C_1), st(C_2), st(C_3), st(C_4).
\end{dlvcode}
with the query $\tt st(c_i),\ st(c_j)$. While the language presented in the previous sections allowed only for
one atom in a query for simplicity, the implementation in \dlv{} allows for a conjunction in a query; it is easy
to see that a conjunctive query can be emulated by a rule with the conjunction in the body and an atom with a
new predicate in the head, which contains all body arguments, and finally replacing the query conjunction with
this atom. In this case this would mean adding a rule $\tt q(c_i,c_j) \derives st(c_i),\ st(c_j)$ and replacing
the query by $\tt q(c_i,c_j)$. For this benchmark we used the instances submitted for the Second Answer Set
Competition.

\item \emph{Conformant Plan Checking:}
In addition, we have included a benchmark problem, which highlights
the fact that our Magic Set technique can yield improvements not only
for the grounding, but also for the model generation phase, as discussed in Section~\ref{sec:relatedwork}. This problem is
inspired by a setting in planning, in particular testing whether a
given plan is conformant with respect to a state transition
diagram \cite{gold-bodd-96}. Such a diagram is essentially a directed graph formed of
nodes representing states, and in which arcs are labeled by actions,
meaning that executing the action in the source state will lead to the
target state. In the considered setting nondeterminism is allowed,
that is, executing an action in one state might lead
nondeterministically to one of several successor states. A plan is a
sequence of actions, and it is conformant with respect to a given
initial state and a goal state if each possible execution of the
action sequence leads to the goal state.

In our benchmark, we assume that the action selection process has
already been done, thus having reduced the state transition diagram to
those transitions that actually occur when executing the given
plan. Furthermore we assume that there are exactly two possible
non-goal successor states for any given state. This can also be viewed
as whether all outgoing paths of a node in a directed graph reach a
particular confluence node. We encoded instances by facts $\tt
ptrans(s_0,s_1,s_2)$ meaning that one of states $\tt s_1$ and $\tt s_2$ will be
reached in the plan execution starting from $\tt s_0$. The problem is
encoded using
\begin{dlvcode}
\tt trans(X,Y)\ \Or\ trans(X,Z)\ \derives\ ptrans(X,Y,Z).\\
\tt reach(X,Y)\ \derives\ trans(X,Y).\\
\tt reach(X,Y)\ \derives\ reach(X,Z),\ trans(Z,Y).
\end{dlvcode}
and the query $\tt reach(0,1)$, where $0$ is the initial state and
$1$ the goal state.  If the query is cautiously true, the plan is
conformant.  The transition graphs in our experiments
have the shape of a binary tree rooted in state $0$, and from each leaf
there is an arc to state $1$, as depicted in
Figure~\ref{fig:istanzeSP}.
\end{itemize}

In addition, we have performed further experiments on an application
scenario modeled from real-world data for answering user queries in a
data integration setting. These latter experiments will be discussed
in more detail in Section~\ref{sec:dataintegration}.

\subsection{Results and Discussion}

The experiments have been performed on a 3GHz Intel$^{\scriptsize\textregistered}$ Xeon$^{\scriptsize\textregistered}$
processor system with 4GB RAM under the Debian 4.0 operating system with a GNU/Linux 2.6.23 kernel.
The \dlv prototype used has been compiled using GCC 4.3.3. For each instance,
we have allowed a maximum running time of 600 seconds (10 minutes) and a
maximum memory usage of 3GB.

On all considered problems, $\DMS$ outperformed $\SMS$, even if
$\SMS$ does not include the rewriting time, as discussed in
Section~\ref{sec:experiments-settings}. Let us analyze the
results for each problem in more detail.

The results for \emph{Simple Path} are reported in
Figure~\ref{fig:simplepath}. \dlv without Magic Sets solves only
the smallest instances, with a very steep increase in execution
time. $\SMS$ does better than \dlv, but scales much worse than
$\DMS$. The difference between $\SMS$ and $\DMS$ is mostly due to the
grounding of the additional predicates that $\SMS$ introduces.

\begin{figure}[t]
 \centering
\includegraphics{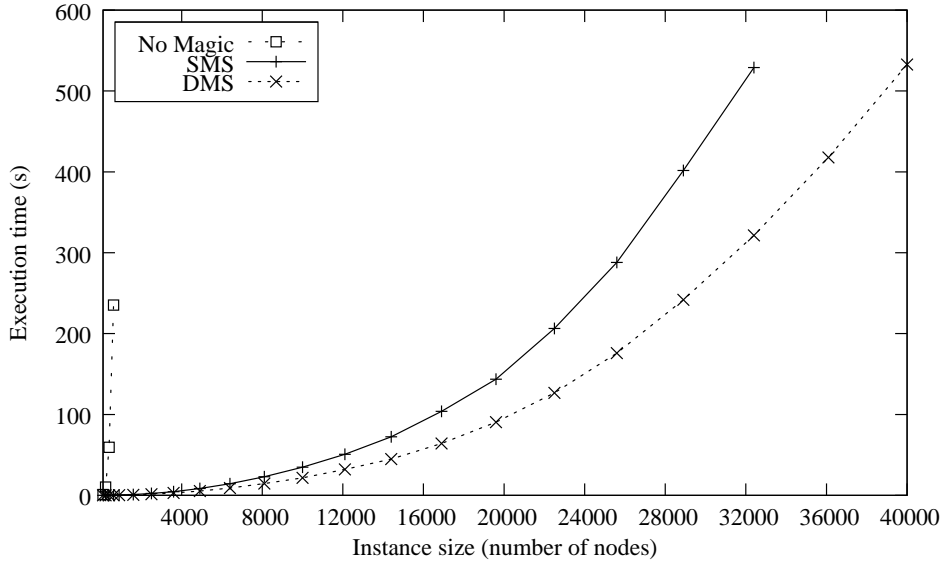}
 \caption{{\em Simple Path:} Average execution time}
 \label{fig:simplepath}
\end{figure}

Figure~\ref{fig:related} reports the results for \emph{Related}.
Compared to Simple Path, \dlv without
Magic Sets exhibits an even steeper increase in runtime, while in
contrast both $\SMS$ and $\DMS$ scale better than on Simple
Path.
Comparing $\SMS$ and $\DMS$,
we note that $\DMS$ appears to have an exponential speedup over $\SMS$.
In this case, the computational gain of $\DMS$ over $\SMS$ is due to the
dynamic optimization of the model search phase resulting from our Magic Sets
definition. This aspect is better highlighted by the
{\em Conformant Plan Checking} benchmark, and will be discussed later in this
section.

\begin{figure}[t]
 \centering
 \includegraphics{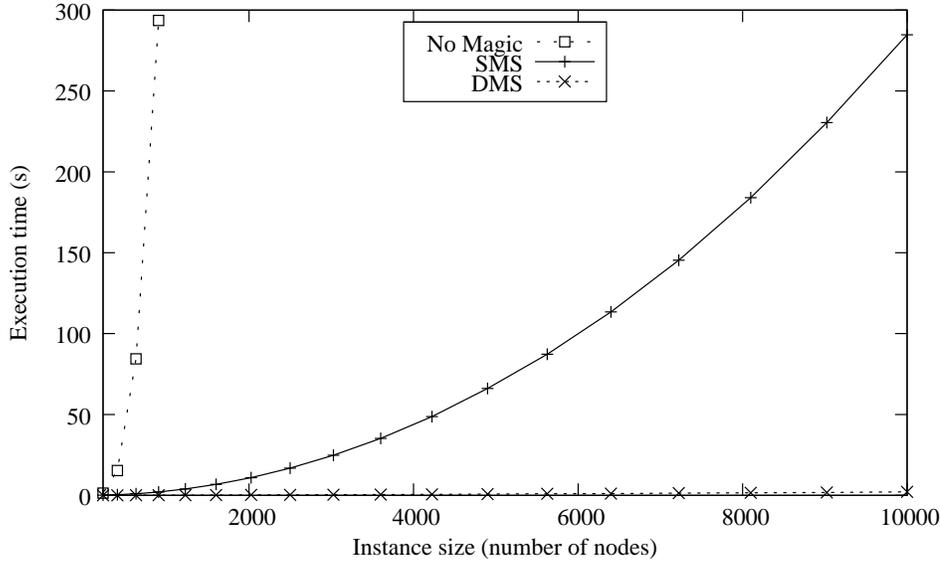}
 \caption{{\em Related:} Average execution time}
 \label{fig:related}
\end{figure}

For \emph{Strategic Companies}, we report the results in
Figure~\ref{fig:strategicCompanies} as a bar diagram, because the
instances do not have a uniform structure. The instances are, however,
ordered by size. Also here, \dlv without Magic Sets is clearly the
least efficient of the tested systems, resolving only the smallest two
instances in the allotted time (600 seconds). Concerning the other systems, $\SMS$
and $\DMS$ essentially show equal performance. In fact, the
situation here is quite different to Simple Path and Related, because
grounding the program produced by the Magic Set rewriting takes only a
negligible amount of time for $\SMS$ and $\DMS$. For this benchmark
the important feature is reducing the ground program to the part which is
relevant for the query, and we could verify that the ground programs
produced by $\SMS$ and $\DMS$ are precisely the same.

\begin{figure}[t]
 \centering
 \includegraphics{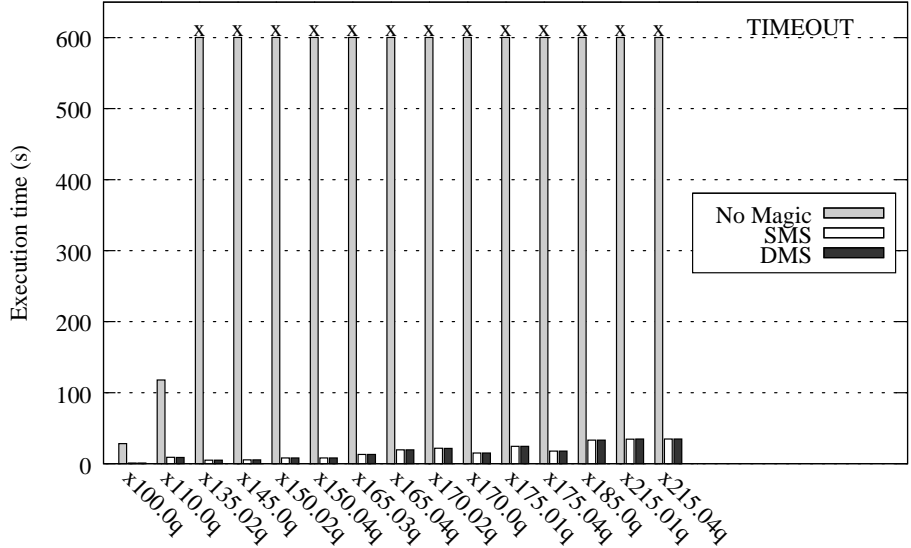}
 \caption{{\em Strategic Companies:} Average execution time}
 \label{fig:strategicCompanies}
\end{figure}

Finally, the results for {\em Conformant Plan Checking} are shown in
Figure~\ref{fig:planning}. While \dlv shows a similar behavior as
for Simple Path and Related, here also $\SMS$ does not scale well at
all, and in fact $\DMS$ appears to have an exponential
speedup over $\SMS$. There is a precise reason for this: While the
Magic Set rewriting of $\SMS$ always creates a deterministic program
defining the magic predicates, this is not true for $\DMS$. As a
consequence, all magic predicates are completely evaluated during the
grounding phase of \dlv for $\SMS$, while for $\DMS$ this is not the
case. At the first glance, this may seem like a disadvantage of
$\DMS$, as one might believe that the ground program becomes
larger. However, it is actually a big advantage of $\DMS$, because it
offers a more precise identification of the relevant part of the
program. Roughly speaking, whatever $\SMS$ identifies as relevant for
the query will also be identified as relevant in $\DMS$, but $\DMS$
can also include nondeterministic relevance information, which $\SMS$
cannot.  This means that in $\DMS$ Magic Sets can be exploited also
during the nondeterministic search phase of \dlv, dynamically
disabling parts of the ground program. In particular, after having
made some choices, parts of the program may no longer be relevant to
the query, but only because of these choices, and the magic atoms
present in the ground program can render these parts satisfied, which
means that they will no longer be considered in this part of the
search. $\SMS$ cannot induce any behavior like this and its effect is
limited to the grounding phase of \dlv, which can make a huge
difference, as evidenced by {\em Conformant Plan Checking}.

\begin{figure}[t]
 \centering
 \includegraphics{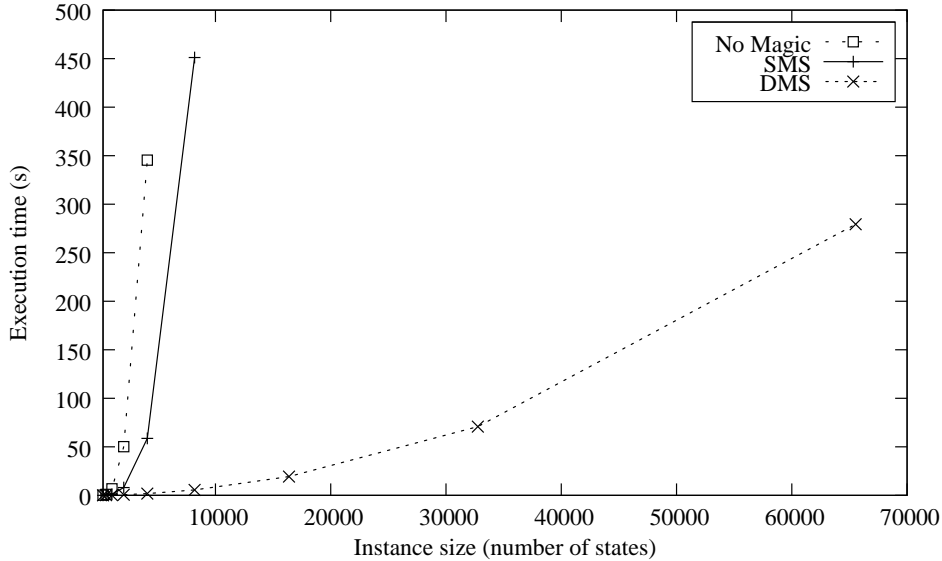}
 \caption{{\em Conformant Plan Checking:} Average execution time}
 \label{fig:planning}
\end{figure}

\subsection{Experimenting DMS with other Disjunctive Datalog Systems}

In order to assess the effectiveness of $\DMS$ on other systems than \dlv,
we tested the grounder Gringo \cite{gebs-etal-2007-lpnmr} with the following
solvers:
ClaspD \cite{dres-etal-2008-KR},
Cmodels \cite{lier-2005-lpnmr},
GnT1 and GnT2 \cite{janh-etal-2005-tocl}.
ClaspD is based on advanced Boolean constraint solving techniques, featuring
backjumping and conflict-driven learning.
Cmodels is based on the definition of program completion and loop formula for
disjunctive programs \cite{lee-lifs-2003,lin-zhao-2002}, and uses a SAT solver
for generating candidate solutions and testing them.
GnT1 is based on Smodels \cite{simo-etal-2002}, a system handling Datalog programs
with unstratified negation (normal programs): A disjunctive program is
translated into a normal program, the stable models of which are computed by
Smodels and represent stable model candidates of the original program. Each of
these candidates is then checked to be a stable model of the original program
by invoking Smodels on a second normal program.
GnT2 is a variant of GnT1 in which the number of candidates
produced by the first normal program is reduced by means of additional rules
that discard unsupported models, i.e., models containing some atom $a$ for which
there is no rule $\R$ such that $\body{\R}$ is true and $a$ is the only true
atom in $\head{\R}$.

All of the benchmarks presented in the previous section were tested on these
systems. Since $\DMS$ is not implemented in these systems, rewritten programs
were produced by \dlv during the preparation of the experiment. We recall
that $\DMS$ does not depend on EDB relations and point out that \dlv computes
rewritten programs for the considered encodings in 1-2 hundredths of a second.
The results of our experiment are reported in
Figures~\ref{fig:simplepath_other}--\ref{fig:planning_other}.
In general, we tried use a consistent scales in the graphs in order to
ease comparability. However, for some graphs we chose a different
scale in order to keep them readable for the main purpose (comparing
performances with and without $\DMS$), and we mention this explicitly
in the accompanying text.

\begin{figure}[t]
 \centering
 \includegraphics[width=0.45\textwidth,viewport=5 10 250 140]{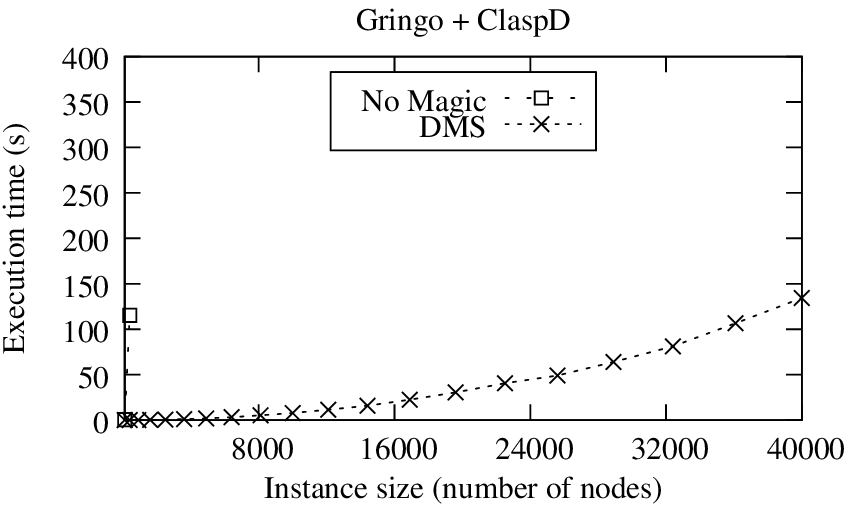}
 \includegraphics[width=0.45\textwidth,viewport=5 10 250 140]{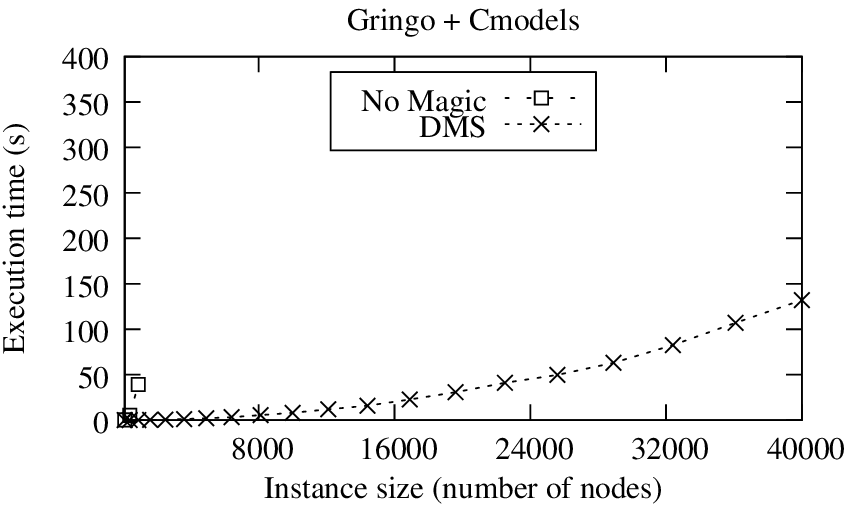}
 \\\vspace{1.5em}
 \includegraphics[width=0.45\textwidth,viewport=5 10 250 140]{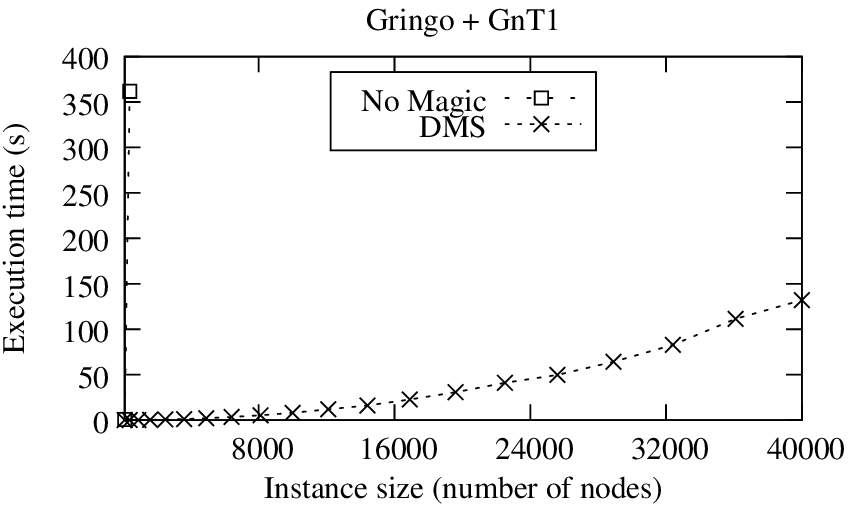}
 \includegraphics[width=0.45\textwidth,viewport=5 10 250 140]{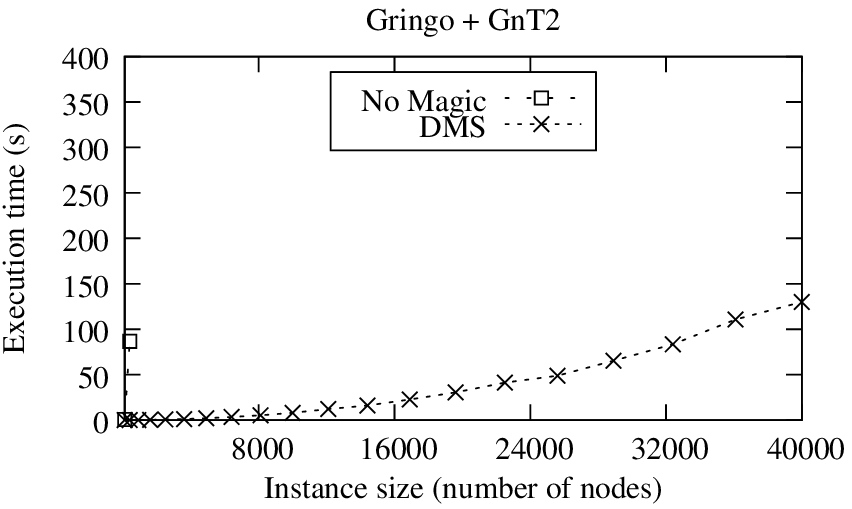}
 \caption{{\em Simple Path:} Average execution time on other systems}
 \label{fig:simplepath_other}
\end{figure}

Concerning \emph{Simple Path}, the advantages of $\DMS$ over the unoptimized
encoding are evident on all tested systems. In fact, as shown in
Figure~\ref{fig:simplepath_other}, without $\DMS$ all tested systems did not
answered in the allotted time (600 seconds) on instances with more than 400
nodes (900 for Cmodels). On the other hand, all of the instances considered in
the benchmark (up to 40 thousands of nodes) were solved by all tested solvers
with the $\DMS$ encoding. We also observe that with $\DMS$ the tested systems
are faster than \dlv in this benchmark, which is a clear indication of the
optimization potential that can be provided to these systems by our Magic Set
technique.

\begin{figure}[t]
 \centering
 \includegraphics[width=0.45\textwidth,viewport=5 10 250 140]{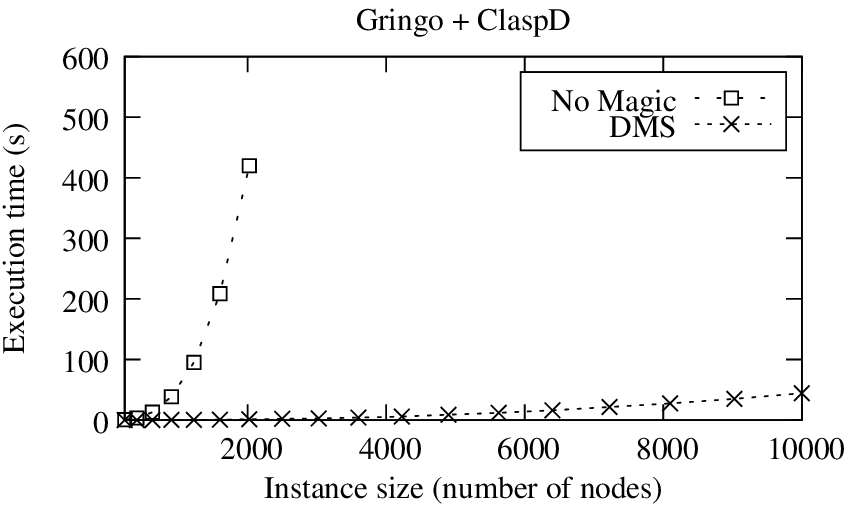}
 \includegraphics[width=0.45\textwidth,viewport=5 10 250 140]{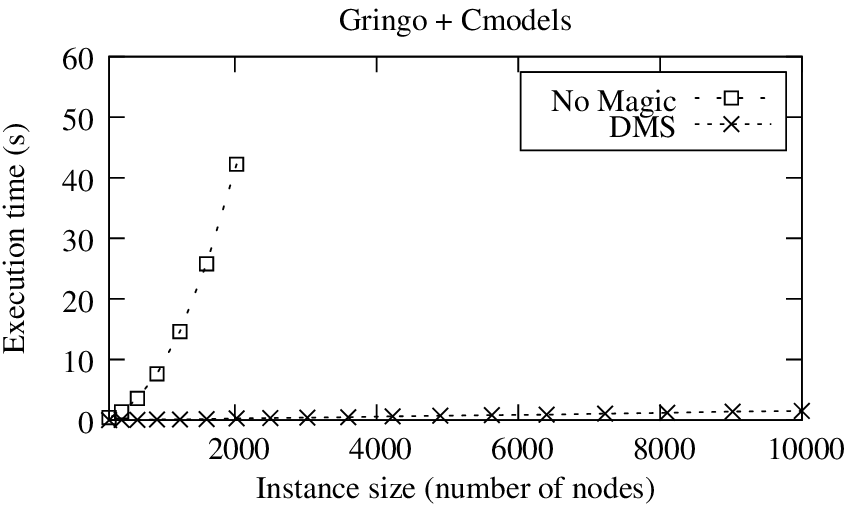}
 \\\vspace{1.5em}
 \includegraphics[width=0.45\textwidth,viewport=5 10 250 140]{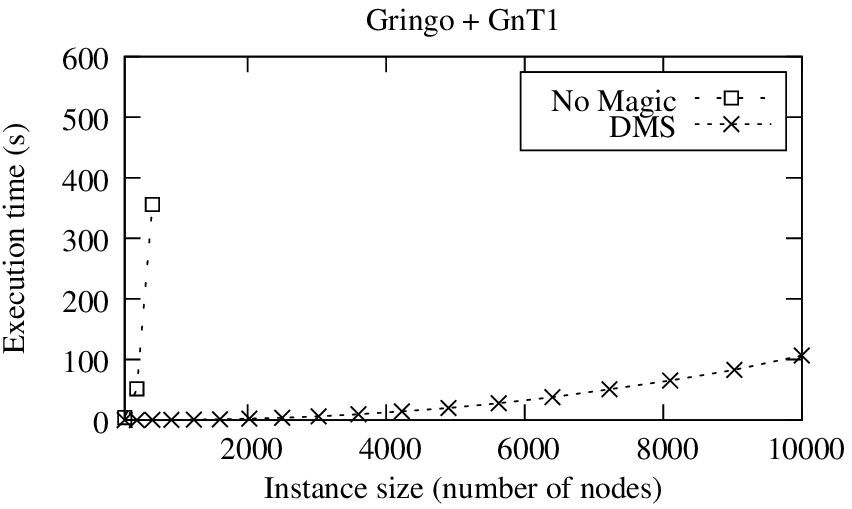}
 \includegraphics[width=0.45\textwidth,viewport=5 10 250 140]{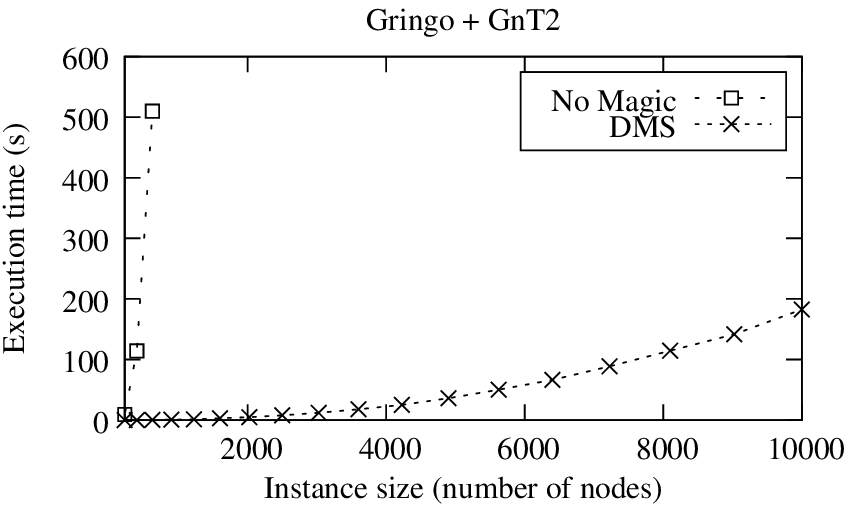}
 \caption{{\em Related:} Average execution time on other systems}
 \label{fig:related_other}
\end{figure}

For \emph{Related} we obtained a similar result, reported in
Figure~\ref{fig:related_other} (we used a different scale for the
y-axis for Cmodels for readability).
Without $\DMS$ only the smallest instances
were solved in the allotted time (up to 2025 nodes for ClaspD and Cmodels,
up to 625 nodes for GnT1 and GnT2). With $\DMS$, instead, all tested systems
solved the biggest instances of the benchmark (up to 10 thousands of nodes).
In particular, with $\DMS$ Cmodels is as performant as \dlv in this benchmark.

\begin{figure}[t]
 \centering
 \includegraphics{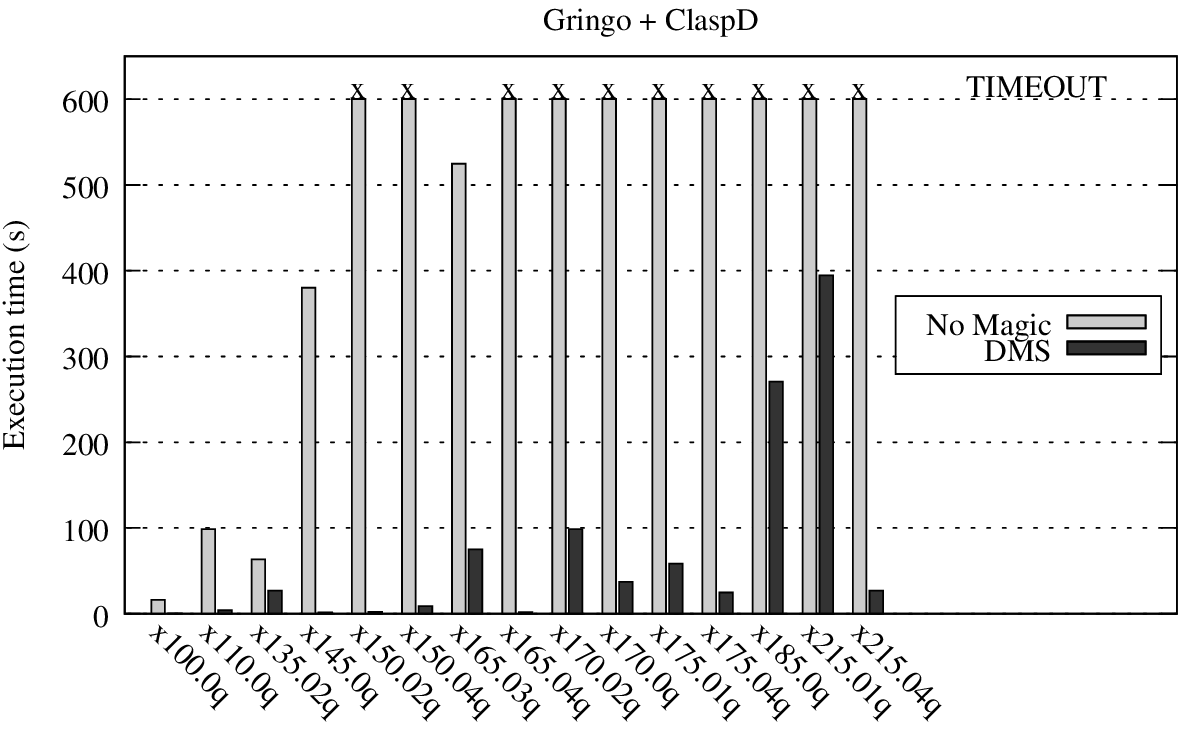}
 \\\vspace{1.5em}
 \includegraphics{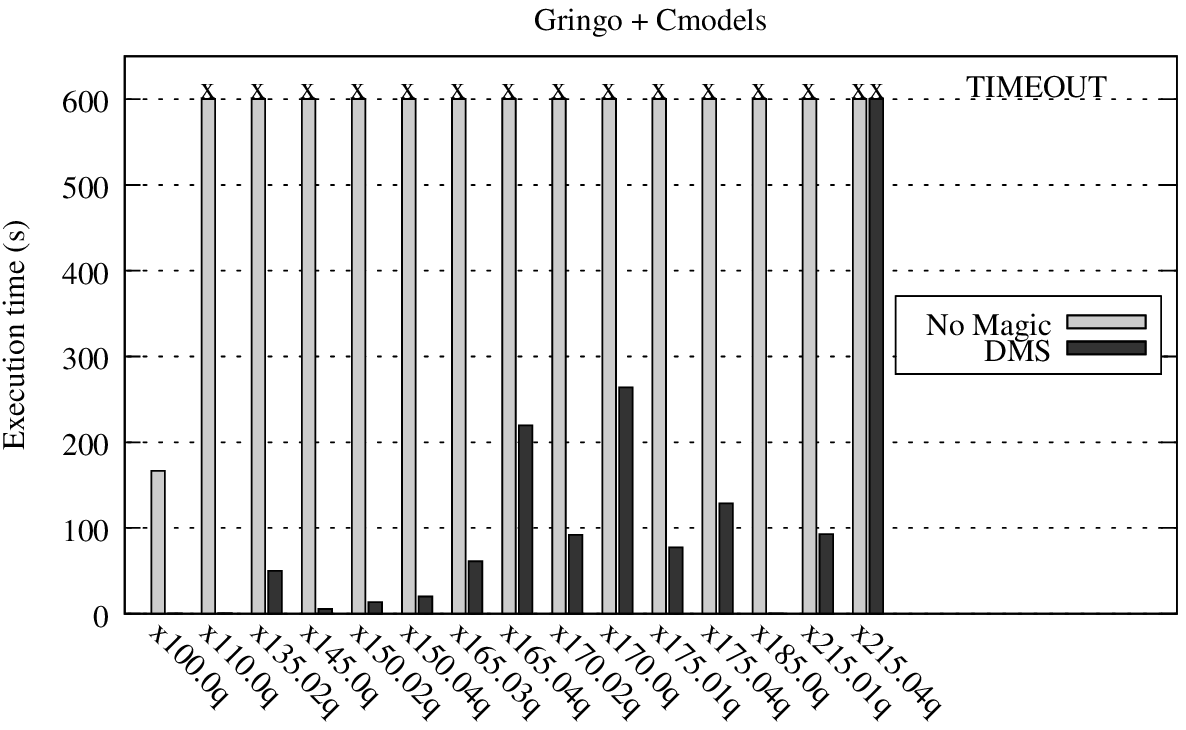}
 \caption{{\em Strategic Companies:} Average execution time on other systems (part 1)}
 \label{fig:strategiccompanies_other1}
\end{figure}
\begin{figure}[t]
 \centering
 \includegraphics{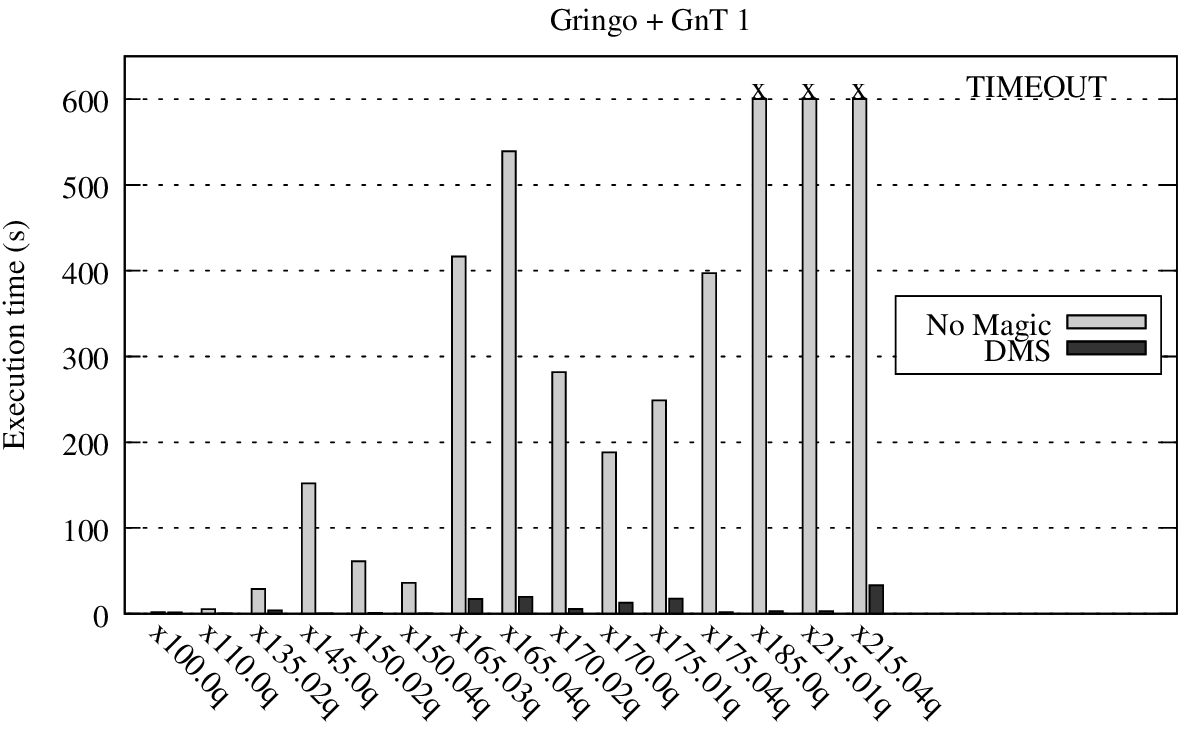}
 \\\vspace{1.5em}
 \includegraphics{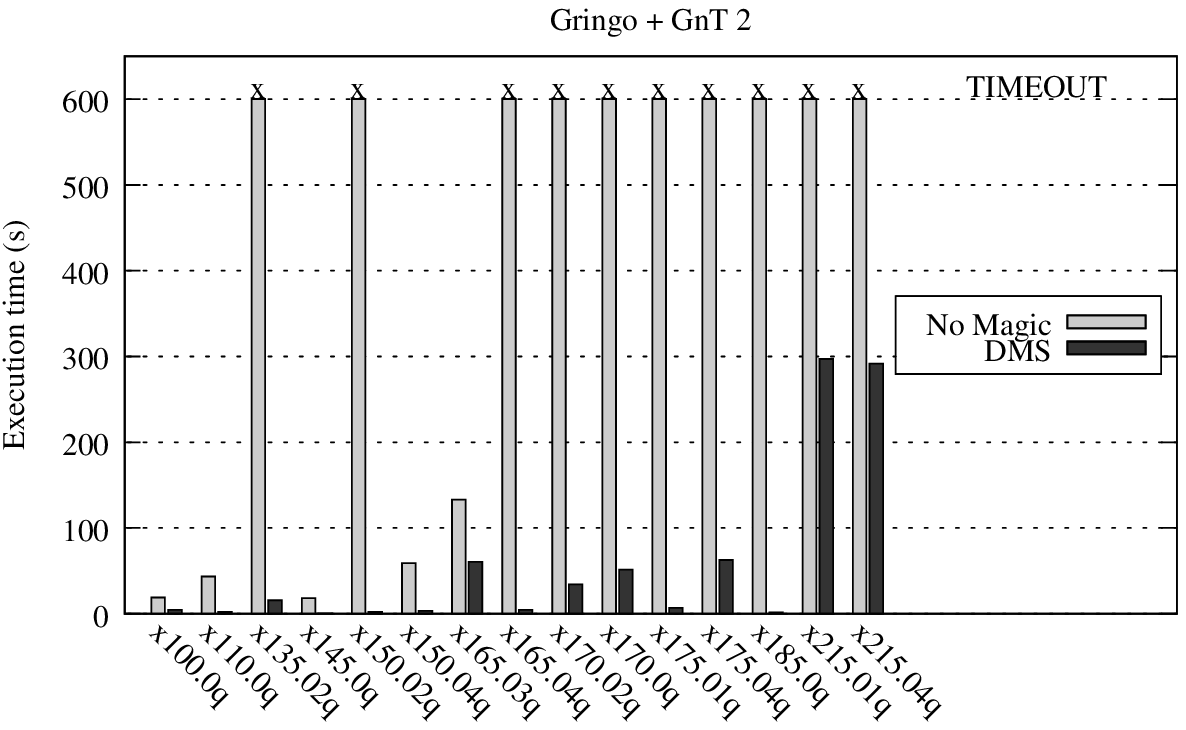}
 \caption{{\em Strategic Companies:} Average execution time on other systems (part 2)}
 \label{fig:strategiccompanies_other2}
\end{figure}

The effectiveness of $\DMS$ is also evident in the \emph{Strategic Companies} benchmark
(Figures~\ref{fig:strategiccompanies_other1}--\ref{fig:strategiccompanies_other2}).
In fact, we observed sensible performance gains of all systems on all tested
instances. GnT1, which is already faster than the other tested systems
in this benchmark, draws particular advantage from $\DMS$, solving all
instances in few seconds. We give another evidence of the optimization
potential provided by $\DMS$ to these systems by comparing the number of solved
instances: Of a total of 60 tests, we counted 37 timeouts on the unoptimized
encoding (10 on ClaspD, 14 on Cmodels, 3 on GnT1 and 10 on GnT2),
while just one on the encoding obtained by applying $\DMS$.
We point out that the timeout on the
rewritten program was obtained by the Cmodels system, which alone collected 14
timeouts on the unoptimized encoding and is thus the least performant on this
benchmark.

\begin{figure}[t]
 \centering
 \includegraphics[width=0.45\textwidth,viewport=5 10 250 140]{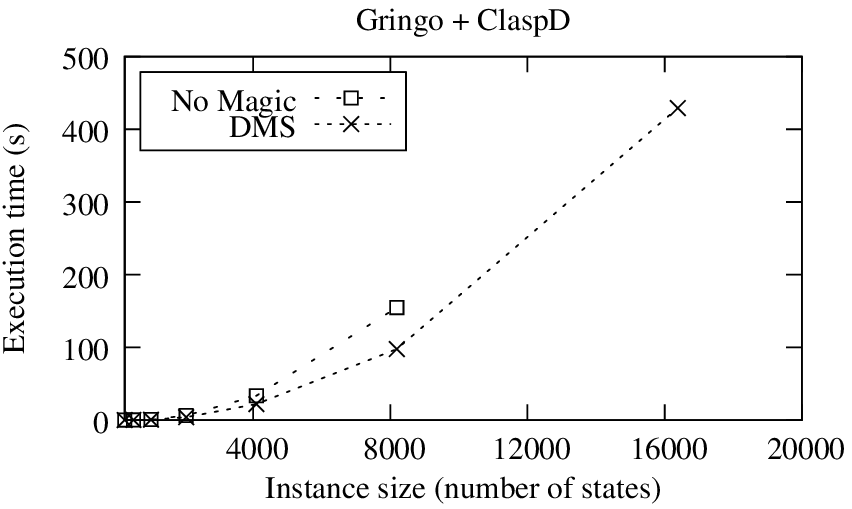}
 \includegraphics[width=0.45\textwidth,viewport=5 10 250 140]{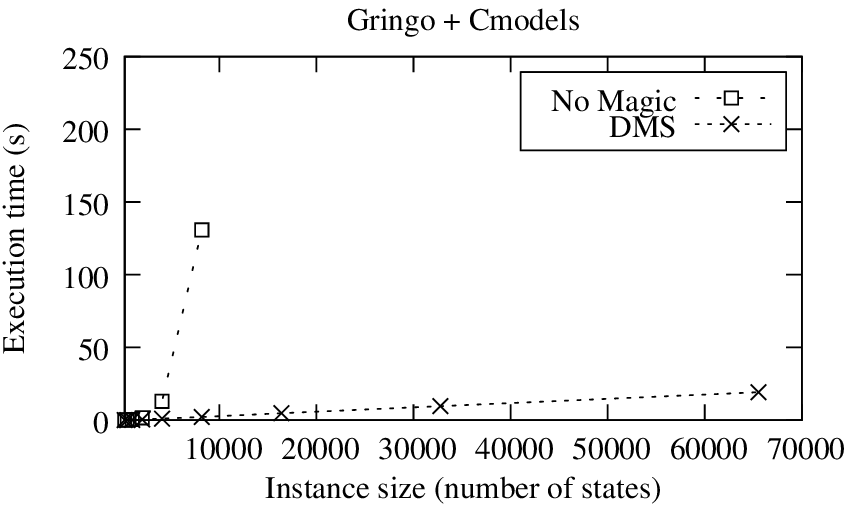}
 \\\vspace{1.5em}
 \includegraphics[width=0.45\textwidth,viewport=5 10 250 140]{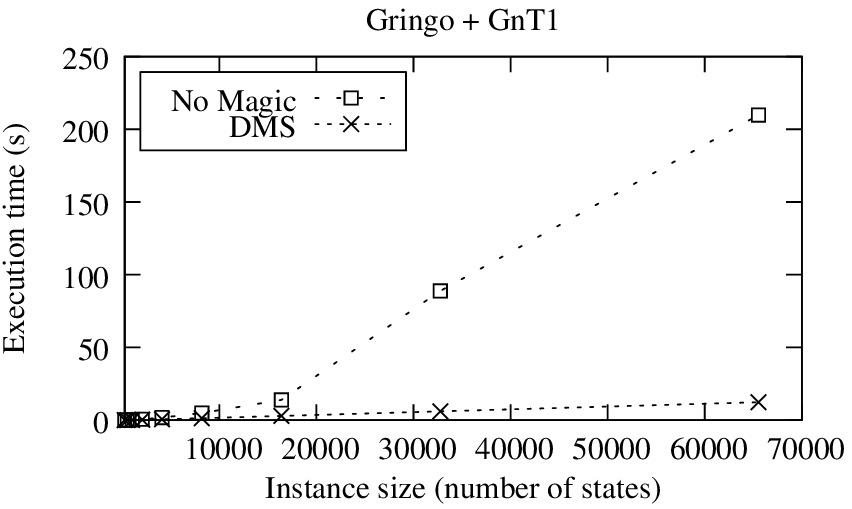}
 \includegraphics[width=0.45\textwidth,viewport=5 10 250 140]{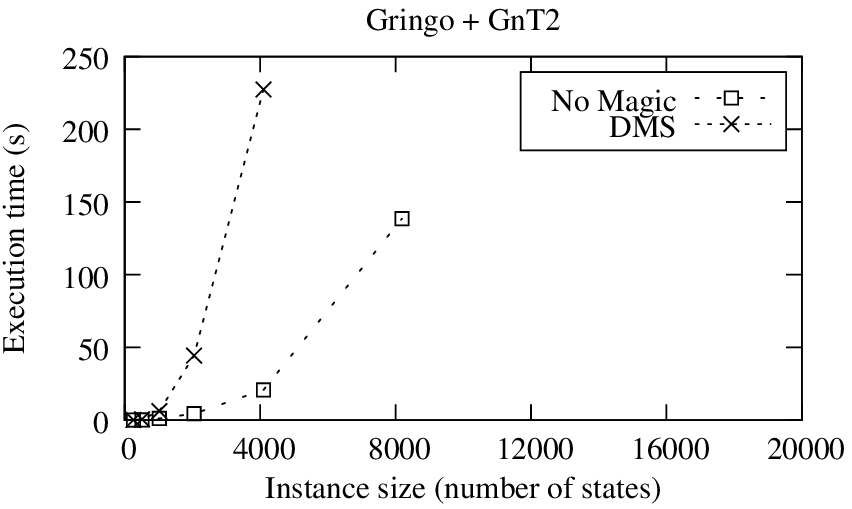}
 \caption{{\em Conformant Plan Checking:} Average execution time on other systems}
 \label{fig:planning_other}
\end{figure}

Finally, consider the results for \emph{Conformant Plan Checking}
reported in Figure~\ref{fig:planning_other} (we used a different scale
on the y-axis for ClaspD for readability; note also that
ClaspD and GnT2 only solved the smallest instances of this benchmark, and we
thus used a different scale for their x-axes). The performance of ClaspD is poor in
this benchmark, nonetheless we observed a slight improvement in execution time
if $\DMS$ is applied on the encoding reported in Section~\ref{sec:experiments-settings}.
Cmodels performs better than ClaspD in this case and the optimization potential
of $\DMS$ emerges with an exponential improvement in performance. A similar
result was observed for GnT1, while GnT2 on this benchmark is the only outlier
of the experiment: Its performance deteriorates if the original program is
processed by $\DMS$. However, in this benchmark GnT2 performs worse that GnT1
also with the original encoding. In fact, while GnT1 solved the biggest instance
(more than 65 thousands of states) in 209.74 seconds (12.28 seconds with the
$\DMS$ encoding), the execution of GnT2 did not terminate in the allotted time
(600 seconds) on instances containing more than 10 thousands of states.
We finally note that with $\DMS$ GnT1 and Cmodels are faster than \dlv in this
benchmark. In fact, for the biggest instance in the benchmark, GnT1 and Cmodels
required 12.28 and 19.13 seconds, respectively, while \dlv terminated in 279.41
seconds. The significant performance gain of GnT1 and Cmodels due to $\DMS$ is
a further confirmation of the potential of our optimization technique.

\section{Application to Data Integration}\label{sec:dataintegration}

In this section we give a brief account of a case study that evidences
the impact of the Magic Set method when used on programs that realize
data integration systems. We first give an overview of data
integration systems, show how they can be implemented using \datds{},
and finally assess the impact of Magic Sets on a data integration system
involving real-world data.

\subsection{Data Integration Systems in a Nutshell}

The main goal of data integration systems is to offer transparent
access to heterogeneous sources by providing users with a \emph{global
  schema}, which users can query without having to know from what
sources the data come from.  In fact, it is the task of the data
integration system to identify and access the data sources which are
relevant for finding the answer to a query over the global schema,
followed by a combination of the data thus obtained. The data
integration system uses a set of \emph{mapping assertions}, which specify the relationship between the
data sources and the global schema.
%
%
Following \cite{Lenz02}, we formalize a data integration system $\I$ as a triple $\tup{\G,\S,\M}$, where:
\begin{enumerate} 

\item $\G$ is the \emph{global (relational) schema}, that is, a pair $\tup{\schemag, \Sigma}$, where $\schemag$ is
a finite set of relation symbols,
  each with an associated positive arity, and $\Sigma$ is a finite set of \emph{integrity constraints} (ICs)
expressed on the symbols in $\schemag$. ICs are first-order assertions that are intended to be satisfied by
database instances.

\item $\S$ is the \emph{source schema}, constituted by the schemas of the various sources that are part of the
data integration system. We assume that $\S$ is a relational schema of the form $\S=\tup{\Psi', \emptyset}$,
which means that there are no integrity constraints on the sources.
This assumption implies that data stored at the sources are locally consistent; this is a common assumption in
data integration, because sources are in general external to the integration system, which is not in charge of
analyzing or restoring their consistency.

\item $\M$ is the \emph{mapping} which establishes the relationship
  between $\G$ and $\S$. In our framework, the mapping
  follows the GAV approach, that is, each global relation
  is associated with a \emph{view}---a $\datds$ query over the sources.
\end{enumerate}

%
%

The main semantic issue in data integration systems is that, since integrated sources are originally
autonomous, their data, transformed via the mapping assertions, may not satisfy the constraints of the global
schema. An approach to remedy to this problem that has lately received a lot of interest in the literature (see,
e.g.,
\cite{aren-etal-01,bert-chom-lead,BCCG02,BrBe03,CaLR03b,chom-marc-2004b,chom-etal-cikm-04,chom-etal-edbt-04,DBLP:conf/sigmod/FuxmanFM05,fuxm-mill-05})
is based on the notion of \emph{repair} for an inconsistent database as introduced in \cite{ArBC99}. Roughly
speaking, a repair of a database is a new database that satisfies the constraints in the schema, and minimally
differs from the original one. Since an inconsistent database might possess
multiple repairs, the
standard approach in answering user queries is to return those answers that are true in every possible repair. These are
called \emph{consistent answers} in the literature.

\subsection{Consistent Query Answering via \datds{} Queries}

There is an intuitive relation between consistent answers to queries
over data integration systems and queries over \datds programs:
Indeed, if one could find a translation from data sources, mapping,
and the query to a \datds{} program, which possesses a stable model
for each possible repair, and a query over it, the consistent answers
within the data integration system will correspond to cautious
consequences of the obtained $\datds$ setting.

In fact, various authors
\cite{ArBC00,BaBe02,BrBe03,CaLR03b,chom-marc-2004b,GrGZ01} considered
the idea of encoding the constraints of the global schema $\G$ into
various kinds of logic programs, such that the stable models of this
program yield the repairs of the database retrieved from the
sources. Some of these approaches use logic programs with unstratified
negation, \cite{CaLR03b}, whereas disjunctive $\dat$ programs together
with unstratified negation have been considered in
\cite{BeBr05,CaBe07}.

It has already been realized earlier that Magic Sets are a crucial
optimization technique in this context, and indeed the availability of
the transformational approach using stable logic programming as its
core language was a main motivation for the research presented in this
article, since in this way a Magic Set method for stable logic
programs immediately yields an optimization technique for data
integration systems. Indeed, the benefits of Magic Sets in the context
of optimizing logic programs with unstratified negation (but without
disjunction) have been discussed in \cite{fabe-etal-2007-jcss}. The
Magic Set technique defined in \cite{fabe-etal-2007-jcss} is quite
different from the one defined in this article, as it does not
consider disjunctive rules, and works only for programs, which are
consistent, that is, have at least one stable model. In \cite{CaBe07}
our preliminary work reported in \cite{cumb-etal-2004-iclp}, which
eventually led to the present article, has been expanded in an ad-hoc
way to particular kinds of \dat{} programs with disjunction and
unstratified negation. It is ad-hoc in the sense that it is tailored
to programs which are created by the transformation described in
\cite{CaBe07}. The experimental results reported in \cite{CaBe07} show
huge computational advantages when using Magic Sets.

We now report an alternative transformation which produces \datds{}
programs (therefore different to \cite{CaBe07}, there are no
unstratified occurrences of negation).  This rewriting has been
devised and used within the INFOMIX system on data integration
\cite{leon-etal-2005}.

Let $\I=\intsys$ be a data integration system where $\G=\tup{\schemag, \Sigma}$, and let $\D$ be a database for
$\G$, which is represented as a set of facts over the relational predicates in $\G$. We
assume that constraints over the global schema are \emph{key} and \emph{exclusion dependencies}. In particular,
we recall that a set of attributes $\bar x$ is a key for the relation $r$ if:

\vspace{-5mm}$$ (r(\bar x,\bar y)\wedge r(\bar x,\bar z)) \rightarrow \bar y=\bar z,\quad\quad \forall
\{r(\bar x,\bar y), r(\bar x,\bar z)\}\subseteq \D
$$

\noindent and that an exclusion dependency holds between a set of attributes $\bar x$
of a relation $r$ and a set of attributes $\bar w$ of a relation $s$ if

\vspace{-5mm}$$ (r(\bar x,\bar y) \land
s(\bar w,\bar z)) \rightarrow \bar y \neq \bar z,\quad\quad \forall \{r(\bar x,\bar y),
s(\bar w,\bar z)\}\subseteq \D
$$

Then, the disjunctive rewriting of a query $q$ with respect to $\I$ is the $\datds$ program $\Pi(\I)=\Pkd \cup \Ped \cup \Pim \cup
\Pi_{coll}$ where:

\begin{itemize}
\item For each relation $r$ in $\G$ and for each key defined over its set of attributes $\bar x$, $\Pkd$
contains the rules:

{
\begin{eqnarray*}
{r}_{out}(\bar x,\bar y)\ \Or\ {r}_{out}(\bar x,\bar z) & \la & r_\D(\bar
x,\bar y) \;,\;
 r_\D(\bar x,\bar z),  Y_1\neq Z_1.\\
&\vdots\\
 {r}_{out}(\bar x,\bar y)\ \Or\ {r}_{out}(\bar x,\bar z) & \la & r_\D(\bar
x,\bar y) \;,\;
 r_\D(\bar x,\bar z),  Y_m\neq Z_m.\\
\end{eqnarray*}
}

\vspace{-10mm}\noindent where $\bar y=Y_1,\dots,Y_m$, and $\bar z=Z_1,\dots,Z_m$.

\item For each exclusion dependency between a set of attributes $\bar x$
of a relation $r$ and a set of attributes $\bar w$ of a relation $s$,
$\Ped$ contains the following rule:

\begin{eqnarray*}
{r}_{out}(\bar x,\bar y)\ \Or\ {s}_{out}(\bar w,\bar z)  & \la &r_\D(\bar
x,\bar y) \;,\;
  s_\D(\bar w,\bar z),\ X_1 = W_1,\ \ldots,\ X_m = W_m.
\end{eqnarray*}

\noindent where $\bar x=X_1,\dots,X_m$, and $\bar w=W_1,\dots,W_m$.
In the implementation the following equivalent rule is used:

\begin{eqnarray*}
r_{out}(\bar x,\bar y)\ \Or\ s_{out}(\bar x,\bar z)\ \derives\ r_{\D}(\bar x,\bar y),\ s_{\D}(\bar x,\bar z).
\end{eqnarray*}

\item For each relation $r$ in $\G$, $\Pi_{coll}$ contains the rule:
\begin{eqnarray*}
r(\bar w) &\la &r_\D(\bar w) \;,\;  \naf\ {r}_{out}(\bar w).
\end{eqnarray*}
%

\item For each $\dat$ rule $\R$ in $\M$ such that:

\begin{dlvcode}
k(\t) \derives q_1(\s_1),\ldots,q_m(\s_m).
\end{dlvcode}

\noindent where $k$ is a relation in $\G$ and $q_i$ (for $1\leq i\leq m$) is a relation in $\S$, $\Pim$
contains the rule:

\begin{dlvcode}
k_\D(\t) \derives q_1(\s_1),\ldots,q_m(\s_m).
\end{dlvcode}
\end{itemize}

\medskip

It can be shown that for each user query $\Q$ (over $\G$) and for each
source database $\F$ (over $\S$), consistent query answers to $\Q$
precisely coincide with the set $\Ans_c(\Q,{\Pi}(\I) \cup \F)$.
Actually, within the INFOMIX project also \emph{inclusion
  dependencies} have been considered according to the rewriting
discussed in \cite{CaLR03b}, whose details we omit for clarity. Since the rewriting for inclusion dependencies
also modifies queries, in the INFOMIX project queries have been limited to conjunctive queries.  It is however
important to notice that the program ${\Pi}(\I)$ contains only stratified negation and is therefore a \datds{}
program, making the Magic Set method defined in this article applicable.

\subsection{Experimental Results}

The effectiveness of the Magic Set method in this crucial application
context has then been assessed via a number of experiments carried out
on the demonstration scenario of the INFOMIX project, which refers to
the information system of the University ``La Sapienza'' in Rome. The
global schema consists of 14 global relations with 29 constraints,
while the data sources include 29 relations of 3 legacy databases and
12 wrappers generating relational data from web pages. This amounts to
more than 24MB of data regarding students, professors and exams in
several faculties of the university.
For a detailed description of the INFOMIX project see
\url{https://www.mat.unical.it/infomix/}.

On this schema, we have tested five typical queries with different
characteristics, which model different use cases. For the sake of completeness,
the full encodings of the tested queries are reported in the Appendix.
In particular, we
measured the average execution time of \dlv computing
$\Ans_c(\Q,{\Pi}(\I) \cup \F)$ and $\Ans_c(\Q,\DMS(\Q,{\Pi}(\I)) \cup
\F)$ on datasets of increasing size. The experiments were performed by
running the INFOMIX prototype system on a 3GHz
Intel$^{\scriptsize\textregistered}$
Xeon$^{\scriptsize\textregistered}$ processor system with 4GB RAM
under the Debian 4.0 operating system with a GNU/Linux 2.6.23 kernel.
The \dlv prototype used as the computational core of the INFOMIX
system had been compiled using GCC 4.3.3. For each instance, we
allowed a maximum running time of 10 minutes and a maximum memory
usage of 3GB.

\begin{figure}[t]
 \centering
 \includegraphics[width=0.45\textwidth,viewport=5 10 250 140]{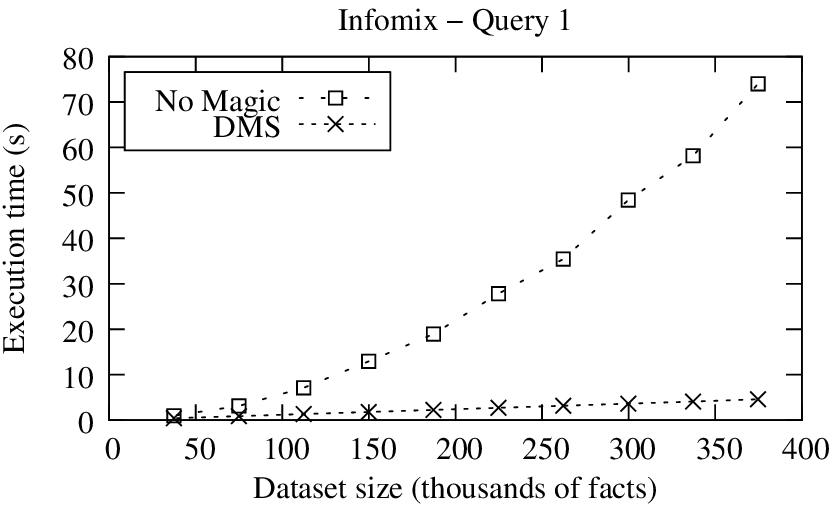}
 \includegraphics[width=0.45\textwidth,viewport=5 10 250 140]{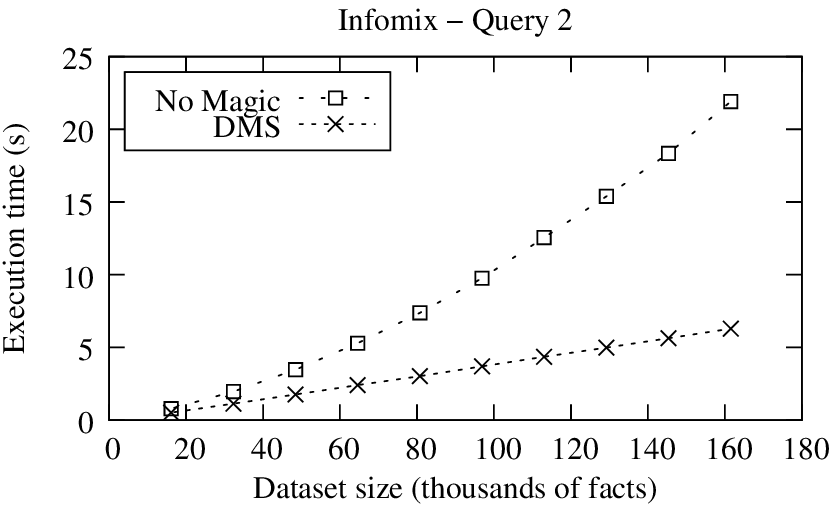}\\
 \vspace{1.5em}
 \includegraphics[width=0.45\textwidth,viewport=5 10 250 140]{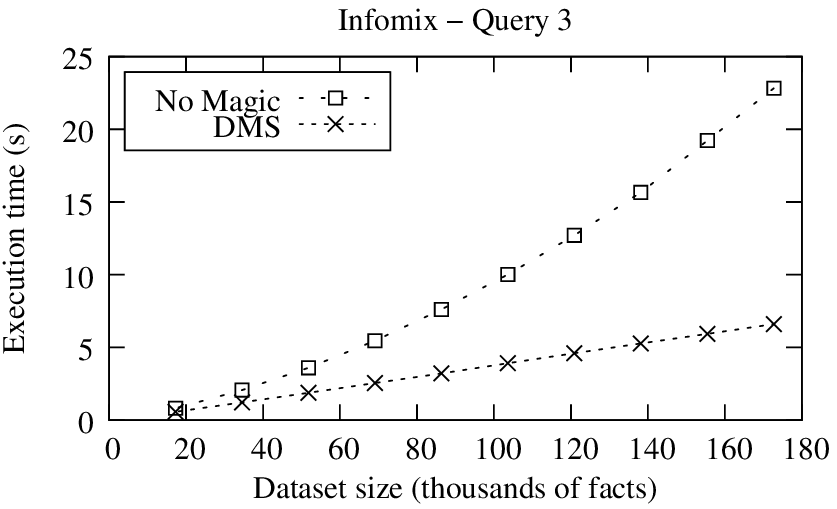}
 \includegraphics[width=0.45\textwidth,viewport=5 10 250 140]{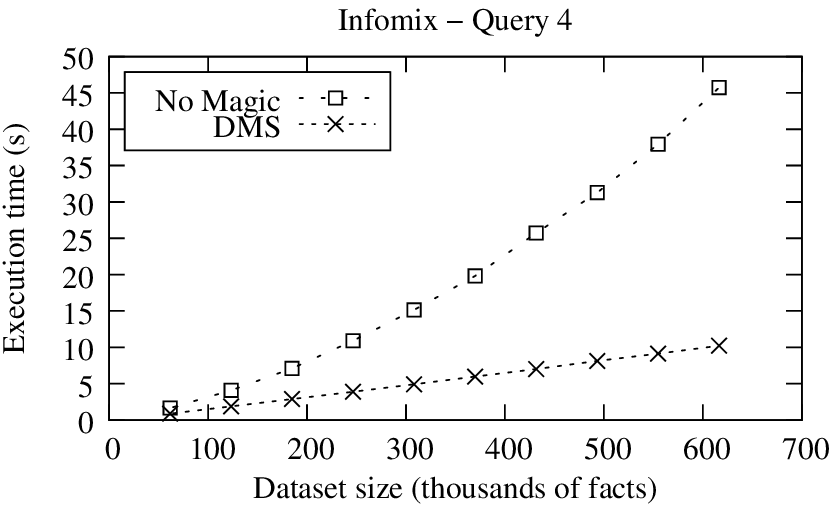}\\
 \vspace{1.5em}
 \includegraphics[width=0.45\textwidth,viewport=5 10 250 140]{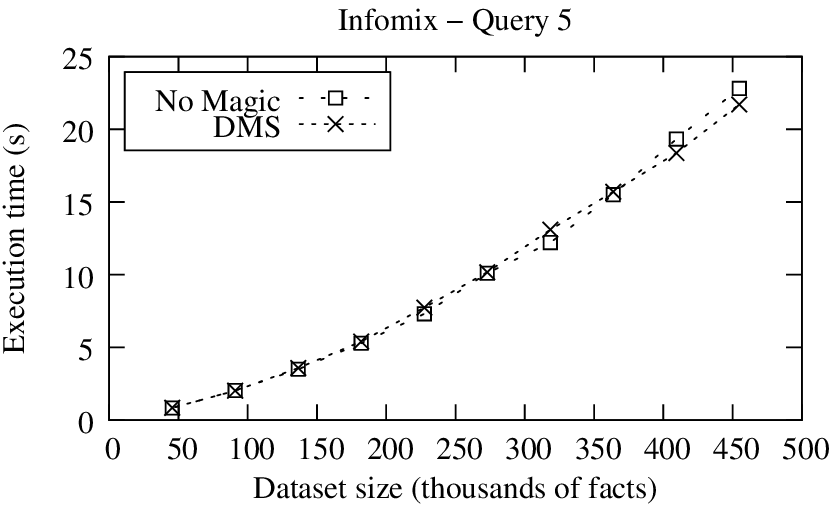}
 \caption{Average execution time of query evaluation in the INFOMIX Demo Scenario}
 \label{fig:Infomix}
\end{figure}

The results, reported in Figure~\ref{fig:Infomix}, confirm that on
these typical queries the performance is considerably improved by
Magic Sets. On Queries 1 to 4 in Figure~\ref{fig:Infomix} the response
time scales much better with Magic Sets than without, appearing
essentially linear on the tested instance sizes, while without Magic
Sets the behavior has a decidedly non-linear appearance.  We also
observe that there is basically no improvement on Query $5$. We have
analyzed this query and for this use case all data seems to be
relevant to the query, which means that Magic Sets cannot have any
positive effect. It is however important to observe that the Magic
Set rewriting does not incur any significant overhead.

\section{Related Work}
\label{sec:relatedwork}

In this section we first discuss the main body of work which is related to
$\DMS$, the technique developed in this paper for query answering
optimization. In particular, we discuss Magic Set techniques for Datalog
languages. The discussion is structured in paragraphs grouping techniques which
cover the same language.
After that, we discuss some applications for which $\DMS$
have already been exploited. All these applications refer to the preliminary
work published in \cite{cumb-etal-2004-iclp}.

\paragraph*{Magic Sets for $\dat$.}
In order to optimize query evaluation in bottom-up systems, like deductive
database systems, several works have proposed the simulation of top-down
strategies by means of suitable transformations introducing new predicates and
rewriting clauses. Among them, Magic Sets for Datalog queries are one of the best
known logical optimization techniques for database systems.
The method, first developed in \cite{banc-etal-86},
has been analyzed and refined by many authors; see, for instance,
\cite{beer-rama-91,mumi-etal-90,stuc-suda-94,ullm-89}.
These works form the foundations of $\DMS$.

\paragraph*{Magic Sets for $\dats$.}
Many authors have addressed the issue of extending the Magic Set technique in
order to deal with Datalog queries involving stratified negation.
The main problem related to the extension of the technique to $\dats$ programs
is how to assign a semantics to the rewritten programs. Indeed, while $\dats$
programs have a natural and accepted semantics, namely the perfect model
semantics \cite{apt-etal-88,vang-88}, the application of Magic Sets
can introduce unstratified negation in the rewritten programs.
A solution has been presented in
\cite{behr-2003,kemp-etal-95,keri-pugi-1988,ross-94}. In particular,
in \cite{kemp-etal-95,ross-94} rewritten programs have been evaluated according
to the well-founded semantics, a three-valued semantics for \datn programs
which is two-valued for stratified programs, while
in \cite{behr-2003,keri-pugi-1988} ad-hoc semantics have been defined.
All of these methods exploit a property of $\dats$ which is not present in
disjunctive Datalog, uniqueness of the intended model. This property in turn
implies that query answering just consists in establishing the truth value of
some atoms in one intended model. Using our terminology, brave and cautious
reasoning coincide for these programs. Therefore, all these methods are
quite different from $\DMS$, the technique developed in this paper.

\paragraph*{Magic Sets for $\datn$.}
Extending the Magic Set technique to $\datn$ programs must face two major difficulties.
First, for a $\datn$ program uniqueness of the intended model is no more
guaranteed, thus query answering in this setting involves a set of stable
models in general.
The second difficulty is that parts of a $\datn$ program may act as constraints,
thus impeding a relevant interpretation to be a stable model.
In \cite{fabe-etal-2007-jcss} a Magic Set method for $\datn$ programs has been
defined and proved to be correct for coherent programs, i.e., programs admitting
at least one stable model. This method
takes special precautions for relevant parts of the program that act as
constraints, called \emph{dangerous rules} in \cite{fabe-etal-2007-jcss}.
We observe that dangerous rules cannot occur in $\datds$ programs,
which allows for the simpler $\DMS$ algorithm to work correctly for this
class of programs.

\paragraph*{Magic Sets for $\datd$.}
The first extension of the Magic Set technique to disjunctive Datalog is due to
\cite{grec-99,grec-2003}, where the $\SMS$ method has been presented and proved
to be correct for $\datd$ programs.
We point out that the main drawback of this method is the introduction of
{\em collecting} predicates. Indeed, magic and collecting predicates of
$\SMS$ have deterministic definitions. As a consequence, their extension
can be completely computed during program instantiation, which means that no
further optimization is provided for the subsequent stable model search.
Moreover, while the correctness of $\DMS$ has been formally established
for $\datds$ programs in general, the applicability of
$\SMS$ to $\datds$ programs has only been outlined in
\cite{grec-99,grec-2003}.

\paragraph*{Applications.}
Magic Sets have been applied in many contexts. In particular,
\cite{BeBr05,hust-etal-2007,CaBe07,moti-2006} have profitably exploited
the optimization provided by $\DMS$.
In particular, in \cite{BeBr05,CaBe07} a data integration system has been
presented. The system is based on disjunctive Datalog and exploits
$\DMS$ for fast query answering.
In \cite{hust-etal-2007,moti-2006}, instead, an algorithm for answering queries
over description logic knowledge bases has been presented. More
specifically, the algorithm reduces a $\mathcal{SHIQ}$ knowledge base to a
disjunctive Datalog program, so that $\DMS$ can be exploited
for query answering optimization.

\section{Conclusion}\label{sec:conclusion}

The Magic Set method is one of the best-known techniques for the optimization of positive recursive $\dat$
programs due to its efficiency and its generality. Just a few other focused methods such as the supplementary
Magic Set and other special techniques for linear and chain queries have gained similar visibility (see, e.g.,
\cite{grec-sacc-95,rama-etal-93,ullm-89}).
After seminal papers \cite{banc-etal-86,beer-rama-91}, the viability of the approach was demonstrated e.g., in
\cite{gupt-mumi-92,mumi-etal-90}. Later on, extensions and refinements were proposed, addressing e.g., query
constraints in \cite{stuc-suda-94}, the well-founded semantics in \cite{kemp-etal-95}, or integration into
cost-based query optimization in \cite{sesh-etal-96}. The research on variations of the Magic Set method is
still going on. For instance, in \cite{fabe-etal-2007-jcss} an extension of the Magic Set method was discussed for the
class of unstratified logic programs (without disjunction). In \cite{behr-2003} a technique for the class of
\emph{soft-stratifiable} programs was given. Finally, in \cite{grec-2003} the first variant of the technique for disjunctive programs ($\SMS$) was described.

In this paper, we have elaborated on the issues addressed in \cite{grec-99,grec-2003}. Our approach is similar
to $\SMS$, but differs in several respects:

\vspace{-2mm}\begin{itemize}
\item $\DMS$ is a dynamic optimization of query answering, in the sense that
in addition to the optimization of the
grounding process (which is the only optimization performed by $\SMS$),
$\DMS$ can drive the model
generation phase by dynamically disabling parts of the program that become
irrelevant in the considered partial interpretations.

\item $\DMS$ has a strong relationship with unfounded sets, allowing for a clean
application to disjunctive $\dat$ programs also in presence of stratified negation.

\item $\DMS$ can be further improved by performing a subsequent subsumption check.

\item $\DMS$ is integrated into the \dlv system \cite{leon-etal-2002-dlv},
profitably exploiting the \dlv internal data-structures and the ability of controlling the grounding module.
\end{itemize}

We have conducted experiments on several benchmarks, many of which
taken from the literature. The results of our experimentation evidence that our
implementation outperforms $\SMS$ in general, often by an exponential factor.
This is mainly due to the optimization of the model generation phase, which is
specific to our Magic Set technique.
In addition, we have conducted further experiments on a real
application scenario, which show that Magic Sets can play a crucial role in optimizing consistent query
answering over inconsistent databases. Importantly, other authors have already recognized the benefits of our
optimization strategies with respect to this very important application domain \cite{CaBe07}, thereby confirming the
validity and the robustness of the work discussed in this paper.

We conclude by observing that it has been noted in the literature (e.g., in \cite{kemp-etal-95}) that in the
non-disjunctive case \emph{memoing} techniques lead to similar computations as evaluations after Magic Set
transformations. Also in the disjunctive case such techniques have been proposed (e.g., Hyper Tableaux
\cite{baum-etal-96}), for which similar relations might hold. While \cite{kemp-etal-95} has already evidenced
that an advantage of Magic Sets over such methods is that they may be more easily combined with other
optimization techniques, we believe that achieving a deeper comprehension of the relationships among these
techniques constitutes an interesting avenue for further research.

Another issue that we leave for future work is to study the impact of
changing some parameters of the \DMS{} method, in particular the
impact of different SIPSes.

\bibliographystyle{plain}

\bibliography{bibtex,magic,main-string,main-bib}

\clearpage

\appendix

\section{Queries on the INFOMIX Demo Scenario}

INFOMIX is a project that was funded by the European Commission in its Information Society Technologies track of the Sixth Framework Programme for providing an advanced
system for information integration.
A detailed description of the project, including references in the literature,
can be found at \url{https://www.mat.unical.it/infomix/}.
Five typical queries of the INFOMIX demo scenario have been considered for
assessing Dynamic Magic Sets.
The full encodings of the tested queries are reported in
Figures~\ref{fig:INFOMIX:queries123}--\ref{fig:INFOMIX:queries45}.
Note that the encodings include the transformation described in
Section~\ref{sec:dataintegration}, and that underlined predicates denote source
relations.

\def\stINFOMIX{\vspace{-.3cm}\scriptstyle}
\def\hrINFOMIX{\hrule\vspace{-.5cm}}
\begin{figure}[b]
\begin{dlvcode}
\stINFOMIX\tt course_\D(X_1,X_2)\ :-\ \underline{esame}(\_,X_1,X_2,\_).\\
\stINFOMIX\tt course_\D(X_1,X_2)\ :-\ \underline{esame\_diploma}(X_1,X_2).\\
\stINFOMIX\tt exam\_record_\D(X_1,X_2,Z,W,X_4,X_5,Y)\ :-\ \underline{affidamenti\_ing\_informatica}(X_2,X_3,Y),\\
\stINFOMIX\tt \quad\quad\quad \underline{dati\_esami}(X_1,\_,X_2,X_5,X_4,\_,Y),\ \underline{dati\_professori}(X_3,Z,W).\\
\stINFOMIX\tt exam\_record_{out}(X_1,X_2,X_3,X_4,Y_5,Y_6,Y_7)\ \Or\ exam\_record_{out}(X_1,X_2,X_3,X_4,Z_5,Z_6,Z_7)\ :-\ \\
\stINFOMIX\tt \quad\quad\quad exam\_record_\D(X_1,X_2,X_3,X_4,Y_5,Y_6,Y_7),\ exam\_record_\D(X_1,X_2,X_3,X_4,Z_5,Z_6,Z_7),\ Y_5 \neq Z_5.\\
\stINFOMIX\tt exam\_record_{out}(X_1,X_2,X_3,X_4,Y_5,Y_6,Y_7)\ \Or\ exam\_record_{out}(X_1,X_2,X_3,X_4,Z_5,Z_6,Z_7)\ :-\ \\
\stINFOMIX\tt \quad\quad\quad exam\_record_\D(X_1,X_2,X_3,X_4,Y_5,Y_6,Y_7),\ exam\_record_\D(X_1,X_2,X_3,X_4,Z_5,Z_6,Z_7),\ Y_6 \neq Z_6.\\
\stINFOMIX\tt exam\_record_{out}(X_1,X_2,X_3,X_4,Y_5,Y_6,Y_7)\ \Or\ exam\_record_{out}(X_1,X_2,X_3,X_4,Z_5,Z_6,Z_7)\ :-\ \\
\stINFOMIX\tt \quad\quad\quad exam\_record_\D(X_1,X_2,X_3,X_4,Y_5,Y_6,Y_7),\ exam\_record_\D(X_1,X_2,X_3,X_4,Z_5,Z_6,Z_7),\ Y_7 \neq Z_7.\\
\stINFOMIX\tt course(X_1,X_2)\ :-\ course_{\D}(X_1,X_2),\ \naf~course_{out}(X_1,X_2).\\
\stINFOMIX\tt exam\_record(X_1,X_2,X_3,X_4,X_5,X_6,X_7)\ :-\ exam\_record_\D(X_1,X_2,X_3,X_4,X_5,X_6,X_7),\\
\stINFOMIX\tt \quad\quad\quad \naf~exam\_record_{out}(X_1,X_2,X_3,X_4,X_5,X_6,X_7).\\
\stINFOMIX\tt query_1(CD)\ :-\ course(C,CD),\ exam\_record(``09089903",C,\_,\_,\_,\_,\_).\\
\stINFOMIX\tt query_1(CD)?
\end{dlvcode}
\hrINFOMIX
\begin{dlvcode}
\stINFOMIX\tt student_\D(X_1,X_2,X_3,X_4,X_5,X_6,X_7)\ :-\ \underline{diploma\_maturita}(Y,X_7),\\
\stINFOMIX\tt \quad\quad\quad \underline{studente}(X_1,X_3,X_2,\_,\_,\_,\_,\_,\_,\_,\_,\_,X_6,X_5,\_,\_,X_4,\_,\_,\_,\_,Y,\_).\\
\stINFOMIX\tt student(X_1,X_2,X_3,X_4,X_5,X_6,X_7)\ :-\ student_\D(X_1,X_2,X_3,X_4,X_5,X_6,X_7),\\
\stINFOMIX\tt \quad\quad\quad  \naf~student_{out}(X_1,X_2,X_3,X_4,X_5,X_6,X_7).\\
\stINFOMIX\tt query_2(SFN,SLN,COR,ADD,TEL,HSS)\ :-\ student(``09089903",SFN,SLN,COR,ADD,TEL,HSS).\\
\stINFOMIX\tt query_2(SFN,SLN,COR,ADD,TEL,HSS)?
\end{dlvcode}
\hrINFOMIX
\begin{dlvcode}
\stINFOMIX\tt student_\D(X_1,X_2,X_3,X_4,X_5,X_6,X_7)\ :-\ \underline{diploma\_maturita}(Y,X_7),\\
\stINFOMIX\tt \quad\quad\quad \underline{studente}(X_1,X_3,X_2,\_,\_,\_,\_,\_,\_,\_,\_,\_,X_6,X_5,\_,\_,X_4,\_,\_,\_,\_,Y,\_).\\
\stINFOMIX\tt student\_course\_plan_\D(X_1,X_2,X_3,X_4,X_5)\ :-\ \underline{orientamento}(Y_1,X_3),\\
\stINFOMIX\tt \quad\quad\quad \underline{piano\_studi}(X_1,X_2,Y_1,X_4,Y_2,\_,\_,\_,\_,\_),\ \underline{stato}(Y_2,X_5).\\
\stINFOMIX\tt student(X_1,X_2,X_3,X_4,X_5,X_6,X_7)\ :-\ student_\D(X_1,X_2,X_3,X_4,X_5,X_6,X_7),\\
\stINFOMIX\tt \quad\quad\quad \naf~student_{out}(X_1,X_2,X_3,X_4,X_5,X_6,X_7).\\
\stINFOMIX\tt student\_course\_plan(X_1,X_2,X_3,X_4,X_5)\ :-\ student\_course\_plan_\D(X_1,X_2,X_3,X_4,X_5),\\
\stINFOMIX\tt \quad\quad\quad \naf~student\_course\_plan_{out}(X_1,X_2,X_3,X_4,X_5).\\
\stINFOMIX\tt query_3(SID,SLN,R)\ :-\ student(SID,``ZNEPB",SLN,\_,\_,\_,\_),\\
\stINFOMIX\tt \quad\quad\quad student\_course\_plan(\_,SID,\_,R,``APPROVATO\ SENZA\ MODIFICHE").\\
\stINFOMIX\tt query_3(SID,SLN,R)?
\end{dlvcode}
\caption{INFOMIX Queries 1--3}\label{fig:INFOMIX:queries123}
\end{figure}

\begin{figure}
\begin{dlvcode}
\stINFOMIX\tt student_\D(X_1,X_2,X_3,X_4,X_5,X_6,X_7)\ :-\ \underline{diploma\_maturita}(Y,X_7),\\
\stINFOMIX\tt \quad\quad\quad \underline{studente}(X_1,X_3,X_2,\_,\_,\_,\_,\_,\_,\_,\_,\_,X_6,X_5,\_,\_,X_4,\_,\_,\_,\_,Y,\_).\\
\stINFOMIX\tt course_\D(X_1,X_2)\ :-\ \underline{esame}(\_,X_1,X_2,\_).\\
\stINFOMIX\tt course_\D(X_1,X_2)\ :-\ \underline{esame\_diploma}(X_1,X_2).\\
\stINFOMIX\tt student\_course\_plan_\D(X_1,X_2,X_3,X_4,X_5)\ :-\ \underline{orientamento}(Y_1,X_3),\\
\stINFOMIX\tt \quad\quad\quad \underline{piano\_studi}(X_1,X_2,Y1,X_4,Y_2,\_,\_,\_,\_,\_),\ \underline{stato}(Y_2,X_5).\\
\stINFOMIX\tt plan\_data_\D(X_1,X_2,X_3)\ :-\ \underline{dati\_piano\_studi}(X_1,X_2,\_),\\
\stINFOMIX\tt \quad\quad\quad \underline{esame\_ingegneria}(X_2,Y_3,Y_2,\_),\ \underline{tipo\_esame}(Y_2,X_3).\\
\stINFOMIX\tt student(X_1,X_2,X_3,X_4,X_5,X_6,X_7)\ :-\ student_\D(X_1,X_2,X_3,X_4,X_5,X_6,X_7),\\
\stINFOMIX\tt \quad\quad\quad \naf~student_{out}(X_1,X_2,X_3,X_4,X_5,X_6,X_7).\\
\stINFOMIX\tt student\_course\_plan(X_1,X_2,X_3,X_4,X_5)\ :-\ student\_course\_plan_\D(X_1,X_2,X_3,X_4,X_5)\\
\stINFOMIX\tt \quad\quad\quad \naf~student\_course\_plan_{out}(X_1,X_2,X_3,X_4,X_5).\\
\stINFOMIX\tt plan\_data(X_1,X_2,X_3)\ :-\ plan\_data_\D(X_1,X_2,X_3),\ \naf~plan\_data_{out}(X_1,X_2,X_3).\\
\stINFOMIX\tt course(X_1,X_2)\ :-\ course_\D(X_1,X_2),\ \naf~course_{out}(X_1,X_2).\\
\stINFOMIX\tt query_4(F,S)\ :-\ course(CID,``RETI LOGICHE"),\ plan\_data(SCID,CID,\_),\\
\stINFOMIX\tt \quad\quad\quad student(SID,F,S,``ROMA",\_,\_,\_),\ student\_course\_plan(SCID,SID,\_,\_,\_).\\
\stINFOMIX\tt query_4(F,S)?
\end{dlvcode}
\hrINFOMIX
\begin{dlvcode}
\stINFOMIX\tt course_\D(X_1,X_2)\ :-\ \underline{esame}(\_,X_1,X_2,\_).\\
\stINFOMIX\tt course_\D(X_1,X_2)\ :-\ \underline{esame\_diploma}(X_1,X_2).\\
\stINFOMIX\tt student\_course\_plan_\D(X_1,X_2,X_3,X_4,X_5)\ :-\ \underline{orientamento}(Y_1,X_3),\\
\stINFOMIX\tt \quad\quad\quad \underline{piano\_studi}(X_1,X_2,Y_1,X_4,Y_2,\_,\_,\_,\_,\_),\ \underline{stato}(Y_2,X_5).\\
\stINFOMIX\tt plan\_data_\D(X_1,X_2,X_3)\ :-\ \underline{dati\_piano\_studi}(X_1,X_2,\_),\\
\stINFOMIX\tt \quad\quad\quad \underline{esame\_ingegneria}(X_2,Y_3,Y_2,\_),\ \underline{tipo\_esame}(Y_2,X_3).\\
\stINFOMIX\tt student\_course\_plan(X_1,X_2,X_3,X_4,X_5)\ :-\ student\_course\_plan_\D(X_1,X_2,X_3,X_4,X_5),\\
\stINFOMIX\tt \quad\quad\quad \naf~student\_course\_plan_{out}(X_1,X_2,X_3,X_4,X_5).\\
\stINFOMIX\tt plan\_data(X_1,X_2,X_3)\ :-\ plan\_data_\D(X_1,X_2,X_3),\ \naf~plan\_data_{out}(X_1,X_2,X_3).\\
\stINFOMIX\tt course(X_1,X_2)\ :-\ course_\D(X_1,X_2),\ \naf~course_{out}(X_1,X_2).\\
\stINFOMIX\tt query_5(D)\ :-\ course(E,D),\ plan\_data(C,E,\_),\ student\_course\_plan(C,``09089903",\_,\_,\_).\\
\stINFOMIX\tt query_5(D)?
\end{dlvcode}
\caption{INFOMIX Queries 4--5}\label{fig:INFOMIX:queries45}
\end{figure}

\end{document}